[1994/12/01]
\documentclass{aomart}
\usepackage[dvipsnames]{xcolor}
\usepackage{cite}

\usepackage{amsfonts}       
\usepackage{booktabs}       
\usepackage{cleveref}       
\usepackage{epsfig}         
\usepackage{float}          
\usepackage{graphicx}       
\usepackage{hyperref}       
\usepackage{mathtools}      
\usepackage{microtype}      
\usepackage{nicefrac}       
\usepackage{subfig}
\usepackage{url}  

\usepackage{xspace}
\usepackage{amssymb}

\usepackage{multirow}
\usepackage{enumitem}

 \usepackage[ruled,vlined,linesnumbered,noresetcount]{algorithm2e}
 \usepackage{algpseudocode}
 \usepackage{amsthm}
 
\usepackage{wrapfig}

\usepackage{thmtools,thm-restate}

\declaretheorem{theorem}

\allowdisplaybreaks

\usepackage{fullpage}
\usepackage{charter}
\usepackage[T1,small,euler-digits]{eulervm}
\newtheorem{Lm}{Lemma}

\newcommand{\OO}{\mathcal{O}}
\newcommand{\E}{\mathbb{E}}
\newcommand{\R}{\mathbb{R}}
\newcommand{\Var}{\texttt{Var}}
\DeclarePairedDelimiter\ceil{\lceil}{\rceil}
\DeclarePairedDelimiter\floor{\lfloor}{\rfloor}

\usepackage{bbm}

\newif\ifsubmit
\submitfalse

\newcommand{\reviewer}[3]{
  \expandafter\newcommand\csname #1\endcsname[1]{
    \ifsubmit
      \textcolor{#3}{}
    \else
      \textcolor{#3}{[#2: ##1]}
    \fi
  }
}

\reviewer{bo}{Bo}{blue}
\reviewer{ruoxi}{Ruoxi}{cyan}
\reviewer{dawn}{Dawn}{red}
\reviewer{ce}{Ce}{magenta}
\reviewer{david}{David}{Emerald}
\reviewer{fh}{France}{orange}
\reviewer{merve}{Merve}{violet}
\reviewer{nick}{Nick}{PineGreen}
\definecolor{neonpurple}{rgb}{0.3,0,1}
\reviewer{dan}{Dan}{neonpurple}

\title[Sample Paper]{Efficient Task-Specific Data Valuation for Nearest Neighbor Algorithms}
\author[AMS]{Ruoxi Jia}
\address{University of California, Berkeley}
\email{ruoxijia@berkeley.edu} 
\author[AMS]{David Dao}
\address{ETH Zurich}
\email{dwddao@gmail.com} 
\author[AMS]{Boxin Wang}
\address{Zhejiang University}
\email{boxin.wang@outlook.com} 
\author[AMS]{Frances Ann Hubis}
\address{ETH Zurich}
\email{hubisf@student.ethz.ch} 
\author[AMS]{Nezihe Merve Gurel}
\address{ETH Zurich}
\email{nezihe.guerel@inf.ethz.ch} 
\author[AMS]{Bo Li}
\address{University of Illinois at Urbana–Champaign}
\email{lxbosky@gmail.com} 
\author[AMS]{Ce Zhang}
\address{ETH Zurich}
\email{ce.zhang@inf.ethz.ch} 
\author[AMS]{Costas J. Spanos}
\address{University of California, Berkeley}
\email{spanos@berkeley.edu} 
\author[AMS]{Dawn Song}
\address{University of California, Berkeley}
\email{dawnsong@gmail.com}

\copyrightnote{}

\begin{document}

\begin{abstract}
Given a data set $\mathcal{D}$ containing millions of data points and a data consumer who is willing to pay for \$$X$ to train a machine learning (ML) model over $\mathcal{D}$, {\em how should we distribute this \$$X$ to each data point to reflect its ``value''?} In this paper, we define the ``relative value of data'' via the Shapley value, as it uniquely possesses properties with appealing real-world interpretations, such as fairness, rationality and decentralizability. For general, bounded utility functions, the Shapley value is known to be challenging to compute: to get Shapley values for all $N$ data points, it requires $O(2^N)$ model evaluations for exact computation and $O(N\log N)$ for $(\epsilon, \delta)$-approximation. 

In this paper, we focus on one popular family of ML models relying on $K$-nearest neighbors ($K$NN). The most surprising result is that for unweighted $K$NN classifiers and regressors, the Shapley value of all $N$ data points can be computed, {\em exactly}, in $O(N\log N)$ time -- an exponential improvement on computational complexity! Moreover, for $(\epsilon, \delta)$-approximation, we are able to develop an algorithm based on Locality Sensitive Hashing (LSH) with only {\em sublinear} complexity $O(N^{h(\epsilon,K)}\log N)$ when $\epsilon$ is not too small and $K$ is not too large. We empirically evaluate our algorithms on up to $10$ million data points and even our {\em exact} algorithm is up to three orders of magnitude faster than the baseline approximation algorithm. The LSH-based approximation algorithm can accelerate the value calculation process even further.

We then extend our algorithms to other scenarios such as (1) weighed $K$NN classifiers, (2) different data points are clustered by different {\em data curators}, and (3) there are {\em data analysts} providing computation who also requires proper valuation. {\em Some} of these extensions, although also being improved exponentially, are less practical for exact computation (e.g., $O(N^K)$ complexity for weighted $K$NN). We thus propose a Monte Carlo approximation algorithm, which is $O(N (\log N)^2 / (\log K)^2)$ times more efficient than the baseline approximation algorithm. 

\end{abstract}

\maketitle

\clearpage
\tableofcontents
\clearpage

\section{Introduction}
``{\em Data is the new oil}'' --- large-scale, high-quality datasets are an enabler for business and scientific discovery and recent years have witnessed the commoditization of
data. In fact, there are not only marketplaces providing
access to data, e.g., IOTA~\cite{IOTA}, DAWEX~\cite{DAWEX}, Xignite~\cite{xignite}, but also 
marketplaces charging for running (relational)
queries over the data, e.g., Google BigQuery~\cite{BIGQUERY}. Many researchers start to envision
marketplaces for ML models~\cite{chen2018model}.

Data commoditization is highly likely to continue
and not
surprisingly, it starts to attract
interests from the database community.
One series of seminal work is conducted by Koutris et al.~\cite{koutris2015query,koutris2013toward}
who systematically studied the theory and practice
of ``query pricing,'' the problem of attaching
value to running relational queries over data.
Recently, Chen et al.~\cite{chen2018model,chen2017model}
discussed ``model pricing'', the problem of valuing ML models. 
This paper is
inspired by the prior work on query and model pricing, but focuses on
a different scenario.
In many real-world applications,
the datasets that support queries and ML are often contributed by {\em multiple individuals}. One example is that complex ML tasks such as chatbot training often relies on massive crowdsourcing efforts. A critical challenge for building a data marketplace is thus to allocate the revenue generated from queries and ML models fairly between different data contributors. In this paper, we ask: {\em How can we
attach value to every single data point in relative terms, with
respect to a specific ML model
trained over the whole dataset? }

Apart from being
inspired by recent research, 
this paper is also motivated by 
our current effort in building a data market
based on privacy-preserving machine 
learning~\cite{hynes2018demonstration, dao2018databright} and an ongoing
clinical trial at the Stanford Hospital,
as illustrated in Figure~\ref{fig:overview}.
In this
clinical trial, each patient uploads their
encrypted medical record (one ``data point'') onto a 
blockchain-backed data store. A ``data consumer'',
or ``buyer'', chooses a subset of patients (selected
according to some non-sensitive
information that is not encrypted)
and trains a ML model. The buyer pays
a certain amount of money that will be distributed
back to each patient. In this paper, we
focus on the data valuation problem that is 
abstracted from this real use case and propose
novel, practical algorithms for this problem. 

\begin{figure}[ht!]
\centering
\includegraphics[width=0.9\columnwidth]{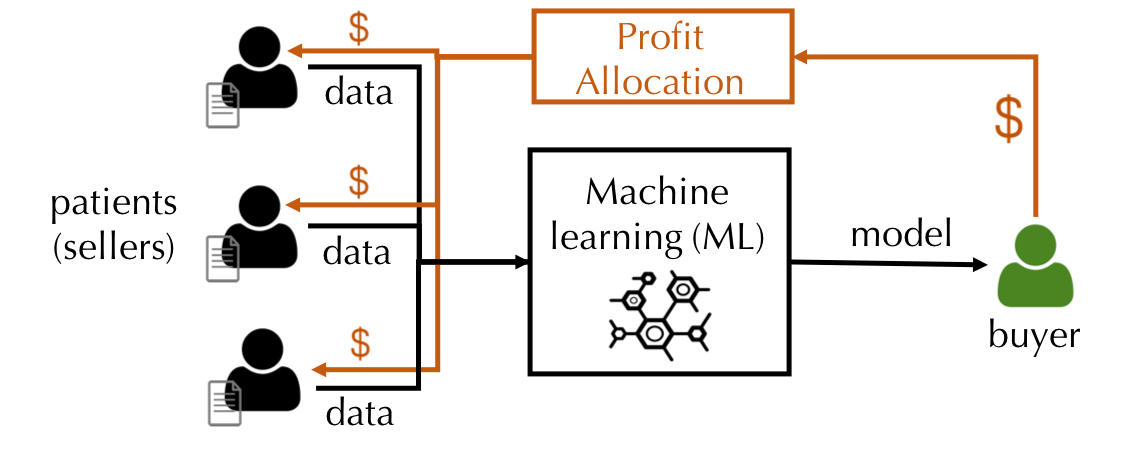}
\caption{Motivating Example of Data Valuation.}
\label{fig:overview}
\end{figure}

Specifically,
we focus on the {\em Shapley value} (SV), arguably one of the most popular way of revenue sharing. It has been applied to various applications, such as power grids~\cite{bremer2013estimating}, supply chains~\cite{bartholdi2005using}, cloud computing~\cite{upadhyaya2012price}, among others. The reason for its wide adoption is that the SV defines a \emph{unique} profit allocation scheme that satisfies a set of appealing properties, such as fairness, rationality, and decentralizability. 
Specifically, let $\mathcal{D} = \{z_1,...,z_N\}$ be $N$ data points and $\nu(\cdot)$ be the ``utility'' of the ML model trained over a subset of the data points; the SV of a given 
data point $z_i$ is
\begin{equation}
    s_i = \frac{1}{N}\sum_{S \subseteq \mathcal{D} \setminus z_i} \frac{1}{ {N - 1 \choose |S|}} \big[\nu(S \cup \{z_i\}) - \nu(S)\big]
\end{equation}

Intuitively, the SV measures the
marginal improvement of utility attributed
to the data point $z_i$, averaged over all possible subsets of data points. Calculating exact SVs requires exponentially many utility evaluations. This poses a radical challenge to using the SV for data valuation--{\em how can we compute the SV efficiently and scale to millions or even billions of data points?} This scale is rare to the previous applications of the SV but is not uncommon for real-world data valuation tasks.
%

To tackle this challenge, we focus on 
a specific family of ML models which restrict the
class of utility functions $\nu(\cdot)$ that
we consider. Specifically, we study
$K$-nearest neighbors ($K$NN) classifiers~\cite{dudani1976distance}, a simple yet popular supervised learning method used in image recognition~\cite{hays2015large}, recommendation systems~\cite{adeniyi2016automated}, healthcare~\cite{li2012using}, etc.
Given a test set,
we focus on a natural utility function, called the {\em $K$NN utility}, which, intuitively, 
measures the boost 
of the likelihood that $K$NN assigns the correct label to each test data point.
When $K = 1$, this utility
is the same as the test accuracy. 
Although some of our techniques also
apply to a broader class of 
utility functions (See Section~\ref{section:extension}),
the $K$NN utility is our
main focus.

\begin{figure}[t!]
\centering
\caption{Time complexity for computing the SV for $K$NN models. $N$ is the total number of training data points. $M$ is the number of data contributors. $h(\epsilon,K)<1$ if $K^* = \max\{1/\epsilon,K\}< C$ for some dataset-dependent constant $C$.}
\label{fig:summary_result}
\begin{tabular}{ccc}
\toprule
                                        & \textbf{Exact} & \textbf{Approximate}                                   \\ \toprule
\textbf{Baseline}                       & $2^NN\log N$          & $\frac{N^2}{\epsilon^2} \log N \log \frac{N}{\delta}$           \\ \midrule
\textbf{Unweighted $K$NN classifier}    & $N\log N$      & $N^{h(\epsilon,K)}\log N\log \frac{K^*}{\delta}$ \\
\textbf{Unweighted $K$NN regression}    & $N\log N$      & ---                                                    \\
\textbf{Weighted $K$NN}                 & $N^K$          & $\frac{N}{\epsilon^2}\log K \log \frac{K}{\delta}$     \\
\textbf{Multiple-data-per-curator $K$NN} & $M^K$          & $\frac{N}{\epsilon^2}\log K \log \frac{K}{\delta}$     \\ \toprule
\end{tabular}
\end{figure}

\noindent
The contribution of this work is a collection of
novel algorithms for efficient data valuation 
within the above scope. Figure~\ref{fig:summary_result}
summarizes our technical results.
Specifically, we made four technical
contributions:

\paragraph{\textbf{Contribution 1: Data Valuation
for $K$NN Classifiers}} The main challenge
of adopting the SV for data valuation
is its computational complexity --- for general,
bounded utility functions, calculating the SV
requires $O(2^N)$ utility evaluations for $N$ data points. Even getting an 
$(\epsilon, \delta)$-approximation (error bounded by $\epsilon$ with probability at least $1-\delta)$
for all data points 
requires
$O(N\log N)$ utility evaluations 
using state-of-the-art methods (See Section~\ref{section:baseline}). For the $K$NN utility, each utility evaluation requires to sort the training data, which has asymptotic complexity $O(N\log N)$.

\noindent
{\bf C1.1 Exact Computation}
We first propose a novel algorithm 
specifically designed for $K$NN classifiers. 
We observe that the
$K$NN utility satisfies what we call the
{\em piecewise utility difference} property:
the
difference in the marginal contribution
of two data points $z_i$ and $z_j$ over
has a ``piecewise form'' (See Section~\ref{section:exact_shapley}):
\begin{align*}
U(S \cup \{z_i\}) - U(S \cup \{z_j\}) = 
\sum_{t=1}^{T} C_{i,j}^{(t)} \mathbbm{1}[S \in \mathcal{S}_t], \forall S \in \mathcal{D} \backslash \{z_i, z_j\}
\end{align*}
where $\mathcal{S}_t \subseteq 2^{\mathcal{D} \backslash \{z_i, z_j\}}$ and $C_{i,j}^{(t)} \in \mathbb{R}$. This combinatorial structure allows us to design a very efficient algorithm
that only has
$O(N\log N)$ complexity
for {\em exact} computation of SVs
on all $N$ data points. This is an exponential
improvement over the $O(2^NN\log N)$ 
baseline!

\noindent
{\bf C1.2 Sublinear Approximation} The exact computation requires to sort the entire training set for each test point, thus becoming time-consuming for large and high-dimensional datasets. 
Moreover, in some applications such as document retrieval, test points could arrive sequentially and the values of each training point needs to get updated and accumulated on the fly, which makes it impossible to complete sorting offline.
Thus, we investigate whether higher efficiency can be achieved by finding approximate SVs instead.
We study the problem of getting $(\epsilon,\delta)$-approximation of the SVs for the $K$NN utility. This happens to be reducible to the problem of answering approximate $\max\{K,1/\epsilon\}$-nearest neighbor queries with probability $1-\delta$.
We designed a novel algorithm by taking advantage of LSH, which 
only requires $O(N^{h(\epsilon,K)} \log N)$ computation where $h(\epsilon,K)$ is dataset-dependent and typically less than $1$ when $\epsilon$ is not too small and $K$ is not too large.

\noindent
{\bf Limitation of LSH} The $h(\epsilon,K)$ term monotonically increases with $\max\{\frac{1}{\epsilon},K\}$. In experiments, we found that the LSH can handle mild error requirements (e.g., $\epsilon=0.1$) but appears to be less efficient than the exact calculation algorithm for stringent error requirements. Moreover, we can extend the exact algorithm to cope with $K$NN regressors and other scenarios detailed in Contribution 2; however, the application of the LSH-based approximation is still confined to the classification case.

To our best knowledge, the above results are one of the very first studies of efficient SV evaluation designed specifically for 
utilities arising from ML applications.

\paragraph{\textbf{Contribution 2: Extensions}} 
Our second contribution is to extend 
our results to different settings beyond
a standard $K$NN classifier and the $K$NN
utility (Section~\ref{section:extension}). Specifically, we studied:

\noindent
{\bf C2.1} Unweighted $K$NN regressors.\\
\noindent
{\bf C2.2} Weighted $K$NN classifiers and regressors.\\
\noindent
{\bf C2.3} One ``data curator'' contributes multiple
data points {\em and} has the freedom to delete all
data points at the same time.\\
\noindent
 {\bf C2.4} One ``data analyst'' provides ML analytics and 
the system attaches value to both the analyst and
data curators.\\

The connection between different settings are illustrated in Figure~\ref{fig:connection}, where each vertical layer represents a different slicing to the data valuation problem. In some of these scenarios, we successfully designed
algorithms that are as efficient as the one
for $K$NN classifiers. In some other cases, including weigthed $K$NN and the multiple-data-per-curator setup, the exact computation algorithm is less practical
although being improved exponentially. 


\begin{figure}[t!]
\centering
\includegraphics[width=0.9\columnwidth]{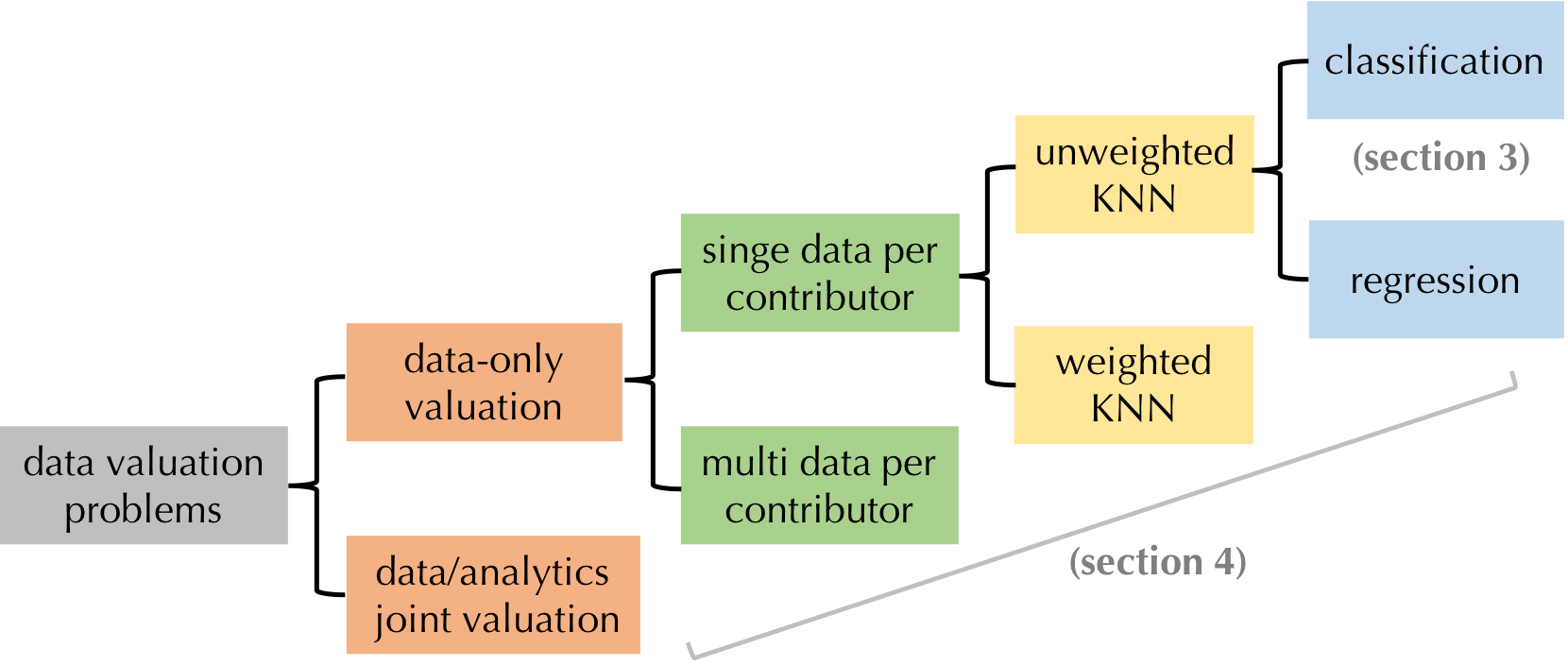}
\caption{Classification of data valuation problems.}
\label{fig:connection}
\end{figure}





\paragraph{\textbf{Contribution 3: Improved Monte Carlo Approximation for $K$NN}} To further improve the efficiency in the less efficient cases, we strengthen the sample complexity bound of the state-of-the-art approximation algorithm,
achieving an $\OO(N \log^2 N / \log^2 K)$ complexity improvement over the state-of-the-art. Our algorithm requires in total $\OO(N/\epsilon^2 \log^2 K)$ computation and is
often practical for reasonable $\epsilon$.




\paragraph{\textbf{Contribution 4: Implementation and Evaluation}}
We implement our 
algorithms and evaluate them 
on datasets up to ten million data points. We observe that our exact SV calculation algorithm can provide up to
three orders of magnitude 
speed-up over the
state-of-the-art Monte Carlo approximation
approach. With the LSH-based approximation method, we can accelerate the SV calculation even further by allowing approximation errors. The actual performance improvement of the LSH-based method over the exact algorithm depends the dataset as well as the error requirements. For instance, on a {\em 10M subset}
of the Yahoo Flickr Creative Commons 100M dataset, 
we observe that the LSH-based method can bring another $4.6\times$ speed-up.


Moreover, to our best knowledge,
this work is also one of the first
papers to evaluate data valuation
at scale. We make our datasets 
publicly available and document
our evaluation methodology in details, with the hope to facilitate future
research on data valuation.

\paragraph{\textbf{Relationship with Our Previous Work}}
Unlike this work which focuses on $K$NN, our previous work~\cite{jia2018shapley} considered some generic properties of ML models, such as boundedness of the utility functions, stability of the learning algorithms, etc, and studied their implications for computing the SV. Also, the algorithms presented in our previous work only produce approximation to the SV. When the desired approximation error is small, these algorithms may still incur considerable computational costs, thus not able to scale up to large datasets. In contrast, this paper presents a scalable algorithm that can calculate the exact SV for $K$NN.
The rest of this
paper is organized as follows.
We provide
background information in Section~\ref{section:preliminary},
and present our efficient 
algorithms for $K$NN classifiers
in Section~\ref{section:valuation_classifier}. We discuss
the extensions in Section~\ref{section:extension} and propose a Monte Carlo approximation algorithm in Section~\ref{section:improved_mc}, which significantly boosts the efficiency for the extensions that have less practical exact algorithms. We evaluate our approach in Section~\ref{section:experiment}. We discuss the integration with
real-world applications in
Section~\ref{section:discussion} and present a survey of related work in Section~\ref{section:related_work}. 

\section{Preliminaries}
\label{section:preliminary}
We present the setup of the data marketplace and introduce the framework for data valuation based on the SV. We then discuss a baseline algorithm to compute the SV.

\subsection{Data Valuation based on the SV}
We consider two types of agents that interact in a data marketplace: the sellers (or data curators) and the buyer. Sellers provide training data instances, each of which is a pair of a feature vector and the corresponding label. The buyer is interested in analyzing the training dataset aggregated from various sellers and producing an ML model, which can predict the labels for unseen features. The buyer pays a certain amount of money which depends on the utility of the ML model. Our goal is to distribute the payment fairly between the sellers.
A natural way to tackle the question of revenue allocation is to view ML as a cooperative game and model each seller as a player. This game-theoretic viewpoint allows us to formally characterize the ``power'' of each seller and in turn determine their deserved share of the revenue. For ease of exposition, we assume that each seller contributes one data instance in the training set; later in Section~\ref{section:extension}, we will discuss the extension to the case where a seller contributes multiple data instances. 

Cooperative game theory studies the behaviors of coalitions formed by game players. Formally, a cooperative game is defined by a pair $(I,\nu)$, where $I=\{1,\ldots,N\}$ denotes the set of all players and $\nu:2^N\rightarrow \mathbb{R}$ is the utility function, which maps each possible coalition to a real number that describes the utility of a coalition, i.e., how much collective payoff a set of players can gain by forming the coalition.
One of the fundamental questions in cooperative game theory is to characterize how important each player is to the overall cooperation. The SV~\cite{shapley1953value} is a classic method to distribute the total gains generated by the coalition of all players. The SV of player $i$ with respect to the utility function $\nu$ is defined as the average marginal contribution of $i$ to coalition $S$ over all $S\subseteq I\setminus \{i\}$: 
\begin{align}
\label{eqn:shapley_definition_no_order}
s(\nu,i) = \frac{1}{N} \sum_{S\subseteq I\setminus\{i\}} \frac{1}{{N-1 \choose |S|}}
\big[\nu(S\cup \{i\})-\nu(S)\big]
\end{align}
We suppress the dependency on $\nu$ when the utility is self-evident and use $s_i$ to represent the value allocated to player $i$.

The formula in (\ref{eqn:shapley_definition_no_order}) can also be stated in the equivalent form: 
\begin{align}
\label{eqn:shapley_definition_order}
    s_i = 
\frac{1}{N!}\sum_{\pi \in \Pi(I)}\big[ \nu(P_i^\pi\cup \{i\}) - \nu(P_i^\pi)\big]
\end{align}
where $\pi \in \Pi(I)$ is a permutation of players and $P_i^\pi$ is the set of players which precede player $i$ in $\pi$. Intuitively, imagine all players join a coalition in a random order, and that every player $i$ who has joined receives the marginal contribution that his participation would bring to those already in the coalition. To
calculate $s_i$, we average these contributions over all the possible orders.  

Transforming these game theory concepts to data valuation, one can think of the players as training data instances and the utility function $\nu(S)$ as a performance measure of the model trained on the set of training data $S$. The SV of each training point thus measures its importance to learning a performant ML model. The following desirable properties that the SV \textit{uniquely} possesses motivate us to adopt it for data valuation.
\begin{enumerate}[leftmargin=*,label=\roman*]
    \item {\bf Group Rationality}: The value of the entire training dataset is completely distributed among all sellers, i.e., $\nu(I) = \sum_{i\in I} s_i$.
    
    \item {\bf Fairness}: (1) Two sellers who are identical with respect to what they contribute to a dataset's utility should have the same value. That is, if seller $i$ and $j$ are equivalent in the sense that $\nu(S\cup \{i\}) = \nu(S\cup \{j\}),\forall S\subseteq I\setminus \{i,j\}$, then $s_i=s_j$. (2) Sellers with zero marginal contributions to all subsets of the dataset receive zero payoff, i.e., $s_i=0$ if $\nu(S\cup \{i\})=0$ for all $S\subseteq I\setminus\{i\}$.

    \item {\bf Additivity}: The values under multiple utilities sum up to the value under a utility that is the sum of all these utilities: $s(\nu_1,i) + s(\nu_2,i) = s(\nu_1+\nu_2,i)$ for $i\in I$.
\end{enumerate}

The \emph{group rationality} property states that any rational group of sellers would expect to distribute the full yield of their coalition. The \emph{fairness} property requires that the names of the sellers play no role in determining the value, which should be sensitive only to how the utility function responds to the presence of a seller. The \emph{additivity} property facilitates efficient value calculation when the ML model is used for multiple applications, each of which is associated with a specific utility function. With additivity, one can decompose a given utility function into an arbitrary sum of utility functions and compute value shares separately, resulting in transparency and decentralizability. The fact that the SV is the only value division scheme that meets these desirable criteria, combined with its flexibility to support different utility functions, leads us to employ the SV to attribute the total gains generated from a dataset to each seller. 

In addition to its theoretical soundness, our previous work~\cite{jia2018shapley} empirically demonstrated that the SV also coincides with people's intuition of data value. For instance, noisy images tend to have lower SVs than the high-fidelity ones; the training data whose distribution is closer to the test data distribution tends to have higher SVs. These empirical results further back up the use of the SV for data valuation. For more details, we refer the readers to~\cite{jia2018shapley}.

\subsection{A Baseline Algorithm}
\label{section:baseline}
One challenge of applying SV is its
computational complexity.
Evaluating the exact SV using Eq.~(\ref{eqn:shapley_definition_no_order}) involves computing the marginal utility of every user to every coalition, which is $\mathcal{O}(2^N)$. Such exponential computation is clearly impractical for valuating a large number of training points. Even worse, in many ML tasks, evaluating the utility function \textit{per se} (e.g., testing accuracy) is computationally expensive as it requires training a ML model. For large datasets, the only feasible approach currently in the literature is Monte Carlo (MC) sampling~\cite{maleki2015addressing}. In this paper, we will use it as a baseline for evaluation. 

The central idea behind the baseline algorithm is to regard the SV definition in (\ref{eqn:shapley_definition_order}) as the expectation of a training instance's marginal contribution over a random permutation and then use the sample mean to approximate it. More specifically, let $\pi$ be a random permutation of $I$ and each permutation has a probability of $1/N!$. Consider the random variable $\phi_i = \nu(P_i^{\pi}\cup\{i\}) - \nu(P_i^{\pi})$. By (\ref{eqn:shapley_definition_order}), the SV $s_i$ is equal to $\E[\phi_i]$. Thus, 
\begin{align}
\label{eqn:permutation_estimator}
    \hat{s}_i = \frac{1}{T}\sum_{t=1}^T \nu(P_i^{\pi_t}\cup\{i\})-\nu(P_i^{\pi_t})
\end{align}
is a consistent estimator of $s_i$, where $\pi_t$ be $t$th sample permutation uniformly drawn from all possible permutations $\Pi(I)$. 

We say that $\hat{s}\in \R^N$ is an $(\epsilon,\delta)$-approximation to the true SV $s=[s_1,\cdots,s_N]^T\in \R^N$ if $P[\max_i |\hat{s}_i-s_i| \leq \epsilon]\geq 1-\delta$. Let $r$ be the range of utility differences $\phi_i$. By applying the Hoeffding's inequality, \cite{maleki2013bounding} shows that for general, bounded utility functions, the number of permutations $T$ needed to achieve an $(\epsilon,\delta)$-approximation is $\frac{r^2}{2\epsilon^2}\log \frac{2N}{\delta}$.
For each permutation, the baseline algorithm evaluates the utility function for $N$ times in order to compute the SV for $N$ training instances; therefore, the total utility evaluations involved in the baseline approach is $\OO(N\log N)$. In general, evaluating $\nu(S)$ in the ML context requires to re-train the model on the subset $S$ of the training data. Therefore, despite its improvements over the exact SV calculation, the baseline algorithm is not efficient for large datasets. 

Take the $K$NN classifier as an example and assume that $\nu(\cdot)$ represents the testing accuracy of the classifier. Then, evaluating $\nu(S)$ needs to sort the training data in $S$ according to their distances to the test point, which has $\OO(|S|\log |S|)$ complexity. Since on average $|S|=N/2$, the asymptotic complexity of calculating the SV for a $K$NN classifier via the baseline algorithm is $\OO(N^2\log^2 N)$, which is prohibitive for large-scale datasets. In the sequel, we will show that it is indeed possible to develop much more efficient algorithms to compute the SV by leveraging the locality of $K$NN models.

\section{Valuing Data for KNN Classifiers}
\label{section:valuation_classifier}

In this section, we present an algorithm that can calculate the exact SV for $K$NN classifiers in quasi-linear time. Further, we exhibit an approximate algorithm based on LSH that could achieve sublinear complexity.

\subsection{Exact SV Calculation}
\label{section:exact_shapley}
$K$NN algorithms are popular supervised learning methods, widely adopted in a multitude of applications such as computer vision, information retrieval, etc. 
Suppose the dataset $D$ consisting of pairs $(x_1,y_1)$, $(x_2,y_2)$, $\ldots$, $(x_N,y_N)$ taking values in $\mathcal{X}\times \mathcal{Y}$, where $\mathcal{X}$ is the feature space and $\mathcal{Y}$ is the label space. Depending on whether the nearest neighbor algorithm is used for classification or regression, $\mathcal{Y}$ is either discrete or continuous. The training phase of $K$NN consists only of storing the features and labels in $D$. The testing phase is aimed at finding the label for a given query (or test) feature. This is done by searching for the $K$ training features most similar to the query feature and assigning a label to the query according to the labels of its $K$ nearest neighbors. 
Given a single testing point $x_\text{test}$ with the label $y_\text{test}$, the simplest, unweighted version of a $K$NN classifier first finds the top-$K$ training points $(x_{\alpha_1},\cdots,x_{\alpha_K})$ that are most similar to $x_\text{test}$ and outputs the probability of $x_\text{test}$ taking the label $y_\text{test}$ as $P[x_\text{test} \rightarrow y_\text{test}]=\frac{1}{K}\sum_{k = 1}^{K} \mathbbm{1}[y_{\alpha_k} = y_\text{test}]$, where $\alpha_k$ is the index of the $k$th nearest neighbor.

One natural way to define the utility of a $K$NN classifier
is by the likelihood of the right label:
\begin{align}
\label{eqn:utility_classification_unweighted}
     \nu(S) =\frac{1}{K} \sum_{k=1}^{\min\{K,|S|\}} \mathbbm{1}[y_{\alpha_k(S)} = y_\text{test}]
\end{align}
where $\alpha_k(S)$ represents the index of the training feature that is $k$th closest to $x_\text{test}$ among the training examples in $S$. Specifically, $\alpha_k(I)$ is abbreviated to $\alpha_k$.

Using this utility function, we can derive an efficient, but exact way of computing the SV. 
\begin{theorem}
\label{thm:KNN_unweighted_class}
Consider the utility function in (\ref{eqn:utility_classification_unweighted}). Then, the SV of each training point can be calculated recursively as follows:
\begin{align}
\label{eqn:KNN_unweighted_class}
&s_{\alpha_N}=\frac{\mathbbm{1}[y_{\alpha_{N}} = y_\text{test}]}{N}
\\
& s_{\alpha_i} = s_{\alpha_{i+1}}\!\! +  \frac{\mathbbm{1}[y_{\alpha_i} = y_\text{test}] - \mathbbm{1}[y_{\alpha_{i+1}} = y_\text{test}]}{K}
    \frac{\min\{K, i\}}{i}
\end{align}
\end{theorem}

Note that the above result for a single test point can be readily extended to the multiple-test-point case, in which the utility function is defined by
\begin{align}
\label{eqn:knn_utility_multiple_test}
 \nu(S) = \frac{1}{N_\text{test}}\sum_{j=1}^{N_\text{test}}\frac{1}{K}\sum_{k=1}^{\min\{K,|S|\}} \mathbbm{1}[y_{\alpha_k^{(j)}(S)} = y_{\text{test},j}]
\end{align}
where $\alpha_k^{(j)}(S)$ is the index of the $k$th nearest neighbor in $S$ to $x_{\text{test},j}$. By the additivity property, the SV for multiple test points is the average of the SV for every single test point. The pseudo-code for calculating the SV for an unweighted $K$NN classifier is presented in Algorithm~\ref{alg:knn_shapley}.
The computational complexity is only $\mathcal{O}(N \log N N_\text{test})$ for $N$ training
data points and $N_\text{test}$ test data points---this is simply to sort $N_\text{test}$ arrays of $N$ numbers! 

\begin{algorithm}[h!]
\SetAlgoLined
\SetKwInOut{Input}{input}
\SetKwInOut{Output}{output}
\Input{Training data $D=\{(x_i,y_i)\}_{i=1}^N$, test data $D_\text{test}=\{(x_{\text{test},i},y_{\text{test},i})\}_{i=1}^{N_\text{test}}$}
\Output{The SV $\{s_i\}_{i=1}^N$}
\For{$j\gets 1$ \KwTo $N_\text{test}$}{
    $(\alpha_1,...,\alpha_N) \leftarrow$ Indices of training data in an ascending order using $d(\cdot,x_\text{test})$\;
    
    $s_{j,\alpha_N}\leftarrow\frac{\mathbbm{1}[y_{\alpha_N}=y_\text{test}]}{N}$\;
    
     \For{$i\gets N-1$ \KwTo $1$ }{
     $s_{j,\alpha_i} \leftarrow  s_{j,\alpha_{i+1}} + \frac{\mathbbm{1}[y_{\alpha_i} = y_{\text{test},j}]-\mathbbm{1}[y_{\alpha_{i+1}} = y_{\text{test},j}]}{K} \frac{\min\{K,i\}}{i} $\;
     }
}
\For{$i\gets 1$ \KwTo $N$}{
$s_i \gets \frac{1}{N_\text{test}}\sum_{j=1}^{N_\text{test}} s_{j,i}$\;
}
 \caption{Exact algorithm for calculating the SV for an unweighted $K$NN classifier.}
 \label{alg:knn_shapley}
\end{algorithm}

The proof of Theorem~\ref{thm:KNN_unweighted_class} relies on the following lemma, which states that the difference in the utility gain induced by either point $i$ or point $j$ translates linearly to the difference in the respective SVs.
\begin{Lm}
\label{lm:shapley_diff}
For any $i,j\in I$, the difference in SVs between $i$ and $j$ is 
\begin{align}
\label{eqn:shapley_diff}
    s_i-s_j = \frac{1}{N-1}  \sum_{S\subseteq I\setminus\{i,j\}} \frac{\nu(S\cup\{i\}) - \nu(S\cup \{j\})}{\binom{N-2}{|S|}}
\end{align}
\end{Lm}

\begin{proof}[Proof of Theorem \ref{thm:KNN_unweighted_class}]
W.l.o.g., we
assume that $x_1,\ldots, x_n$ are
sorted according to their similarity to $x_\text{test}$, that is, $x_i = x_{\alpha_i}$. For any given subset $S\subseteq I\setminus\{i,{i+1}\}$ of size $k$, we split the subset into two disjoint sets $S_1$ and $S_2$ such that $S = S_1 \cup S_2$ and  $|S_1| + |S_2| = |S| = k $. Given two neighboring points with indices $i,{i+1}\in I$, we constrain $S_1$ and $S_2$ to $S_1 \subseteq \{1,...,{i-1}\}$
and $S_2 \subseteq \{{i+2},...,{N}\}$.

Let $s_i$ be the SV of data point $x_i$. By Lemma~\ref{lm:shapley_diff}, we can draw conclusions about the SV difference $s_i - s_{i+1}$ by inspecting the utility difference $\nu(S\cup\{i\}) - \nu(S\cup\{i+1\})$ for any $S\subseteq I\setminus\{i,i+1\}$. We analyze $\nu(S\cup\{i\}) - \nu(S\cup\{i+1\})$ by considering the following cases.

\textbf{(1) $|S_1|\geq K$.} In this case, we know that $i,i+1 > K$ and therefore $\nu(S \cup \{i\}) = \nu(S \cup \{i+1\})= \nu(S)$, hence $\nu(S\cup\{i\}) - \nu(S\cup\{i+1\})=0$. 

\textbf{(2) $|S_1|< K$.} In this case, we know that $i \leq K$ and therefore $\nu(S \cup \{i\}) -\nu(S)$ might be nonzero. Note that including a point $i$ into $S$ can only expel the $K$th nearest neighbor from the original set of $K$ nearest neighbors. Thus, $\nu(S\cup\{i\})-\nu(S) = \frac{1}{K}(\mathbbm{1}[y_{i} = y_\text{test}] - \mathbbm{1}[y_K = y_\text{test}])$. The same hold for the inclusion of point $i+1$: $\nu(S\cup\{i+1\})-\nu(S) = \frac{1}{K}(\mathbbm{1}[y_{i+1} = y_\text{test}] - \mathbbm{1}[y_K = y_\text{test}])$. Combining the two equations, we have 
\begin{align*}
    \nu(S\cup\{i\}) - \nu(S\cup\{i+1\}) =\frac{\mathbbm{1}[y_{i} = y_\text{test}] - \mathbbm{1}[y_{i+1} = y_\text{test}]}{K} 
\end{align*}

Combining the two cases discussed above and applying Lemma~\ref{lm:shapley_diff}, we have
\begin{align}
        &s_i - s_{i+1} \nonumber\\
    =&
    \frac{1}{N-1} \sum_{k=0}^{N-2} \frac{1}{\binom{N-2}{k}}\!\!\!\!\!\!\!\
    \sum \limits_{ \substack{  S_1 \subseteq \{1,...,i-1\}, \\ S_2 \subseteq \{i+2,...,N\}: \\ |S_1| + |S_2|=k, |S_1| < K} }\!\!\!\!\!\!\!\!\!\!\!\!\!\!\!\!\!\! \frac{\mathbbm{1}[y_i = y_\text{test}] - \mathbbm{1}[y_{i+1} = y_\text{test}]}{K}  \nonumber\\
    =&
    \label{eqn:dev_mid}
    \frac{\mathbbm{1}[y_i = y_\text{test}] - \mathbbm{1}[y_{i+1} = y_\text{test}]}{K}
    \times
    \frac{1}{N-1} \sum_{k=0}^{N-2} \frac{1}{\binom{N-2}{k}} 
    \sum_{m = 0}^{\min(K-1,k)} \binom{i-1}{m}  \binom{N-i-1}{k-m}
\end{align}
The sum of binomial coefficients in (\ref{eqn:dev_mid}) can be simplified as follows:
\begin{align}
    & \sum_{k=0}^{N-2} \frac{1}{\binom{N-2}{k}} 
    \sum_{m = 0}^{\min\{K-1,k\}} \binom{i-1}{m}  \binom{N-i-1}{k-m}\\
    & = \sum_{m=0}^{\min\{K-1, i-1\}} \sum_{k'=0}^{N-i-1} \frac{\binom{i-1}{m} \binom{N-i-1}{k'}}   {\binom{N-2}{m+k'}} \\
    & = \frac{ \min\{K, i\} (N-1) }{i}
\end{align}
where the first equality is due to the exchange of the inner and outer summation and the second one is by taking $v= N-i-1$ and $u = i-1$ in the binomial identity $\sum_{j=0}^v \frac{{u\choose i}{v\choose j}}{{u+v\choose i+j}} = \frac{u+v+1}{u+1} $.

Therefore, we have the following recursion
\begin{align}
    s_i - s_{i+1} =
\frac{\mathbbm{1}[y_i = y_\text{test}] - \mathbbm{1}[y_{i+1} = y_\text{test}]}{K}
    \frac{\min\{K, i\}}{i}
\end{align}

Now, we analyze the formula for $s_N$, the starting point of the recursion. Since $x_N$ is farthest to $x_\text{test}$ among all training points, $x_N$ results in non-zero marginal utility only when it is added to the subsets of size smaller than $K$. Hence, $s_N$ can be written as
\begin{align}
    s_N &= \frac{1}{N} \sum_{k=0}^{K-1} \frac{1}{{N-1\choose k}} \sum_{|S|=k, S\subseteq I\setminus\{N\}} \nu(S\cup N) - \nu(S)\\
    &=\frac{1}{N}\sum_{k=0}^{K-1} \frac{1}{{N-1\choose k}} \sum_{|S|=k, S\subseteq I\setminus\{N\}} \frac{\mathbbm{1}[y_N=y_\text{test}]}{K}\\
    &= \frac{\mathbbm{1}[y_N=y_\text{test}]}{N}
\end{align}

\end{proof}


\subsection{LSH-based Approximation}
The exact calculation of the $K$NN SV for a query instance requires to sort the entire training dataset, and has computation complexity $\OO(N_\text{test}(Nd+N\log(N)))$, where $d$ is the feature dimension. Thus, the exact method becomes expensive for large and high-dimensional datasets. We now present a sublinear algorithm to approximate the $K$NN SV for classification tasks.

The key to boosting efficiency is to realize that only $\OO(1/\epsilon)$ nearest neighbors are needed to estimate the $K$NN SV with up to $\epsilon$ error. Therefore, we can avert the need of sorting the entire database for every new query point. 
\begin{theorem}
\label{thm:knn_shapley_approx}
Consider the utility function defined in (\ref{eqn:utility_classification_unweighted}). Consider $\{\hat{s}_i\}_{i=1}^N$ defined recursively by
\begin{align}
\label{eqn:knn_recursion_approx_1}
    &\hat{s}_{\alpha_i} = 0 \quad \quad \quad \text{   if $i\geq K^*$}\\
    \label{eqn:knn_recursion_approx_2}
    &\hat{s}_{\alpha_i} = \hat{s}_{\alpha_{i+1}} + \frac{\mathbbm{1}[y_{\alpha_{i}}=y_\text{test}]-\mathbbm{1}[y_{\alpha_{i+1}}=y_\text{test}]}{K} \frac{\min\{K,i\}}{i} \quad\quad\quad\text{   if $i\leq K^*-1$}
\end{align}
where $K^* = \max\{K,\ceil*{1/\epsilon}\}$ for some $\epsilon>0$. Then, $[\hat{s}_{\alpha_1}$,$\ldots$, $\hat{s}_{\alpha_N}]$ is an $(\epsilon,0)$-approximation to the true SV $[s_{\alpha_1}$,$\ldots$, $s_{\alpha_N}]$ and $\hat{s}_i-\hat{s}_{i+1} = s_i-s_{i+1}$ for $i\leq K^*-1$.
\end{theorem}

Theorem~\ref{thm:knn_shapley_approx} indicates that we only need to find $\max\{K,\ceil*{1/\epsilon}\}$($\triangleq K^*$) nearest neighbors to obtain an $(\epsilon,0)$-approximation. Moreover, since $\hat{s}_i-\hat{s}_{i+1} = s_i-s_{i+1}$ for $i\leq K^*-1$, the approximation retains the original value rank for $K^*$ nearest neighbors.

The question on how to efficiently retrieve nearest neighbors to a query in large-scale databases has been studied extensively in the past decade. 
Various techniques, such as the kd-tree~\cite{mount1998ann}, LSH~\cite{datar2004locality}, have been proposed to find approximate nearest neighbors.
Although all of these techniques can potentially help improve the efficiency of the data valuation algorithms for $K$NN, we focus on LSH in this paper, as it was experimentally shown to achieve large speedup over several tree-based data structures~\cite{gionis1999similarity,har2012approximate,datar2004locality}. 
In LSH, every training instance $x$ is converted into codes in each hash table by using a series of hash functions $h_j(x)$, $j=1,\ldots,m$. Each hash function is designed to preserve the relative distance between different training instances; similar instances have the same hashed value with high probability. Various hash functions have been proposed to approximate $K$NN under different distance metrics~\cite{charikar2002similarity,datar2004locality}. 
We will focus on the distance measured in $l_2$ norm; in that case, a commonly used hash function is $h(x)=\floor*{\frac{w^Tx+b}{r}}$, where $w$ is a vector with entries sampled from a $p$-stable distribution, and $b$ is uniformly chosen from the range $[0,r]$. It is shown in~\cite{datar2004locality}: 
\begin{align}
    P[h(x_i)=h(x_\text{test})] = f_h(\|x_i-x_\text{test}\|_2)
\end{align}
where the function $f_h(c) = \int_0^r \frac{1}{c}f_2(\frac{z}{c})(1-\frac{z}{r})dz$ is a monotonically decreasing with $c$. Here, $f_2$ is the probability density function of the absolute value of a $2$-stable random variable.

We now present a theorem which relates the success rate of finding approximate nearest neighbors to the intrinsic property of the dataset and the parameters of LSH.

\begin{theorem}
\label{thm:lsh_complexity}
LSH with $\OO(d\log(N) N^{g(C_K\!)} \log \frac{K}{\delta})$ time complexity, $\OO(Nd+N^{g(C_K\!)\!+\!1} \log \frac{K}{\delta})$ space complexity, and $\OO(N^{g(C_K\!)}$ $\log \frac{K}{\delta})$ hash tables can find the exact $K$ nearest neighbors with probability $1-\delta$, where $g(C_K)=\log f_h(1/C_K)/\log f_h(1)$ is a monotonically decreasing function. $C_K=D_\text{mean}/D_K$, where $D_\text{mean}$ is the expected distance of a random training instance to a query $x_\text{test}$ and $D_K$ is the expected distance between $x_\text{test}$ to its $K$th nearest neighbor denoted by $x_{\alpha_i}(x_\text{test})$, i.e.,
\begin{align}
    &D_\text{mean} = \E_{x,x_\text{test}}[D(x,x_\text{test})]\\
    &D_K = \E_{x_\text{test}}[D(x_{\alpha_i}(x_\text{test}),x_\text{test}]
\end{align}
\end{theorem}

The above theorem essentially extends the $1$NN hardness analysis in Theorem 3.1 of~\cite{he2012difficulty} to $K$NN. $C_K$ measures the ratio between the distance from a query instance to a random training instance and that to its $K$th nearest neighbor. We will hereinafter refer to $C_K$ as \emph{$K$th relative contrast}. Intuitively, $C_K$ signifies the difficulty of finding the $K$th nearest neighbor. A smaller $C_K$ implies that some random training instances are likely to have the same hashed value as the $K$th nearest neighbor, thus entailing a high computational cost to differentiate the true nearest neighbors from the false positives. Theorem~\ref{thm:lsh_complexity} shows that among the datasets of the same size, the one with higher relative contrast will need lower time and space complexity and fewer hash tables to approximate the $K$ nearest neighbors. 
Combining Theorem~\ref{thm:knn_shapley_approx} and Theorem~\ref{thm:lsh_complexity}, we obtain the following theorem that explicates the tradeoff between $K$NN SV approximation errors and computational complexity.

\begin{theorem}
\label{thm:lsh_approx_sp}
Consider the utility function defined in (\ref{eqn:knn_utility_multiple_test}). Let $\hat{x}_{\alpha_k^{(j)}}$ denote the $k$th closest training point to $x_{\text{test},j}$ output by LSH with $\OO(\!N_\text{test}d\log(N) N^{g(C_{K^*}\!)} \!\log\! \frac{N_\text{test}K^*}{\delta}\!)$ time complexity, $\OO(\!Nd+N^{g(C_{K^*}\!)\!+\!1} \!\log\! \frac{N_\text{test}K^*}{\delta}\!)$ space complexity, and $\OO(\!N^{g(C_{K^*}\!)}\!\log\!\frac{N_\text{test}K^*}{\delta}\!)$ hash tables, where $K^*=\max(K,\ceil*{1/\epsilon})$. Suppose that  $\{\hat{s}_{i}\}_{i=1}^N$ is computed via $\hat{s}_i = \frac{1}{N_\text{test}}\sum_{j=1}^{N_\text{test}} \hat{s}_{i,j}$
and $\hat{s}_{i,j}$ ($j=1,\ldots, N_\text{test}$) are defined recursively by
\begin{align}
\label{eqn:knn_recursion_approx_lsh_1}
    &\hat{s}_{\alpha_i^{(j)},j} = 0 \quad \quad \quad \text{ if $i\geq K^*$}\\
    \label{eqn:knn_recursion_approx_lsh_2}
    &\hat{s}_{\alpha_i^{(j)},j} = \hat{s}_{\alpha_{i+1}^{(j)},j}+ \frac{\mathbbm{1}[\hat{y}_{\alpha_i^{(j)}}=y_{\text{test},j}]-\mathbbm{1}[\hat{y}_{\alpha_{i+1}^{(j)}}=y_{\text{test},j}]}{K} \frac{\min\{K,i\}}{i} \quad \quad \quad\text{ if $i\leq K^*-1$}
\end{align}
where $\hat{y}_{\alpha_i^{(j)}}$ and $y_{\text{test},j}$ are the labels associated with $\hat{x}_{\alpha_i^{(j)}}$ and $x_\text{test,j}$, respectively. Let the true SV of $\hat{x}_{\alpha_k}$ be denoted by $s_{\alpha_i}$. Then, $[\hat{s}_{\alpha_1},\ldots,\hat{s}_{\alpha_N}]$ is an $(\epsilon,\delta)$-approximation to the true SV $[s_{\alpha_1},\ldots,s_{\alpha_N}]$.
\end{theorem}
The gist of the LSH-based approximation is to focus only on the SV of the retrieved nearest neighbors and neglect the values of the rest of the training points since their values are small enough. For a error requirement $\epsilon$ not too small such that $C_{K^*}>1$, the LSH-based approximation has sublinear time complexity, thus enjoying higher efficiency than the exact algorithm.

\section{Extensions}
\label{section:extension}

We extend the exact algorithm for unweighted $K$NN to other settings. Specifically, as illustrated by Figure~\ref{fig:connection}, we categorize a data valuation problem according to whether data contributors are valued in tandem with a data analyst; whether each data contributor provides a single data instance or multiple ones; whether the underlying ML model is a weighted $K$NN or unweighted; and whether the model solves a regression or a classification task. We will discuss the valuation algorithm for each of the above settings.

\paragraph{\textbf{Unweighted $K$NN Regression}} For regression tasks, we define the utility function by the negative mean square error of an unweighted $K$NN regressor:
\begin{align}
\label{eqn:utility_regression}
 U(S) = -\bigg(\frac{1}{K}\sum_{k=1}^{\min\{K,|S|\}}  y_{\alpha_k(S)} - y_\text{test}\bigg)^2
\end{align}
Using similar proof techniques to Theorem~\ref{thm:KNN_unweighted_class}, we provide a simple iterative procedure to compute the SV for unweighted $K$NN regression in Appendix~\ref{section:appendix_regression}.

\paragraph{\textbf{Weighted $K$NN}} A weighted $K$NN estimate produced by a training set $S$ can be expressed as $\hat{y}(S) = \sum_{k=1}^{\min\{K,|S|\}}w_{\alpha_k(S)} y_{\alpha_k}$,
where $w_{\alpha_k(S)}$ is the weight associated with the $k$th nearest neighbor in $S$. The weight assigned to a neighbor in the weighted $K$NN estimate often varies with the neighbor-to-test distance so that the evidence from more nearby neighbors is weighted more heavily~\cite{dudani1976distance}.
Correspondingly, we define the utility function associated with weighted $K$NN classification and regression tasks as
\begin{align}
\label{eqn:utility_classification_weighted}
    U(S) = \sum_{k=1}^{\min\{K,|S|\}} w_{\alpha_k(S)} \mathbbm{1}[y_{\alpha_k(S)} = y_\text{test}]
\end{align}
and 
\begin{align}
\label{eqn:utility_regression_weighted}
 U(S) = -\bigg(\sum_{k=1}^{\min\{K,|S|\}} w_{\alpha_k(S)} y_{\alpha_k(S)} - y_\text{test}\bigg)^2.
\end{align}


For weighted $K$NN classification and regression, the SV can no longer be computed exactly in $\mathcal{O}(N\log N)$ time. In Appendix~\ref{section:appendix_weighted_knn}, 
we present a theorem showing that it is however possible to compute the exact SV for weighted $K$NN in $\mathcal{O}(N^K)$ time. Figure~\ref{fig:illustrate_weighted_knn} illustrates the origin of the polynomial complexity result. When applying (\ref{eqn:shapley_definition_no_order}) to $K$NN, we only need to focus on the subsets whose utility might be affected by the addition of $i$th training instance. Since there are only $N^K$ possible distinctive combinations for $K$ nearest neighbors, the number of distinct utility values for all $S\subseteq I$ is upper bounded by $N^K$.

\begin{figure}
    \centering
\includegraphics[width=0.5\textwidth]{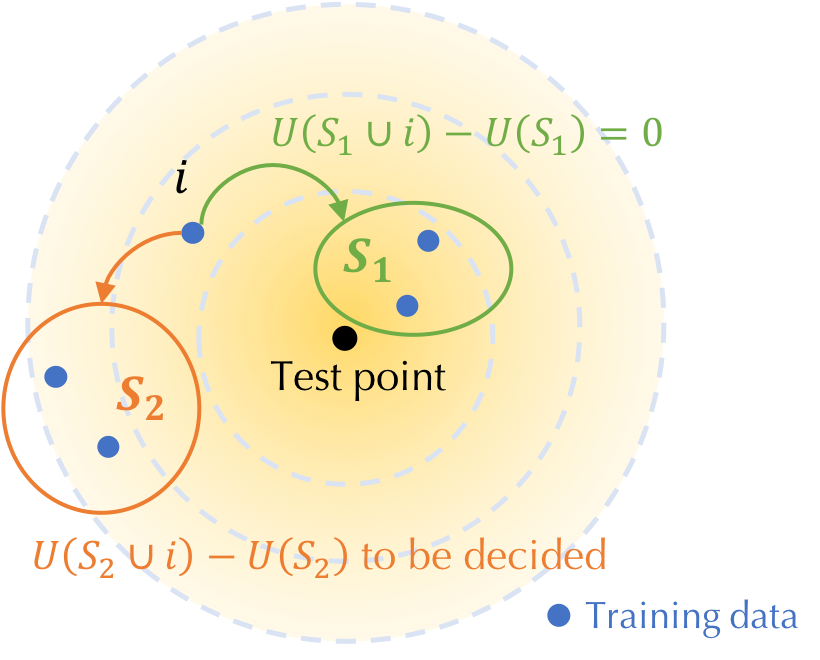}
\caption{Illustration of the idea to compute the SV for weighted $K$NN.}
    \label{fig:illustrate_weighted_knn}
\end{figure}



\paragraph{\textbf{Multiple Data Per Contributor}} We now study the case where each seller provides more than one data instance. The goal is to fairly value individual sellers in lieu of individual training points. In Appendix~\ref{section:appendix_multi_data_per_selelr}, 
we show that for both unweighted/weighted classifiers/regressors, the complexity for computing the SV of each seller is $\OO(M^K)$, where $M$ is the number of sellers. Particularly, when $K=1$, even though each seller can provision multiple instances, the utility function only depends on the training point that is nearest to the query point. Thus, for $1$NN, the problem of computing the multi-data-per-seller $K$NN SV reduces to the single-data-per-seller case; thus, the corresponding computational complexity is $\OO(M\log M)$.

\paragraph{\textbf{Valuing Computation}} Oftentimes, the buyer may outsource data analytics to a third party, which we call the analyst throughout the rest of the paper. The analyst analyzes the training dataset aggregated from different sellers and returns an ML model to the buyer. In this process, the analyst contributes various computation efforts, which may include intellectual property pertaining to data anlytics, usage of computing infrastructure, among others. Here, we want to address the problem of appraising both sellers (data contributors) and analysts (computation contributors) within a unified game-theoretic framework. 

Firstly, we extend the game-theoretic framework for data valuation to model the interplay between data and computation. The resultant game is termed a \emph{composite game}. By contrast, the game discussed previously which involves only the sellers is termed a \emph{data-only game}. In the composite game, there are $M+1$ players, consisting of $M$ sellers denoted by $I^s$ and one analyst denoted by $C$. We can express the utility function $\nu_c$ associated with the game in terms of the utility function $\nu$ in the data-only game as follows. Since in the case of outsourced analytics, both contributions from data sellers and data analysts are necessary for building models, the value of a set $S\subseteq I^s\cup \{C\}$ in the composite game is zero if $S$ only contains the sellers or the analyst; otherwise, it is equal to $\nu$ evaluated on all the sellers in $S$. Formally, we define the utility function $\nu_c$ by 
\begin{align}
\label{eqn:utility_composite}
    \nu_c(S)=\left\{\begin{array}{l}
         0,\text{ if $S=\{C\}$ or $S\subseteq I^s$}\\
         \nu(S\setminus\{C\}), \text{ otherwise}
    \end{array}
\right.
\end{align}
The goal in the composite game is to allocate $\nu_c(\{I^s,C\})$ to the individual sellers and the analyst. $s(\nu_c,i)$ and $s(\nu_c,C)$ represent the value received by seller $i$ and the analyst, respectively. We suppress the dependency of $s$ on the utility function whenever it is self-evident, denoting the value allocated to seller $i$ and the analyst by $s_i$ and $s_c$, respectively.

In Appendix~\ref{section:appendix_computation_valuation}, 
we show that one can compute the SV for both the sellers and the analyst with the same computational complexity as the one needed for the data-only game. 

\paragraph{\textbf{Comments on the Proof Techniques}} We have shown that we can circumvent the exponential complexity for computing the SV for a standard unweighted $K$NN classifier and its extensions. A natural question is whether it is possible to abstract the commonality of these cases and provide a general property of the utility function that one can exploit to derive efficient algorithms.


Suppose that some group of $S$'s induce the same $\nu(S\cup\{i\})-\nu(S\cup\{j\})$ and there only exists $T$ number of such groups. More formally, consider that $\nu(S\cup\{i\})-\nu(S\cup\{j\})$ can be represented by a ``piecewise'' form:
\begin{align}
\label{eqn:piecewise_util}
   \nu(S\cup\{i\})-\nu(S\cup\{j\}) = \sum_{t=1}^T C_{ij}^{(t)} \mathbbm{1}[S\in \mathcal{S}_t]
\end{align}
where $\mathcal{S}_t\subseteq 2^{I\setminus\{i,j\}}$ and $C_{i,j}^{(t)}\in \mathbb{R}$ is a constant associated with $t$th ``group.'' An application of Lemma~\ref{lm:shapley_diff} to the utility functions with the piecewise utility difference form indicates that the SV difference between $i$ and $j$ is
\begin{align}
    &s_i - s_j = \frac{1}{N-1}\sum_{S\subseteq I\setminus\{i,j\}} \sum_{t=1}^T \frac{ C_{ij}^{(t)}}{{N-2\choose |S|}} \mathbbm{1}[S\in \mathcal{S}_t] \\
    \label{eqn:piecewise_shapley_diff}
    &= \frac{1}{N-1}\sum_{t=1}^T  C_{ij}^{(t)} \bigg[\sum_{k=0}^{N-2} \frac{|\{S:S\in \mathcal{S}_t, |S|=k\}| }{{N-2\choose k}}\bigg]
\end{align}
With the piecewise property (\ref{eqn:piecewise_util}), the SV calculation is reduced to a counting problem. As long as the quantity in the bracket of (\ref{eqn:piecewise_shapley_diff}) can be efficiently evaluated, the SV difference between any pair of training points can be computed in $\OO(TN)$.

Indeed, one can verify that the utility function for unweighted $K$NN classification, regression and weighted $K$NN have the aforementioned ``piecewise'' utility difference property with $T=1,N-1, \sum_{k=0}^K {N-2\choose k}$, respectively. More details can be found in Appendix~\ref{section:appendix_piecewise}.

\section{Improved MC Approximation}
\label{section:improved_mc}

As discussed previously, the SV for unweighted $K$NN classification and regression can be computed exactly with $\OO(N\log N)$ complexity. However, for the variants including the weighted $K$NN and multiple-data-per-seller $K$NN, the complexity to compute the exact SV is $\OO(N^K)$ and $\OO(M^K)$, respectively, which are clearly not scalable. We propose a more efficient way to evaluate the SV up to provable approximation errors, which modifies the existing MC algorithm presented in Section~\ref{section:baseline}. By exploiting the locality property of the $K$NN-type algorithms, we propose a tighter upper bound on the number of permutations for a given approximation error and exhibit a novel implementation of the algorithm using efficient data structures.

The existing sample complexity bound is based on Hoeffding's inequality, which bounds the number of permutations needed in terms of the range of utility difference $\phi_i$. This bound is not always optimal as it depends on the extremal values that a random variable can take and thus accounts for the worst case. For $K$NN, the utility does not change after adding training instance $i$ for many subsets; therefore, the variance of $\phi_i$ is much smaller than its range. This inspires us to use Bennett's inequality, which bounds the sample complexity in terms of the variance of a random variable and often results in a much tighter bound than Hoeffding's inequality. 



\begin{theorem}
\label{thm:bennett_bound}
Given the range $[-r,r]$ of the utility difference $\phi_i$, an error bound $\epsilon$, and a confidence $1-\delta$, the sample size required such that
\begin{align*}
    P[\|\hat{s}-s\|_\infty \geq \epsilon] \leq \delta
\end{align*}
is $T\geq T^*$. $T^*$ is the solution of
\begin{align}
\label{eqn:bennett_bound}
        \sum_{i=1}^N \exp(-T^*(1-q_i^2) h(\frac{\epsilon}{(1-q_i^2)r})) = \delta/2.
\end{align}
where $h(u) = (1+u)\log (1+u)-u$ and 
\begin{align}
    q_i=\left\{\begin{array}{ll}
         &0, \text{ $i=1,\ldots,K$}  \\
         &\frac{i-K}{i} , \text{ $i=K+1,\ldots,N$}
    \end{array}
    \right.
\end{align}
\end{theorem}

Given $\epsilon$, $\delta$, and $r$, the required permutation size $T^*$ derived from Bennett's bound can be computed numerically. For general utility functions the range $r$ of the utility difference is twice the range of the utility function, while for the special case of the unweighted $K$NN classifier, $r=\frac{1}{K}$.

Although determining exact $T^*$ requires numerical calculation, we can nevertheless gain insights into the relationship between $N$, $\epsilon$, $\delta$ and $T^*$ through some approximation.
We leave the detailed derivation to Appendix~\ref{section:appendix_lower_bound},
but it is often reasonable
to use the following $\tilde{T}$ as an approximation of $T^*$:
\begin{align}
\label{eqn:bennett_approximate}
    \tilde{T}\geq \frac{r^2}{\epsilon^2}\log \frac{2K}{\delta}
\end{align}
The sample complexity bound derived above does not change with $N$. On the one hand, a larger training data size implies more unknown SVs to be estimated, thus requiring more random permutations. On the other hand, the variance of the SV across all training data decreases with the training data size, because an increasing proportion of training points makes insignificant contributions to the query result and results in small SVs. These two opposite driving forces make the required permutation size about the same across all training data sizes. 

The algorithm for the improved MC approximation is provided in Algorithm~\ref{alg:sampling}. We use a max-heap to organize the $K$NN. Since inserting any training data to the heap costs $\OO(\log K)$, incrementally updating the $K$NN in a permutation costs $\OO(N\log K)$. Using the bound on the number of permutations in (\ref{eqn:bennett_approximate}), we can show that the total time complexity for our improved MC algorithm is $\OO(\frac{N}{\epsilon^2}\log K \log \frac{K}{\delta})$.

\begin{algorithm}[t!]
\SetAlgoLined
\SetKwInOut{Input}{input}
\SetKwInOut{Output}{output}
\Input{Training set - $D = \{(x_i,y_i)\}_{i=1}^N$, utility function $\nu(\cdot)$, the number of measurements - $M$, the number of permutations - $T$}
\Output{The SV of each training point - $\hat{s}\in \R^N$}
    \For{$t\gets 1$ \KwTo $T$}{
        $\pi_t \leftarrow \text{GenerateUniformRandomPermutation}(D)$\;
        Initialize a length-$K$ max-heap $H$ to maintain the $K$NN\;
        \For{$i\gets1$ \KwTo $N$}{
            Insert $\pi_{t,i}$ to $H$\;
            \uIf{$H$ changes}{
                $\phi^t_{\pi_{t,i}} \leftarrow  \nu(\pi_{t,1:i}) -  \nu(\pi_{t,1:i-1})$\;
            }
            \Else{
                $\phi^t_{\pi_{t,i}} \leftarrow \phi^t_{\pi_{t,i-1}}$\;
            }
        }
    }
    $\hat{s}_i = \frac{1}{T}\sum_{t=1}^T \phi^t_i$ for $i=1,\ldots,N$\;
 \caption{Improved MC Approach}
 \label{alg:sampling}
\end{algorithm}

\section{Experiments}
\label{section:experiment}
We evaluate the proposed approaches to computing the SV of training data for various nearest neighbor algorithms.

\subsection{Experimental Setup}

\paragraph{\textbf{Datasets}} We used the following popular benchmark datasets of different sizes: (1) \texttt{dog-fish}~\cite{koh2017understanding} contains the features of dog and cat images extracted from ImageNet, with $900$ training examples and $300$ test examples for each class. The features have $2048$ dimensions, generated by the state-of-the-art Inception v3 network~\cite{szegedy2016rethinking} with all but the top layer. (2) \texttt{MNIST}~\cite{lecun-mnisthandwrittendigit-2010} is a handwritten digit dataset with $60000$ training images and $10000$ test images. We extracted  $1024$-dimensional features via a convolutional network. (3) The \texttt{CIFAR-10} dataset consists of $60000$ $32 \times 32$ color images in $10$ classes, with $6000$ images per class. The deep features have $2048$ dimensions and were extracted via the ResNet-50~\cite{he2016deep}. (4) \texttt{ImageNet}~\cite{imagenet_cvpr09} is an image dataset with more than $1$ million images organized according to the WordNet hierarchy. We chose $1000$ classes which have in total around $1$ million images and extracted $2048$-dimensional deep features by the ResNet-50 network. (5) Yahoo Flickr Creative Commons 100M
that consists of 99.2 million photos. We randomly chose a 10-million subset (referred to as \texttt{Yahoo10m} hereinafter) for our experiment, and used the deep features extracted by~\cite{amato2016yfcc100m}.

\paragraph{\textbf{Parameter selection for LSH}} The three main parameters that affect the performance of the LSH are the number of projections per hash value ($m$), the number of hash tables ($h$), and the width of the project ($r$). Decreasing $r$ decreases the probability of collision for any two points, which is equivalent to increasing $m$. Since a smaller $m$ will lead to better efficiency, we would like to set $r$ as small as possible. However, decreasing $r$ below a certain threshold increases the quantity $g(C_K)$, thereby requiring us to increase $h$. 
Following~\cite{datar2004locality}, we performed grid search to find the optimal value of $r$ which we used in our experiments. 
Following~\cite{gionis1999similarity}, we set $m=\alpha \log N/\log (f_h(D_\text{mean})^{-1})$. For a given value of $m$, it is easy to find the optimal value of $h$ which will guarantee that the SV approximation error is no more than a user-specified threshold. We tried a few values for $\alpha$ and reported the $m$ that leads to lowest runtime.
For all experiments pertaining to the LSH, we divided the dataset into two disjoint parts: one for selecting the parameters, and another for testing the performance of LSH for computing the SV.

\subsection{Experimental Results}
\subsubsection{Unweighted KNN Classifier}

\paragraph{\textbf{Correctness}}
We first empirically validate our theortical
result. We randomly 
selected $1000$ training points and $100$ test points from \texttt{MNIST}. We computed the SV of each training point with respect to the $K$NN utility using the exact algorithm and the baseline MC method. Figure~\ref{fig:correctness} shows that the MC estimate of the SV for each training point converges to the result of the exact algorithm. 

\begin{figure}
    \centering
\includegraphics[width=0.5\textwidth]{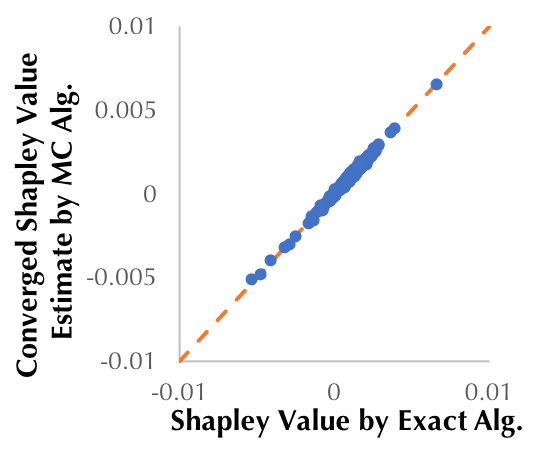}
\caption{The SV produced by the exact algorithm and the baseline MC approximation algorithm. }
    \label{fig:correctness}
\end{figure}




\paragraph{\textbf{Performance}} We validated the hypothesis that
our exact algorithm and the LSH-based method outperform
the baseline MC method.
We take the approximation error $\epsilon=0.1$ and $\delta=0.1$ for both MC and LSH-based approximations. 
We bootstrapped the \texttt{MNIST} dataset to synthesize training datasets of various sizes. The three SV calculation methods were implemented on a machine with 2.6 GHz Intel Core i7 CPU. The runtime of the three methods for different datasets is illustrated in Figure~\ref{fig:runtime_unweighted} (a). The proposed exact algorithm is faster than the baseline approximation by several orders magnitude and it produces the exact SV. By circumventing the computational complexity of sorting a large array, the LSH-based approximation can significantly outperform the exact algorithm, especially when the training size is large. Figure~\ref{fig:runtime_unweighted} (b) sheds light on the increasing performance gap between the LSH-based approximation and the exact method with respect to the training size. The relative contrast of these bootstrapped datasets grows with the number of training points, thus requiring fewer hash tables and less time to search for approximate nearest neighbors. We also tested the approximation approach proposed in our prior work~\cite{jia2018shapley}, which achieves the-start-of-the-art performance for ML models that cannot be incrementally maintained. However, for models that have efficient incremental training algorithms, like $K$NN, it is less efficient than the baseline approximation, and the experiment for $1000$ training points did not finish in $4$ hours.


\begin{figure}[t!]
\centering
\includegraphics[width=\columnwidth]{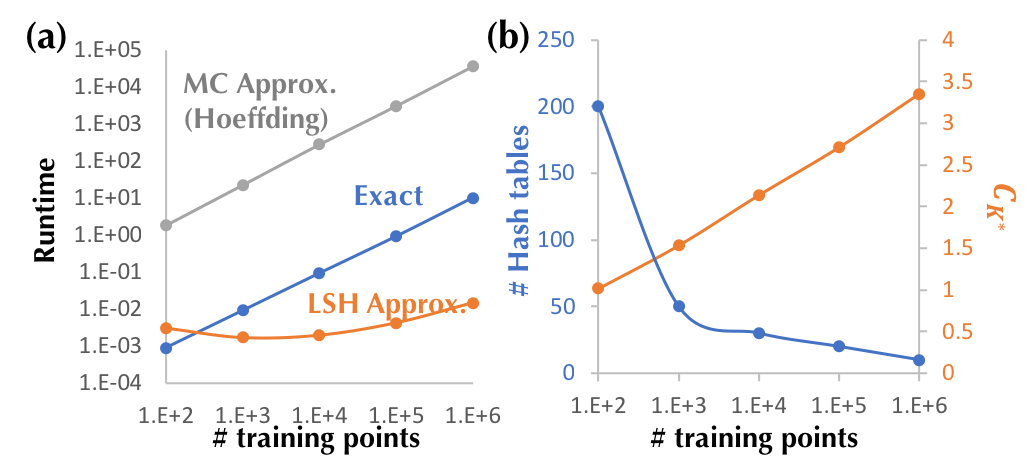}
\caption{Performance of unweighted $K$NN classification in the single-data-per-seller case. }
\label{fig:runtime_unweighted}
\end{figure}

Using a machine with the Intel Xeon E5-2690 CPU and $256$ GB RAM, we benchmarked the runtime of the exact and the LSH-based approximation algorithm on three popular datasets, including \texttt{CIFAR-10}, \texttt{ImageNet}, and \texttt{Yahoo10m}. For each dataset, we randomly selected $100$ test points, computed the SV of all training points with respect to each test point, and reported the average runtime across all test points. 
The results for $K=1$ are reported in Figure~\ref{table:runtime}. We can see that the LSH-based method can bring a $3\times$-$5\times$ speed-up compared with the exact algorithm. The performance of LSH depends heavily on the dataset, especially in terms of its relative contrast. This effect will be thoroughly studied in the sequel. We compare the prediction accuracy of $K$NN ($K=1,2,5$) with the commonly used logistic regression and the result is illustrated in Figure~\ref{table:accuracy}. We can see that $K$NN achieves comparable prediction power to logistic regression when using features extracted via deep neural networks. The runtime of the exact and the LSH-based approximation for $K=2,5$ is similar to the $K=1$ case in Figure~\ref{table:runtime}, so we will leave their corresponding results to Appendix~\ref{appendix:runtime_k25}.

\begin{figure}[t!]
\caption{Average runtime of the exact and the LSH-based approximation algorithm for computing the unweighted $K$NN SV for a single test point. We take $\epsilon,\delta=0.1$ and $K=1$.}
\label{table:runtime}
\begin{tabular}{ccccc}
\hline
\textbf{Dataset}    & \textbf{Size} & \begin{tabular}[c]{@{}c@{}}\textbf{Estimated}  \\ \textbf{Contrast}\end{tabular} & \begin{tabular}[c]{@{}c@{}}\textbf{Runtime} \\ \textbf{(Exact)}\end{tabular} & \begin{tabular}[c]{@{}c@{}}\textbf{Runtime}\\ \textbf{(LSH)}\end{tabular} \\ \hline
\texttt{CIFAR-10}      &  $6$E+$4$    & 1.2802 &  0.78s                                                            &        0.23s                                                              \\ \hline
\texttt{ImageNet}    &   $1$E+$6$   &  1.2163  &    $11.34$s                                                        &      $2.74$s                                                \\ \hline
\texttt{Yahoo10m} &    $1$E+$7$  &  1.3456 &                                   $203.43$s                            &                                 $44.13$s                                      \\ \hline
\end{tabular}

\end{figure}






\begin{figure}[t!]
\caption{Comparison of prediction accuracy of $K$NN vs. logistic regression on deep features.}
\label{table:accuracy}
\begin{tabular}{ccccc}
\hline
\textbf{Dataset}   & \textbf{$1$NN} & \textbf{$2$NN} & \textbf{$5$NN} & \textbf{Logistic Regression}  \\ \hline
\texttt{CIFAR-10}     &  $81\%$    &   $83\%$       &  $80\%$ &  $\mathbf{87\%}$         \\ \hline
\texttt{ImageNet}  &   $77\%$   &  $73\%$   & $\mathbf{84\%}$ & $82\%$                \\ \hline
\texttt{Yahoo10m} &    $90\%$  &   $96\%$   & $\mathbf{98\%}$ & $96\%$                \\ \hline
\end{tabular}
\end{figure}


\paragraph{\textbf{Effect of relative contrast on the LSH-based method}}
Our theoretical result suggests that the $K^*$th relative contrast ($K^*=\max\{K,\ceil*{1/\epsilon}\}$) determines the complexity of the LSH-based approximation. We verified the effect of relative contrast by experimenting on three datasets, namely, \texttt{dog-fish}, \texttt{deep} and \texttt{gist}. \texttt{deep} and \texttt{gist} were constructed by extracting the deep features and gist features~\cite{siagian2007rapid} from \texttt{MNIST}, respectively. All of these datasets were normalized such that $D_\text{mean}=1$. Figure~\ref{fig:LSH} (a) shows that the relative contrast of each dataset decreases as $K^*$ increases. In this experiment, we take $\epsilon=0.01$ and $K=2$, so the corresponding $K^*=1/\epsilon=100$. At this value of $K^*$, the relative contrast is in the following order: \texttt{deep} ($1.57$) $>$ \texttt{gist} ($1.48$) $>$ \texttt{dog-fish} ($1.17$).
 From Figure~\ref{fig:LSH} (b) and (c), we see that the number of hash tables and the number of returned points required to meet the $\epsilon$ error tolerance for the three datasets follow the reversed order of their relative contrast, as predicted by Theorem~\ref{thm:lsh_approx_sp}. Therefore, the LSH-based approximation will be less efficient if the $K$ in the nearest neighbor algorithm is very large or the desired error $\epsilon$ is small. Figure~\ref{fig:LSH} (d) shows that
 the LSH-based method can better approximate the true SV as the recall of the underlying nearest neighbor retrieval gets higher. For the datasets with high relative contrast, e.g., \texttt{deep} and \texttt{gist}, a moderate value of recall ($\sim0.7$) can already lead to an approximation error below the desired threshold. On the other hand, \texttt{dog-fish}, which has low relative contrast, will need fairly accurate nearest neighbor retrieval (recall $\sim1$) to obtain a tolerable approximation error. The reason for the different retrieval accuracy requirements is that for the dataset with higher relative contrast, even if the retrieval of the nearest neighbors is inaccurate, the rank of the erroneous elements in the retrieved set may still be close to that of the missed true nearest neighbors. Thus, these erroneous elements will have only little impacts on SV approximation errors.

\begin{figure}[t!]
\centering
\includegraphics[width=\columnwidth]{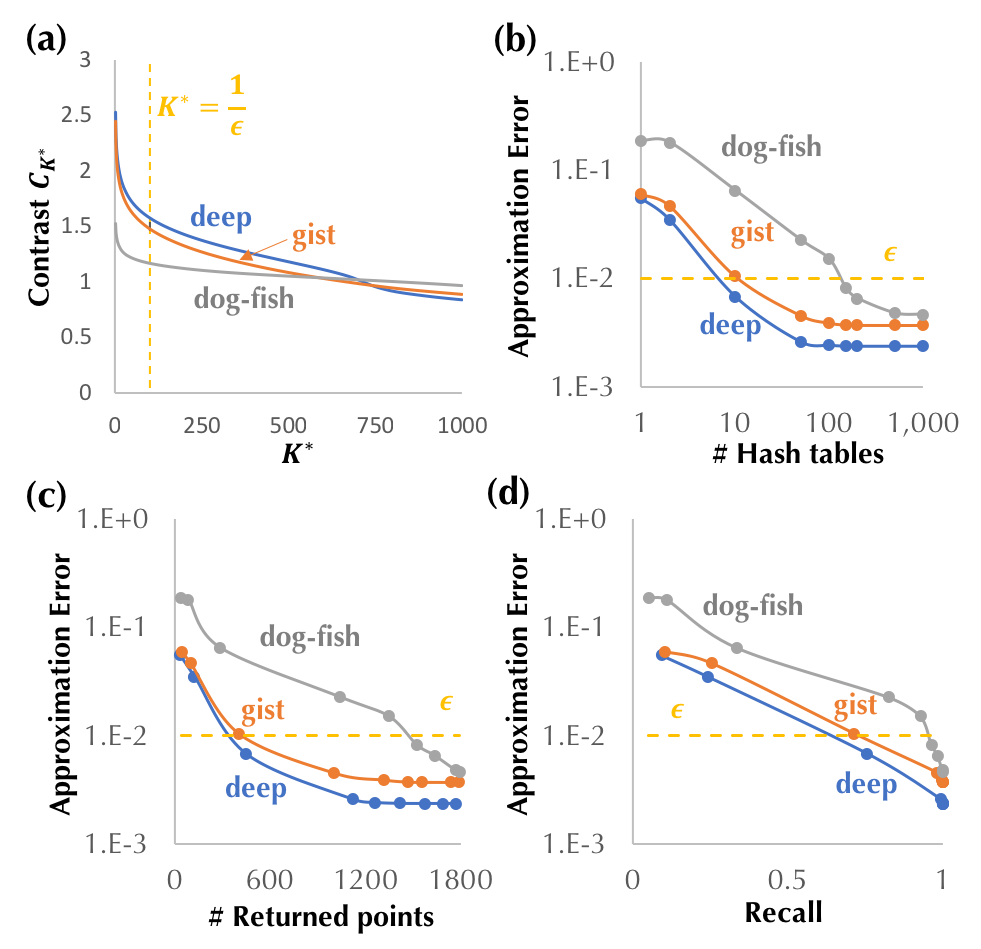}
\caption{Performance of LSH on three datasets: \texttt{deep}, \texttt{gist}, \texttt{dog-fish}. (a) Relative contrast $C_{K^*}$ vs. $K^*$. (b), (c) and (d) illustrate the trend of the SV approximation error for different number of hash tables, returned points and recalls.}
\label{fig:LSH}
\end{figure}

\paragraph{\textbf{Simulation of the theoretical bound of LSH}} According to Theorem~\ref{thm:lsh_approx_sp}, the complexity of the LSH-based approximation is dominated by the exponent $g(C_{K^*})$, where $K^*=\min\{K,1/\epsilon\}$ and $g(\cdot)$ depends on the width $r$ of the $p$-stable distribution used for LSH. We computed $C_{K^*}$ and $g(C_{K^*})$ for $\epsilon\in\{0.001,0.01,0.1,1\}$ and let $K=1$ in this simulation. The orange line in Figure~\ref{fig:LSH_bound} (a) shows that a larger $\epsilon$ induces a larger value of relative contrast $C_{K^*}$, rendering the underlying nearest neighbor retrieval problem of the LSH-based approximation method easier. In particular, $C_{K^*}$ is greater than $1$ for all epsilons considered except for  $\epsilon=0.001$. Recall that $g(C_K)=\log f_h(1/C_K)/\log f_h(1)$; thus, $g(C_{K^*})$ will exhibit different trends for the epsilons with $C_{K^*}>1$ and the ones with $C_{K^*}<1$, as shown in Figure~\ref{fig:LSH_bound} (b). Moreover, Figure~\ref{fig:LSH_bound} (b) shows that the value of $g(C_{K^*})$ is more or less insensitive to $r$ after a certain point. For $\epsilon$ that is not too small, we can choose $r$ to be the value at which $g(C_{K^*})$ is minimized. It does not make sense to use the LSH-based approximation if the desired error $\epsilon$ is too small to have the corresponding $g(C_{K^*})$ less than one, since its complexity is theoretically higher than the exact algorithm. The blue line in Figure~\ref{fig:LSH_bound} (a) illustrates the exponent $g(C_{K^*})$ as a function of $\epsilon$ when $r$ is chosen to minimize $g(C_{K^*})$. We observe that $g(C_{K^*})$ is always below $1$ except when $\epsilon=0.001$.

\begin{figure}[t!]
\centering
\includegraphics[width=\columnwidth]{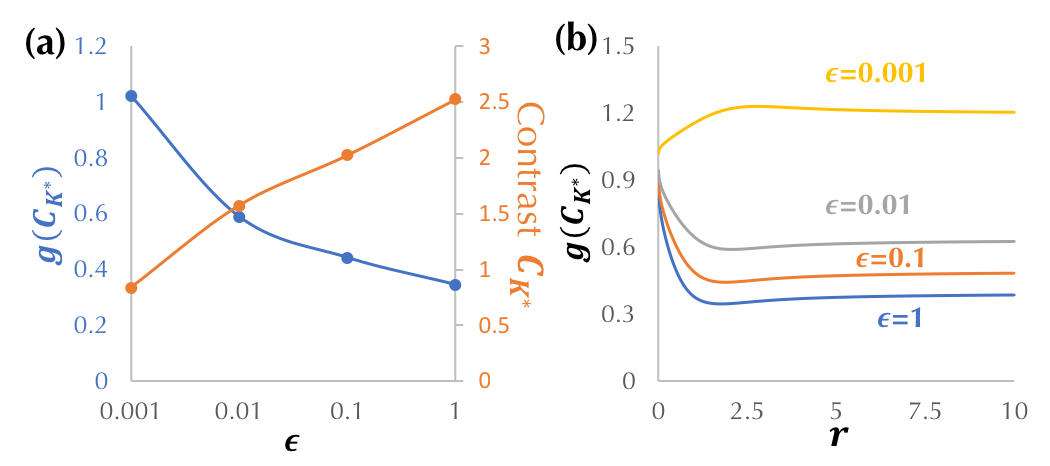}
\caption{(a) The exponent $g(C_{K^*})$ in the complexity bound of the LSH-based method and the relative contrast $C_{K^*}$ computed for different $\epsilon$. $K$ is fixed to $1$. (b) $g(C_{K^*})$ vs. the projection width $r$ of the LSH.}
\label{fig:LSH_bound}
\end{figure}

\subsubsection{Evaluation of Other Extensions}
We introduced the extensions of the exact SV calculation algorithm to the settings beyond unweighted $K$NN classification. Some of these settings require polynomial time to compute the exact SV, which is impractical for large-scale datasets. For those settings, we need to resort to the MC approximation method. We first compare the sample complexity of different MC methods, including the baseline and our improved MC method (Section~\ref{section:improved_mc}). Then, we demonstrate data values computed in various settings.

\paragraph{\textbf{Sample complexity for MC methods}} The time complexity of the MC-based SV approximation algorithms is largely dependent on the number of permutations. Figure~\ref{fig:permutation_bound} compares the permutation sizes used in the following three methods against the actual permutation size needed to achieve a given approximation error (marked as ``ground truth'' in the figure): (1) ``Hoeffding'', which is the baseline approach and uses the Hoeffding's inequality to decide the number of permutations; (2) ``Bennett'', which is our proposed approach and exploits Bennett's inequality to derive the permutation size; (3) ''Heuristic'', which terminates MC simulations when the change of the SV estimates in the two consecutive iterations is below a certain value, which we set to $\epsilon/50$ in this experiment. We notice that the ground truth requirement for the permutation size decreases at first and remains constant when the training data size is large enough.
From Figure~\ref{fig:permutation_bound}, the bound based on the Hoeffding's inequality is too loose to correctly predict the correct trend of the required permutation size. By contrast, our bound based on Bennett's inequality exhibits the correct trend of permutation size with respect to training data size. In terms of runtime, our improved MC method based on Bennett's inequality is more than $2\times$ faster than the baseline method when the training size is above $1$ million. Moreover, using the aforementioned heuristic, we were able to terminate the MC approximation algorithm even earlier while satisfying the requirement of the approximation error. 

\begin{figure}[t!]
\centering
\includegraphics[width=\columnwidth]{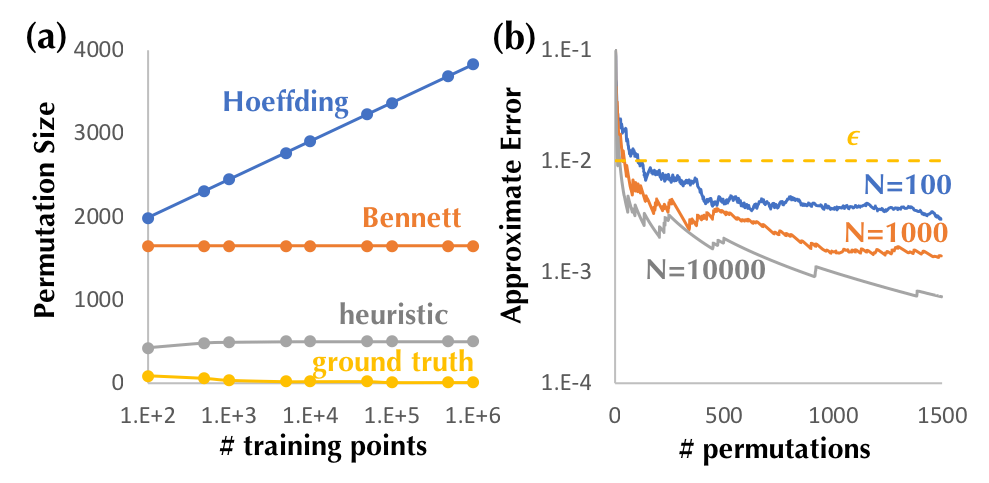}
\caption{Comparison of the required permutation sizes for different number of training points derived from the Hoeffding's inequality (baseline), Bennett's inequality and the heuristic method against the ground truth.}
\label{fig:permutation_bound}
\end{figure}

\paragraph{\textbf{Performance}} We conducted experiments on the \texttt{dog-fish} dataset to compare the runtime of the exact algorithm and our improved MC method. We took $\epsilon=0.01$ and $\delta=0.01$ in the approximation algorithm and used the heuristic to decide the stopping iteration.

Figure~\ref{fig:runtime_weighted} compares the runtime of the exact algorithm and our improved MC approximation for weighted $K$NN classification. In the first plot, we fixed $K=3$ and varied the number of training points. In the second plot, we set the training size to be $100$ and changed $K$. We can see that the runtime of the exact algorithm exhibits polynomial and exponential growth with respect to the training size and $K$, respectively. By contrast, the runtime of the approximation algorithm increases slightly with the number of training points and remains unchanged for different values of $K$.

\begin{figure}[t!]
\centering
\includegraphics[width=\columnwidth]{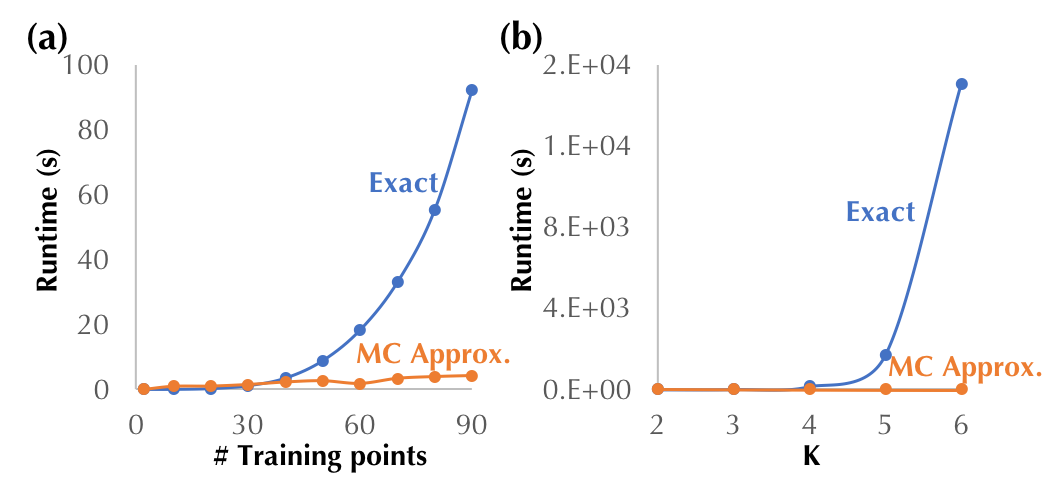}
\caption{Performance of the weighted $K$NN classification. }
\label{fig:runtime_weighted}
\end{figure}

Figure~\ref{fig:runtime_multi} compares the runtime of the exact algorithm and the MC approximation for the unweighted $K$NN classification when each seller can own multiple data instances. To generate Figure~\ref{fig:runtime_multi} (a), we set $K=2$ and varied the number of sellers. We kept the total number of training instances of all sellers constant and randomly assigned the same number of training instances to each seller. We can see that the exact calculation of the SV in the multi-data-per-seller case has polynomial time complexity, while the runtime of the approximation algorithm barely changes with the number of sellers. Since the training data in our approximation algorithm were sequentially inserted into a heap, the complexity of the approximation algorithm is mainly determined by the total number of training data held by all sellers. Moreover, as we kept the total number of training points constant, the approximation algorithm appears invariant over the number of sellers. Figure~\ref{fig:runtime_multi} (b) shows that the runtime of exact algorithm increases with $K$, while the approximation algorithm's runtime is not sensitive to $K$. To summarize, the approximation algorithm is preferable to the exact algorithm when the number of sellers and $K$ are large.

\begin{figure}[t!]
\centering
\includegraphics[width=\columnwidth]{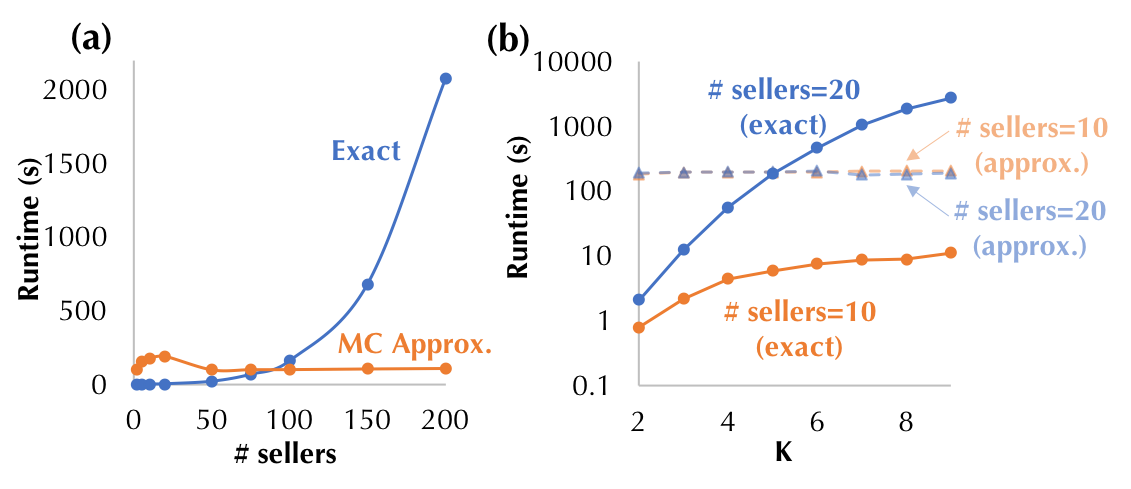}
\caption{Performance of the $K$NN classification in the multi-data-per-seller case. }
\label{fig:runtime_multi}
\end{figure}

\paragraph{\textbf{Unweighted vs. weighted $K$NN SV}} We constructed an unweighted $K$NN classifier using the \texttt{dog-fish}. Figure~\ref{fig:weighted_unweighted_class} (a) illustrates the training points with top $K$NN SVs with respect to a specific test image. We see that the returned images are semantically correlated with the test one. We further trained a weighted $K$NN on the same training set using the weight function that weighs each nearest neighbor inversely proportional to the distance to a given test point; and compared the SV with the ones obtained from the unweighted $K$NN classifier. We computed the average SV across all test images for each training point and demonstrated the result in Figure~\ref{fig:weighted_unweighted_class} (b). Every point in the figure represents the SVs of a training point under the two classifiers. We can see that the unweighted $K$NN SV is close to the weighted one. This is because in the high-dimensional feature space, the distances from the retrieved nearest neighbors to the query point are large, in which case the weights tend to be small and uniform. 

\begin{figure}[t!]
\centering
\includegraphics[width=\columnwidth]{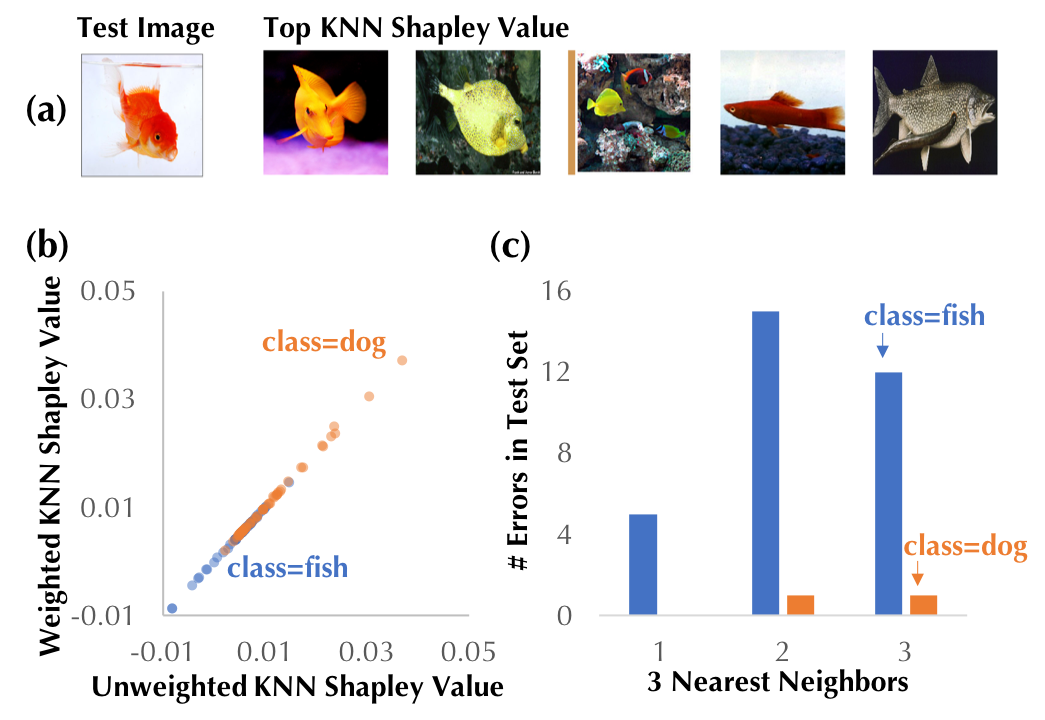}
\caption{Data valuation on \textsc{Dog-Fish} dataset ($K=3$). (a)
top valued data points; (b) unweighted vs. weighted $K$NN SV on the whole test set; (c) Per-class top-$K$ neighbors labeled inconsistently with the misclassified test example.}
\label{fig:weighted_unweighted_class}
\end{figure}

Another observation from Figure~\ref{fig:weighted_unweighted_class} (b) is that the $K$NN SV assigns more values to dog images than fish images.
Figure~\ref{fig:weighted_unweighted_class} (c) plots the distribution of the number test examples with regard to the number of their top-$K$ neighbors in the training set are with a label inconsistent with the true label of the test example. We see that most of the nearest neighbors with inconsistent labels belong to the fish class. In other words, the fish training images are more close to the dog images in the test set than the dog training images to the test fish. Thus, the fish training images are more susceptible to mislead the predictions and should have lower values. This intuitively explains why the $K$NN SV places a higher importance on the dog images.

\paragraph{\textbf{Data-only vs. composite game}} We introduced two game-theoretic models for distributing the gains from an ML model and would like to understand how the shares of the analyst and the data contributors differ in the two models. We constructed an unweighted $K$NN classifier with $K=10$ on the \texttt{dog-fish} dataset and compute the SV of each player in the data-only and the composite game. Recall that the total utility of both games is defined as the average test accuracy trained on the full set of training data. Figure~\ref{fig:unweighted_class_dataonly_general} (a) shows that the SV for the analyst increases with the total utility. Therefore, under the composite game formulation, the analyst has huge incentive to train a good ML model as the values assigned to the analyst gets larger with a better ML model. In addition, in the composite game formulation, the analyst has exclusive control over the computational resources and the data only creates value when it is analyzed with computational modules, the analyst should take the greatest share of the utility extracted from the ML model. This intuition is reflected in Figure~\ref{fig:unweighted_class_dataonly_general} (a). Figure~\ref{fig:unweighted_class_dataonly_general} (b) demonstrates that the SV of the data contributors in the composite game is correlated with that in the data-only game, although the actual value is much smaller. Figure~\ref{fig:unweighted_class_dataonly_general} (c) exhibits the trend of the SV of the analyst and data contributors as more data contributors participate in a data transaction. The SV of the analyst gets larger with more data contributors, while the average value obtained by each data contributor decreases in both composite and data-only games. Figure~\ref{fig:unweighted_class_dataonly_general} (d) zooms into the change of the maximum and minimum value among all data contributors in the data-only game setting (the result in the composite game setting is similar). We can see that both the maximum and minimum value decreases at the beginning; as more data contributors are involved in a data transaction, the minimum value demonstrates a small increment. The points with lowest values tend to hurt the ML model performance when they are added into the training set. With more data contributors and more training points, the negative impacts of these ``outliers'' can get mitigated.

\begin{figure}[t!]
\centering
\includegraphics[width=\columnwidth]{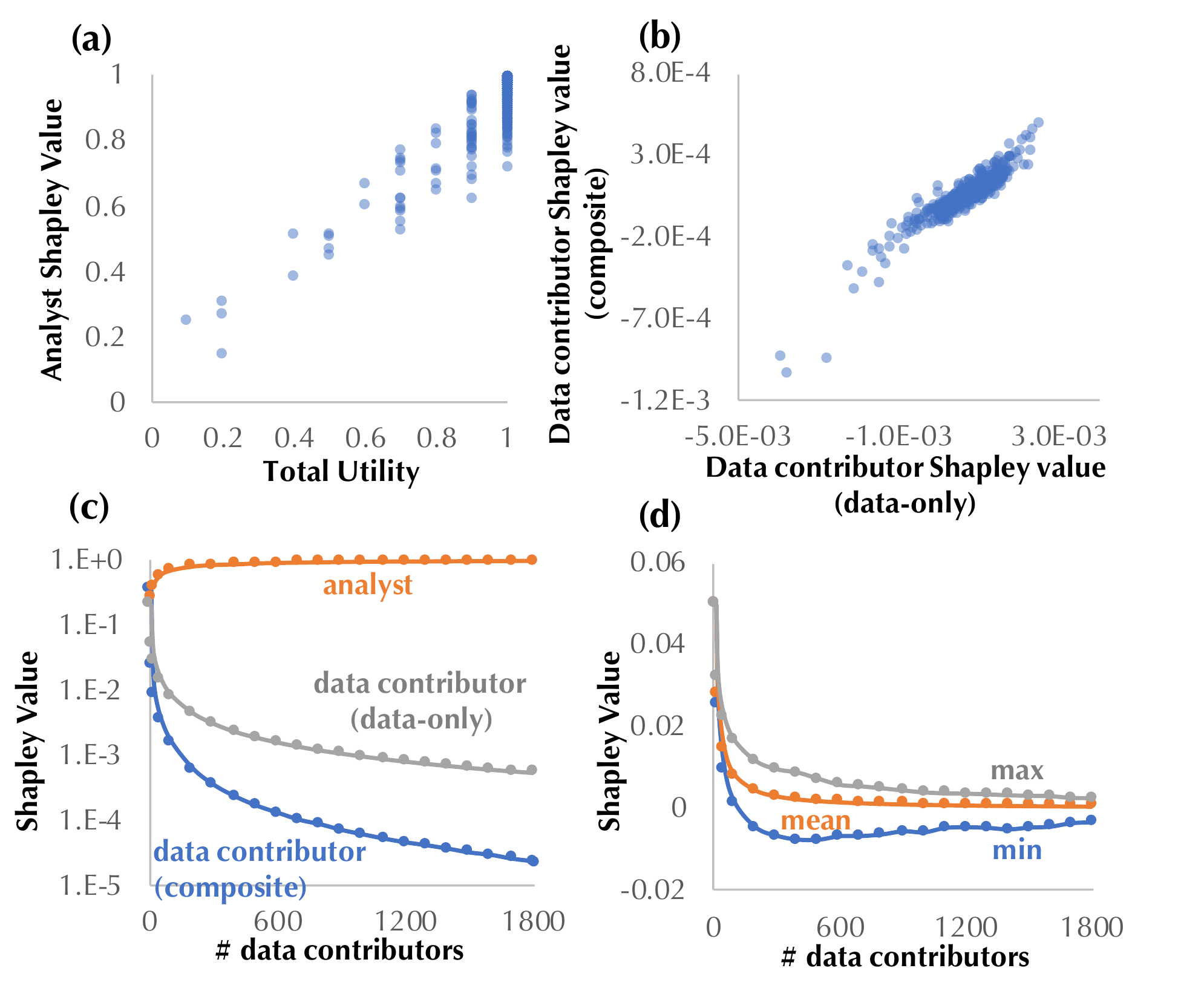}
\caption{(a) The SV of the analyst in the composite game vs. total utility obtained from the ML model; (b) the correlation between the data contributors' SV in the composite game with that in the data-only game; (c) The SV of all players in the two games for different number of data contributors; (d) The mean, maximum, minimum of the data contributors' SVs in the data-only game.}
\label{fig:unweighted_class_dataonly_general}
\end{figure}

\paragraph{\textbf{Remarks}} We summarize several takeaways from our experimental evaluation. (1) For unweighted $K$NN classifiers, the LSH-based approximation is more preferable than the exact algorithm when a moderate amount of approximation error can be tolerated and $K$ is relatively small. Otherwise, it is recommended to use the exact algorithm as a default approach for data valuation. (2) For weighted $K$NN regressors or classifiers, computing the exact SV has $\OO(N^K)$ compleixty, thus not scalable for large datasets and large $K$. Hence, it is recommended to adopt the Monte Carlo method in Algorithm~\ref{alg:sampling}. Moreover, using the heuristic based on the change of SV estimates in two consecutive iterations to decide the termination point of the algorithm is much more efficient than using the theoretical bounds, such as Hoeffding or Bennett. 


\section{Discussion}
\label{section:discussion}

\paragraph{\textbf{From the $K$NN SV to Monetary Reward}} Thus far, we have focused on the problem of attributing the $K$NN utility and its extensions to each data and computation contributor. In practice, the buyer pays a certain amount of money depending on the model utility and it is required to determine the share of each contributor in terms of monetary rewards. Thus, a remaining question is how to map the $K$NN SV, a share of the total model utility, to a share of the total revenue acquired from the buyer. A simple method for such mapping is to assume that the revenue is an affine function of the model utility, i.e., $R(S) = a\nu(S)+b$ where $a$ and $b$ are some constants which can be determined via market research. Due to the additivity property, we have $s(R, i) = as(\nu,i)+b$. Thus, we can apply the same affine function to the $K$NN SV to obtain the the monetary reward for each contributor.

\begin{figure}
    \centering
\includegraphics[width=0.5\textwidth]{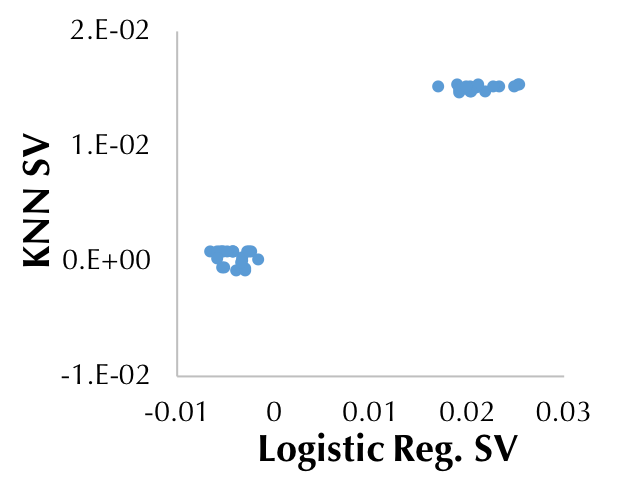}
\caption{Comparison of the SV for a logistic regression and a $K$NN trained on the \texttt{Iris} dataset.}
    \label{fig:knn_logistic}
\end{figure}


\paragraph{\textbf{Computing the SV for Models Beyond $K$NN}} The efficient algorithms presented in this paper are possible only because of the ``locality'' property of $K$NN. However, given many previous empirical results showing that a $K$NN classifier can often achieve a classification accuracy that is comparable with classifiers such as SVMs and logistic regression given sufficient memory, we could use the $K$NN SV as a proxy for other classifiers. We compute the SV for a logistic regression classifier and a $K$NN classifier trained on the same dataset namely \texttt{Iris}, and the result shows that the SVs under these two classifiers are indeed correlated (see Figure~\ref{fig:knn_logistic}). The only caveat is that $K$NN SV does not distinguish between neighboring data points that have the same label. If this caveat is acceptable, we believe that the $K$NN SV provides an efficient way to approximately assess the relative contribution of different data points for other classifiers as well. Moreover, for calculating the SV for general deep neural networks, we can take the deep features (i.e., the input to the last softmax layer) and corresponding labels, and train a $K$NN classifier on the deep features. We calibrate $K$ such that the resulting $K$NN mimics the performance of the original deep net and then employ the techniques presented in this paper to calculate a surrogate for the SV under the deep net.

\paragraph{\textbf{Implications of Task-Specific Data Valuation}} Since the SV depends on the utility function associated with the game, data dividends based on the SV are contingent on the definition of model usefulness in specific ML tasks. The task-specific nature of our data valuation framework offers clear advantages---it allows to accommodate the variability of a data point's utility from one application to another and assess its worth accordingly. Moreover, it enables the data buyer to defend against \emph{data poisoning attacks}, wherein the attacker intentionally contributes adversarial training data points crafted specifically to degrade the performance of the ML model. In our framework, the ``bad'' training points will naturally have low SVs because they contribute little to boosting the performance of the model. \\
Having the data values dependent on the ML task, on the other hand, may raise some concerns about whether the data values may inherit the flaws of the ML models as to which the values are computed: if the ML model is biased towards a subpopulation with specific sensitive attributes (e.g., gender, race), will the data values reflect the same bias? Indeed, these concerns can be addressed by designing proper utility functions that devalue the unwanted properties of ML models. For instance, even if the ML model may be biased towards specific subpopulation, the buyer and data contributors can agree on a utility function that gives lower score to unfair models and compute the data values with respect to the concordant utility function. In this case, the training points will be appraised partially according to how much they contribute to improving the model fairness and the resulting data values would not be affected by the bias of the underlying model.
Moreover, there is a venerable line of works studying algorithms to help improve fairness~\cite{zemel2013learning,woodworth2017learning,hardt2016equality}. These algorithms can also be applied to resolve the potential bias in value assignments. For instance, before providing the data to the data buyer, data contributors can preprocess the training data so that the ``sanitized'' data removes the information correlated with sensitive attributes~\cite{zemel2013learning}. However, to ensure that the data values are accurately computed according to an appropriate utility function that the buyer and the data contributors agree on or that the models are trained with proper fairness criteria, it is necessary to develop systems that can support transparent machine learning processes. Recent work has been studying training machine learning models on blockchains for removing the middleman to audit the model performance and enhancing transparency~\cite{blockchain}. We are currently implementing the data valuation framework on a blockchain-based data market, which can naturally resolve the problems of transparency and trust. Since the focus of this work is the algorithmic foundation of data valuation, we will leave the discussion of the combination of blockchains and data valuation for future work.

\section{Related Work}
\label{section:related_work}
The problem of data pricing has received a lot of attention recently. The pricing schemes deployed in the existing data marketplaces are simplistic, typically setting a fixed price for the whole or parts of the dataset. Before withdrawn by Microsoft, the Azure Data Marketplace adopted a subscription model that gave users access to a certain number of result pages per month~\cite{koutris2015query}. 
Xignite~\cite{xignite} sells financial datasets and prices data based on the data type, size, query frequency, etc. 

There is rich literature on query-based pricing~\cite{koutris2015query,koutris2013toward,koutris2012querymarket,deep2017qirana,lin2014arbitrage,li2012pricing,upadhyaya2016price}, aimed at design pricing schemes for fine-grained queries over a dataset. In query-based pricing, a seller can assign prices to a few views and the price for any queries purchased by a buyer is automatically derived from the explicit prices over the views. Koutris et al.~\cite{koutris2015query} identified two important properties that the pricing function must satisfy, namely, arbitrage-freeness and discount-freeness. The arbitrage-freeness indicates that whenever
query $Q_1$ discloses more information than query $Q_2$, we want to ensure that the price of $Q_1$ is higher than $Q_2$; otherwise, the data buyer has an arbitrage opportunity to purchase the desired information at a lower price. The discount-freeness requires that the prices offer no additional discounts than the ones specified by the data seller. The authors further proved the uniqueness of the pricing function with the two properties, and established a dichotomy on the complexity of the query pricing problem when all views are selection queries. Li et al.~\cite{li2012pricing} proposed additional criteria for data pricing, including non-disclosiveness (preventing the buyers from inferring unpaid query answers by analyzing the publicly available prices of queries) and regret-freeness (ensuring that the price of asking a sequence of queries in multiple interactions is not higher than asking them all-at-once), and investigated the class of pricing functions that meet these criteria. Zheng et al.~\cite{zheng2019arete} studied how data uncertainty should affect the price of data, and proposed a data pricing framework for mobile crowd-sensed data. Recent work on query-based pricing focuses on enabling efficient pricing over a wider range of queries, overcoming the issues such as double-charging arising from building practical data marketplaces~\cite{koutris2013toward,deep2017qirana,upadhyaya2016price}, and compensating data owners for their privacy loss~\cite{li2017theory}. Due to the increasing pervasiveness of ML-based analytics, there is an emerging interest in studying the cost of acquiring data for ML. Chen et al.~\cite{chen2018model,chen2017model} proposed a formal framework to price ML model instances, wherein an optimization problem was formulated to find the arbitrage-free price that maximizes the revenue of a seller. The model price can be also used for pricing its training dataset. This paper is complementary to these works in that we consider the scenario where the training set is contributed by multiple sellers and focus on the revenue sharing problem thereof.


While the interaction between data analytics and economics has been extensively studied in the context of both relational database queries and ML, few works have dived into the vital problem of allocating revenues among data owners. \cite{koutris2012querymarket} presented a technique for fair revenue sharing when multiple sellers are involved in a relational query. By contrast, our paper focuses on the revenue allocation for nearest neighbor algorithms, which are widely adopted in the ML community. Moreover, our approach establishes a formal notion of fairness based on the SV. The use of the SV for pricing personal data can be traced back to~\cite{kleinberg2001value}, which studied the SV in the context of marketing survey, collaborative filtering, and recommendation systems. \cite{chessa2017cooperative} also applied the SV to quantify the value of personal information when the population of data contributors can be modeled as a network. \cite{michalak2013efficient} showed that for specific network games, the exact SV can be computed efficiently.

There exist various methods to rank the importance of training data, which can also potentially be used for data valuation. For instance, influence functions~\cite{koh2017understanding} approximate the change of the model performance after removing a training point for smooth parametric ML models. Ogawa et al.~\cite{ogawa2013safe} proposed rules to identify and remove the least influential data when training support vector machines (SVM) to reduce the computation cost. 
However, unlike the SV, these approaches do not satisfy the group rationality, fairness, and additivity properties simultaneously.

Despite the desirable properties of the SV, computing the SV is known to be expensive. In its
most general form, the SV can be $\mathsf{\#P}$-complete to compute~\cite{deng1994complexity}. For bounded utility functions, Maleki et al.~\cite{maleki2013bounding} described a sampling-based approach that requires $\mathcal{O}(N\log N)$ samples to achieve a desired approximation error. By taking into account special properties of the utility function, one can derive more efficient approximation algorithms. For instance, Fatima et al.~\cite{fatima2008linear} proposed a probabilistic approximation algorithm with $\mathcal{O}(N)$ complexity for weighted voting games. Ghorbani et al.~\cite{ghorbani2019data} developed two heuristics to accelerate the estimation of the SV for complex learning algorithms, such as neural networks. One is to truncate the calculation of the marginal contributions as the change in performance by adding only one
more training point becomes smaller and smaller. Another is to use one-step gradient to approximate the marginal contribution. The authors also demonstrate the use of the approixmate SV for outlier identification and informed acquisition of new training data. However, their algorithms do not provide any guarantees on the approximation error, thus limiting its viability for practical data valuation. Raskar et al~\cite{raskar2019data} presented a taxonomy of data valuation problems for data markets and discussed challenges associated with data sharing. 


\section{Conclusion}
The SV has been long advocated as a useful economic concept to measure data value but has not found its way into practice due to the issue of exponential computational complexity. This paper presents a step towards practical algorithms for data valuation based on the SV. We focus on the case where data are used for training a $K$NN classifier and develop algorithms that can calculate data values exactly in quasi-linear time and approximate them in sublinear time. We extend the algorithms to the case of $K$NN regression, the situations where a contributor can own multiple data points, and the task of valuing data contributions and analytics simultaneously. 
For future work, we will integrate our proposed data valuation algorithms into the clinical data market that we are currently building. We will also explore efficient algorithms to compute the data values for other popular ML algorithms such as gradient boosting, logistic regression, and deep neural networks.

\section*{Acknowledgement}
This work is supported in part by the Republic of Singapore’s National Research
Foundation through a grant to the Berkeley Education Alliance for Research in
Singapore (BEARS) for the Singapore-Berkeley Building Efficiency and
Sustainability in the Tropics (SinBerBEST) Program. This work is also supported in part by the CLTC (Center for Long-Term
Cybersecurity); FORCES (Foundations Of Resilient
CybEr-Physical Systems), which receives support from
the National Science Foundation (NSF award numbers
CNS-1238959, CNS-1238962, CNS-1239054, CNS1239166); the National Science Foundation under
Grant No. TWC-1518899; and DARPA FA8650-18-2-7882. 
CZ and the DS3Lab gratefully acknowledge the support from the Swiss National Science Foundation (Project Number 200021\_184628) and a Google Focused Research Award.








\bibliography{aomsample}

\providecommand{\etalchar}[1]{$^{#1}$}
\providecommand{\bysame}{\leavevmode\hbox to3em{\hrulefill}\thinspace}
\providecommand{\noopsort}[1]{}
\providecommand{\mr}[1]{\href{http://www.ams.org/mathscinet-getitem?mr=#1}{MR~#1}}
\providecommand{\zbl}[1]{\href{http://www.zentralblatt-math.org/zmath/en/search/?q=an:#1}{Zbl~#1}}
\providecommand{\jfm}[1]{\href{http://www.emis.de/cgi-bin/JFM-item?#1}{JFM~#1}}
\providecommand{\arxiv}[1]{\href{http://www.arxiv.org/abs/#1}{arXiv~#1}}
\providecommand{\doi}[1]{\url{http://dx.doi.org/#1}}
\providecommand{\MR}{\relax\ifhmode\unskip\space\fi MR }
\providecommand{\MRhref}[2]{%
  \href{http://www.ams.org/mathscinet-getitem?mr=#1}{#2}
}
\providecommand{\href}[2]{#2}
\begin{thebibliography}{MTTH{\etalchar{+}}13}

\bibitem[blo]{blockchain}
{A Decentralized Kaggle: Inside Algorithmia’s Approach to Blockchain-Based AI
  Competitions},
  \url{https://towardsdatascience.com/a-decentralized-kaggle-inside-algorithmias-approach-to-blockchain-based-ai-competitions-8c6aec99e89b}.

\bibitem[DAW]{DAWEX}
Dawex, \url{https://www.dawex.com/en/}.

\bibitem[BIG]{BIGQUERY}
Google bigquery, \url{https://cloud.google.com/bigquery/}.

\bibitem[IOT]{IOTA}
Iota, \url{https://data.iota.org/#/}.

\bibitem[xig]{xignite}
{Xignite}, \url{https://apollomapping.com/}.

\bibitem[AWY16]{adeniyi2016automated}
\bgroup\scshape{}D.~Adeniyi\egroup{}, \bgroup\scshape{}Z.~Wei\egroup{}, and
  \bgroup\scshape{}Y.~Yongquan\egroup{}, Automated web usage data mining and
  recommendation system using k-nearest neighbor (knn) classification method,
  \emph{Applied Computing and Informatics} \textbf{12} (2016), 90--108.

\bibitem[AFGR16]{amato2016yfcc100m}
\bgroup\scshape{}G.~Amato\egroup{}, \bgroup\scshape{}F.~Falchi\egroup{},
  \bgroup\scshape{}C.~Gennaro\egroup{}, and
  \bgroup\scshape{}F.~Rabitti\egroup{}, Yfcc100m-hnfc6: a large-scale deep
  features benchmark for similarity search,  in \emph{International Conference
  on Similarity Search and Applications}, Springer, 2016, pp.~196--209.

\bibitem[BKZ05]{bartholdi2005using}
\bgroup\scshape{}J.~J. Bartholdi\egroup{} and
  \bgroup\scshape{}E.~Kemahlio{\u{g}}lu-Ziya\egroup{}, Using shapley value to
  allocate savings in a supply chain,  in \emph{Supply chain optimization},
  Springer, 2005, pp.~169--208.

\bibitem[BS13]{bremer2013estimating}
\bgroup\scshape{}J.~Bremer\egroup{} and
  \bgroup\scshape{}M.~Sonnenschein\egroup{}, Estimating shapley values for fair
  profit distribution in power planning smart grid coalitions,  in \emph{German
  Conference on Multiagent System Technologies}, Springer, 2013, pp.~208--221.

\bibitem[Cha02]{charikar2002similarity}
\bgroup\scshape{}M.~S. Charikar\egroup{}, Similarity estimation techniques from
  rounding algorithms,  in \emph{Proceedings of the thiry-fourth annual ACM
  symposium on Theory of computing}, ACM, 2002, pp.~380--388.

\bibitem[CKK17]{chen2017model}
\bgroup\scshape{}L.~Chen\egroup{}, \bgroup\scshape{}P.~Koutris\egroup{}, and
  \bgroup\scshape{}A.~Kumar\egroup{}, Model-based pricing: Do not pay for more
  than what you learn!,  in \emph{Proceedings of the 1st Workshop on Data
  Management for End-to-End Machine Learning}, ACM, 2017, p.~1.

\bibitem[CKK18]{chen2018model}
\bgroup\scshape{}L.~Chen\egroup{}, \bgroup\scshape{}P.~Koutris\egroup{}, and
  \bgroup\scshape{}A.~Kumar\egroup{}, Model-based pricing for machine learning
  in a data marketplace,  \emph{arXiv preprint arXiv:1805.11450} (2018).

\bibitem[CL17]{chessa2017cooperative}
\bgroup\scshape{}M.~Chessa\egroup{} and \bgroup\scshape{}P.~Loiseau\egroup{}, A
  cooperative game-theoretic approach to quantify the value of personal data in
  networks,  in \emph{Proceedings of the 12th workshop on the Economics of
  Networks, Systems and Computation}, ACM, 2017, p.~9.

\bibitem[DAMZ18]{dao2018databright}
\bgroup\scshape{}D.~Dao\egroup{}, \bgroup\scshape{}D.~Alistarh\egroup{},
  \bgroup\scshape{}C.~Musat\egroup{}, and \bgroup\scshape{}C.~Zhang\egroup{},
  Databright: Towards a global exchange for decentralized data ownership and
  trusted computation,  \emph{arXiv preprint arXiv:1802.04780} (2018).

\bibitem[DIIM04]{datar2004locality}
\bgroup\scshape{}M.~Datar\egroup{}, \bgroup\scshape{}N.~Immorlica\egroup{},
  \bgroup\scshape{}P.~Indyk\egroup{}, and \bgroup\scshape{}V.~S.
  Mirrokni\egroup{}, Locality-sensitive hashing scheme based on p-stable
  distributions,  in \emph{Proceedings of the twentieth annual symposium on
  Computational geometry}, ACM, 2004, pp.~253--262.

\bibitem[DKB17]{deep2017qirana}
\bgroup\scshape{}S.~Deep\egroup{}, \bgroup\scshape{}P.~Koutris\egroup{}, and
  \bgroup\scshape{}Y.~Bidasaria\egroup{}, Qirana demonstration: real time
  scalable query pricing,  \emph{PVLDB} \textbf{10} (2017), 1949--1952.

\bibitem[DDS{\etalchar{+}}09]{imagenet_cvpr09}
\bgroup\scshape{}J.~Deng\egroup{}, \bgroup\scshape{}W.~Dong\egroup{},
  \bgroup\scshape{}R.~Socher\egroup{}, \bgroup\scshape{}L.-J. Li\egroup{},
  \bgroup\scshape{}K.~Li\egroup{}, and \bgroup\scshape{}L.~Fei-Fei\egroup{},
  {ImageNet: A Large-Scale Hierarchical Image Database},  in \emph{CVPR09},
  2009.

\bibitem[DP94]{deng1994complexity}
\bgroup\scshape{}X.~Deng\egroup{} and \bgroup\scshape{}C.~H.
  Papadimitriou\egroup{}, On the complexity of cooperative solution concepts,
  \emph{Mathematics of Operations Research} \textbf{19} (1994), 257--266.

\bibitem[Dud76]{dudani1976distance}
\bgroup\scshape{}S.~A. Dudani\egroup{}, The distance-weighted
  k-nearest-neighbor rule,  \emph{IEEE Transactions on Systems, Man, and
  Cybernetics} (1976), 325--327.

\bibitem[FWJ08]{fatima2008linear}
\bgroup\scshape{}S.~S. Fatima\egroup{},
  \bgroup\scshape{}M.~Wooldridge\egroup{}, and \bgroup\scshape{}N.~R.
  Jennings\egroup{}, A linear approximation method for the shapley value,
  \emph{Artificial Intelligence} \textbf{172} (2008), 1673--1699.

\bibitem[GZ19]{ghorbani2019data}
\bgroup\scshape{}A.~Ghorbani\egroup{} and \bgroup\scshape{}J.~Zou\egroup{},
  Data shapley: Equitable valuation of data for machine learning,  \emph{arXiv
  preprint arXiv:1904.02868} (2019).

\bibitem[GIM{\etalchar{+}}99]{gionis1999similarity}
\bgroup\scshape{}A.~Gionis\egroup{}, \bgroup\scshape{}P.~Indyk\egroup{},
  \bgroup\scshape{}R.~Motwani\egroup{}, and \bgroup\scshape{}others\egroup{},
  Similarity search in high dimensions via hashing,  in \emph{Vldb},
  \textbf{99}, 1999, pp.~518--529.

\bibitem[HPIM12]{har2012approximate}
\bgroup\scshape{}S.~Har-Peled\egroup{}, \bgroup\scshape{}P.~Indyk\egroup{}, and
  \bgroup\scshape{}R.~Motwani\egroup{}, Approximate nearest neighbor: Towards
  removing the curse of dimensionality,  \emph{Theory of computing} \textbf{8}
  (2012), 321--350.

\bibitem[HPS{\etalchar{+}}16]{hardt2016equality}
\bgroup\scshape{}M.~Hardt\egroup{}, \bgroup\scshape{}E.~Price\egroup{},
  \bgroup\scshape{}N.~Srebro\egroup{}, and \bgroup\scshape{}others\egroup{},
  Equality of opportunity in supervised learning,  in \emph{Advances in neural
  information processing systems}, 2016, pp.~3315--3323.

\bibitem[HE15]{hays2015large}
\bgroup\scshape{}J.~Hays\egroup{} and \bgroup\scshape{}A.~A. Efros\egroup{},
  Large-scale image geolocalization,  in \emph{Multimodal Location Estimation
  of Videos and Images}, Springer, 2015, pp.~41--62.

\bibitem[HKC12]{he2012difficulty}
\bgroup\scshape{}J.~He\egroup{}, \bgroup\scshape{}S.~Kumar\egroup{}, and
  \bgroup\scshape{}S.-F. Chang\egroup{}, On the difficulty of nearest neighbor
  search,  in \emph{Proceedings of the 29th International Coference on
  International Conference on Machine Learning}, Omnipress, 2012, pp.~41--48.

\bibitem[HZRS16]{he2016deep}
\bgroup\scshape{}K.~He\egroup{}, \bgroup\scshape{}X.~Zhang\egroup{},
  \bgroup\scshape{}S.~Ren\egroup{}, and \bgroup\scshape{}J.~Sun\egroup{}, Deep
  residual learning for image recognition,  in \emph{Proceedings of the IEEE
  conference on computer vision and pattern recognition}, 2016, pp.~770--778.

\bibitem[HDY{\etalchar{+}}18]{hynes2018demonstration}
\bgroup\scshape{}N.~Hynes\egroup{}, \bgroup\scshape{}D.~Dao\egroup{},
  \bgroup\scshape{}D.~Yan\egroup{}, \bgroup\scshape{}R.~Cheng\egroup{}, and
  \bgroup\scshape{}D.~Song\egroup{}, A demonstration of sterling: a
  privacy-preserving data marketplace,  \emph{PVLDB} \textbf{11} (2018),
  2086--2089.

\bibitem[JDW{\etalchar{+}}19]{jia2018shapley}
\bgroup\scshape{}R.~Jia\egroup{}, \bgroup\scshape{}D.~Dao\egroup{},
  \bgroup\scshape{}B.~Wang\egroup{}, \bgroup\scshape{}F.~A. Hubis\egroup{},
  \bgroup\scshape{}N.~Hynes\egroup{}, \bgroup\scshape{}B.~Li\egroup{},
  \bgroup\scshape{}C.~Zhang\egroup{}, \bgroup\scshape{}D.~Song\egroup{}, and
  \bgroup\scshape{}C.~Spanos\egroup{}, Towards efficient data valuation based
  on the shapley value,  \emph{AISTATS} (2019).

\bibitem[KPR01]{kleinberg2001value}
\bgroup\scshape{}J.~Kleinberg\egroup{}, \bgroup\scshape{}C.~H.
  Papadimitriou\egroup{}, and \bgroup\scshape{}P.~Raghavan\egroup{}, On the
  value of private information,  in \emph{Proceedings of the 8th conference on
  Theoretical aspects of rationality and knowledge}, Morgan Kaufmann Publishers
  Inc., 2001, pp.~249--257.

\bibitem[KL17]{koh2017understanding}
\bgroup\scshape{}P.~W. Koh\egroup{} and \bgroup\scshape{}P.~Liang\egroup{},
  Understanding black-box predictions via influence functions,  in
  \emph{International Conference on Machine Learning}, 2017, pp.~1885--1894.

\bibitem[KUB{\etalchar{+}}12]{koutris2012querymarket}
\bgroup\scshape{}P.~Koutris\egroup{}, \bgroup\scshape{}P.~Upadhyaya\egroup{},
  \bgroup\scshape{}M.~Balazinska\egroup{}, \bgroup\scshape{}B.~Howe\egroup{},
  and \bgroup\scshape{}D.~Suciu\egroup{}, Querymarket demonstration: Pricing
  for online data markets,  \emph{PVLDB} \textbf{5} (2012), 1962--1965.

\bibitem[KUB{\etalchar{+}}13]{koutris2013toward}
\bgroup\scshape{}P.~Koutris\egroup{}, \bgroup\scshape{}P.~Upadhyaya\egroup{},
  \bgroup\scshape{}M.~Balazinska\egroup{}, \bgroup\scshape{}B.~Howe\egroup{},
  and \bgroup\scshape{}D.~Suciu\egroup{}, Toward practical query pricing with
  querymarket,  in \emph{proceedings of the 2013 ACM SIGMOD international
  conference on management of data}, ACM, 2013, pp.~613--624.

\bibitem[KUB{\etalchar{+}}15]{koutris2015query}
\bgroup\scshape{}P.~Koutris\egroup{}, \bgroup\scshape{}P.~Upadhyaya\egroup{},
  \bgroup\scshape{}M.~Balazinska\egroup{}, \bgroup\scshape{}B.~Howe\egroup{},
  and \bgroup\scshape{}D.~Suciu\egroup{}, Query-based data pricing,
  \emph{Journal of the ACM (JACM)} \textbf{62} (2015), 43.

\bibitem[LC10]{lecun-mnisthandwrittendigit-2010}
\bgroup\scshape{}Y.~LeCun\egroup{} and \bgroup\scshape{}C.~Cortes\egroup{},
  {MNIST} handwritten digit database,  (2010). Available at
  \url{http://yann.lecun.com/exdb/mnist/}.

\bibitem[LLMS17]{li2017theory}
\bgroup\scshape{}C.~Li\egroup{}, \bgroup\scshape{}D.~Y. Li\egroup{},
  \bgroup\scshape{}G.~Miklau\egroup{}, and \bgroup\scshape{}D.~Suciu\egroup{},
  A theory of pricing private data,  \emph{Communications of the ACM}
  \textbf{60} (2017), 79--86.

\bibitem[LM12]{li2012pricing}
\bgroup\scshape{}C.~Li\egroup{} and \bgroup\scshape{}G.~Miklau\egroup{},
  Pricing aggregate queries in a data marketplace.,  in \emph{WebDB}, 2012,
  pp.~19--24.

\bibitem[LZZ{\etalchar{+}}12]{li2012using}
\bgroup\scshape{}C.~Li\egroup{}, \bgroup\scshape{}S.~Zhang\egroup{},
  \bgroup\scshape{}H.~Zhang\egroup{}, \bgroup\scshape{}L.~Pang\egroup{},
  \bgroup\scshape{}K.~Lam\egroup{}, \bgroup\scshape{}C.~Hui\egroup{}, and
  \bgroup\scshape{}S.~Zhang\egroup{}, Using the k-nearest neighbor algorithm
  for the classification of lymph node metastasis in gastric cancer,
  \emph{Computational and mathematical methods in medicine} \textbf{2012}
  (2012).

\bibitem[LK14]{lin2014arbitrage}
\bgroup\scshape{}B.-R. Lin\egroup{} and \bgroup\scshape{}D.~Kifer\egroup{}, On
  arbitrage-free pricing for general data queries,  \emph{PVLDB} \textbf{7}
  (2014), 757--768.

\bibitem[Mal15]{maleki2015addressing}
\bgroup\scshape{}S.~Maleki\egroup{}, \emph{Addressing the computational issues
  of the Shapley value with applications in the smart grid}, Ph.D. thesis,
  University of Southampton, 2015.

\bibitem[MTTH{\etalchar{+}}13]{maleki2013bounding}
\bgroup\scshape{}S.~Maleki\egroup{}, \bgroup\scshape{}L.~Tran-Thanh\egroup{},
  \bgroup\scshape{}G.~Hines\egroup{}, \bgroup\scshape{}T.~Rahwan\egroup{}, and
  \bgroup\scshape{}A.~Rogers\egroup{}, Bounding the estimation error of
  sampling-based shapley value approximation,  \emph{arXiv preprint
  arXiv:1306.4265} (2013).

\bibitem[MAS{\etalchar{+}}13]{michalak2013efficient}
\bgroup\scshape{}T.~P. Michalak\egroup{}, \bgroup\scshape{}K.~V.
  Aadithya\egroup{}, \bgroup\scshape{}P.~L. Szczepanski\egroup{},
  \bgroup\scshape{}B.~Ravindran\egroup{}, and \bgroup\scshape{}N.~R.
  Jennings\egroup{}, Efficient computation of the shapley value for
  game-theoretic network centrality,  \emph{Journal of Artificial Intelligence
  Research} \textbf{46} (2013), 607--650.

\bibitem[MA98]{mount1998ann}
\bgroup\scshape{}D.~M. Mount\egroup{} and \bgroup\scshape{}S.~Arya\egroup{},
  Ann: library for approximate nearest neighbour searching,  (1998).

\bibitem[OST13]{ogawa2013safe}
\bgroup\scshape{}K.~Ogawa\egroup{}, \bgroup\scshape{}Y.~Suzuki\egroup{}, and
  \bgroup\scshape{}I.~Takeuchi\egroup{}, Safe screening of non-support vectors
  in pathwise svm computation,  in \emph{International Conference on Machine
  Learning}, 2013, pp.~1382--1390.

\bibitem[RVSS19]{raskar2019data}
\bgroup\scshape{}R.~Raskar\egroup{}, \bgroup\scshape{}P.~Vepakomma\egroup{},
  \bgroup\scshape{}T.~Swedish\egroup{}, and
  \bgroup\scshape{}A.~Sharan\egroup{}, Data markets to support ai for all:
  Pricing, valuation and governance,  \emph{arXiv preprint arXiv:1905.06462}
  (2019).

\bibitem[Sha53]{shapley1953value}
\bgroup\scshape{}L.~S. Shapley\egroup{}, A value for n-person games,
  \emph{Contributions to the Theory of Games} \textbf{2} (1953), 307--317.

\bibitem[SI07]{siagian2007rapid}
\bgroup\scshape{}C.~Siagian\egroup{} and \bgroup\scshape{}L.~Itti\egroup{},
  Rapid biologically-inspired scene classification using features shared with
  visual attention,  \emph{IEEE transactions on pattern analysis and machine
  intelligence} \textbf{29} (2007), 300--312.

\bibitem[SVI{\etalchar{+}}16]{szegedy2016rethinking}
\bgroup\scshape{}C.~Szegedy\egroup{}, \bgroup\scshape{}V.~Vanhoucke\egroup{},
  \bgroup\scshape{}S.~Ioffe\egroup{}, \bgroup\scshape{}J.~Shlens\egroup{}, and
  \bgroup\scshape{}Z.~Wojna\egroup{}, Rethinking the inception architecture for
  computer vision,  in \emph{Proceedings of the IEEE conference on computer
  vision and pattern recognition}, 2016, pp.~2818--2826.

\bibitem[Top06]{topsok2006some}
\bgroup\scshape{}F.~Topsok\egroup{}, Some bounds for the logarithmic function,
  \emph{Inequality theory and applications} \textbf{4} (2006), 137.

\bibitem[UBS12]{upadhyaya2012price}
\bgroup\scshape{}P.~Upadhyaya\egroup{},
  \bgroup\scshape{}M.~Balazinska\egroup{}, and
  \bgroup\scshape{}D.~Suciu\egroup{}, How to price shared optimizations in the
  cloud,  \emph{PVLDB} \textbf{5} (2012), 562--573.

\bibitem[UBS16]{upadhyaya2016price}
\bgroup\scshape{}P.~Upadhyaya\egroup{},
  \bgroup\scshape{}M.~Balazinska\egroup{}, and
  \bgroup\scshape{}D.~Suciu\egroup{}, Price-optimal querying with data apis,
  \emph{PVLDB} \textbf{9} (2016), 1695--1706.

\bibitem[WGOS17]{woodworth2017learning}
\bgroup\scshape{}B.~Woodworth\egroup{}, \bgroup\scshape{}S.~Gunasekar\egroup{},
  \bgroup\scshape{}M.~I. Ohannessian\egroup{}, and
  \bgroup\scshape{}N.~Srebro\egroup{}, Learning non-discriminatory predictors,
  \emph{arXiv preprint arXiv:1702.06081} (2017).

\bibitem[ZWS{\etalchar{+}}13]{zemel2013learning}
\bgroup\scshape{}R.~Zemel\egroup{}, \bgroup\scshape{}Y.~Wu\egroup{},
  \bgroup\scshape{}K.~Swersky\egroup{}, \bgroup\scshape{}T.~Pitassi\egroup{},
  and \bgroup\scshape{}C.~Dwork\egroup{}, Learning fair representations,  in
  \emph{International Conference on Machine Learning}, 2013, pp.~325--333.

\bibitem[ZPW{\etalchar{+}}19]{zheng2019arete}
\bgroup\scshape{}Z.~Zheng\egroup{}, \bgroup\scshape{}Y.~Peng\egroup{},
  \bgroup\scshape{}F.~Wu\egroup{}, \bgroup\scshape{}S.~Tang\egroup{}, and
  \bgroup\scshape{}G.~Chen\egroup{}, Arete: On designing joint online pricing
  and reward sharing mechanisms for mobile data markets,  \emph{IEEE
  Transactions on Mobile Computing} (2019).

\end{thebibliography}
\bibliographystyle{aomalpha}

\onecolumn

\appendix

\section{Additional Experiments}

\subsection{Runtime Comparision for Computing the Unweighted $K$NN SV}
\label{appendix:runtime_k25}

For each dataset, we randomly selected $100$ test points, computed the SV of all training points with respect to each test point, and reported the average runtime across all test points. The results for $K=2,5$ are presented in Figure~\ref{table:runtime_k_25}. We can see that the LSH-based method can bring a $3\times$-$5\times$ speed-up compared with the exact algorithm. 

\begin{figure}[h!]
\caption{Average runtime of the exact and the LSH-based approximation algorithm for computing the unweighted $K$NN SV for a single test point. We take $\epsilon,\delta=0.1$ and $K=2,5$.}
\label{table:runtime_k_25}
\begin{tabular}{ccccccc}
\hline
\multirow{2}{*}{\textbf{Dataset}} & \multirow{2}{*}{\textbf{Size}} & \multirow{2}{*}{\textbf{\begin{tabular}[c]{@{}c@{}}Estimated\\ Contrast\end{tabular}}} & \multicolumn{2}{c}{\textbf{K=2}}                                                                                                      & \multicolumn{2}{c}{\textbf{K=5}}                                                                                                      \\ \cline{4-7} 
                                  &                                &                                                                                        & \textbf{\begin{tabular}[c]{@{}c@{}}Runtime\\ (Exact)\end{tabular}} & \textbf{\begin{tabular}[c]{@{}c@{}}Runtime\\ (LSH)\end{tabular}} & \textbf{\begin{tabular}[c]{@{}c@{}}Runtime\\ (Exact)\end{tabular}} & \textbf{\begin{tabular}[c]{@{}c@{}}Runtime\\ (LSH)\end{tabular}} \\ \hline
\texttt{CIFAR-10  }                        & $6$E+$4$                       & 1.2802                                                                                 & 0.83                                                               & 0.25                                                             & 0.82                                                               & 0.26                                                             \\ \hline
\texttt{ImageNet}                          & $1$E+$6$                       & 1.2163                                                                                 & 12.71                                                              & 3.29                                                             & 12.57                                                              & 3.25                                                             \\ \hline
\texttt{Yahoo10m}                          & $1$E+$7$                       & 1.3456                                                                                 & 198.73                                                             & 41.83                                                            & 200.06                                                             & 39.20                                                            \\ \hline
\end{tabular}

\end{figure}


    
    

\section{Proof of Lemma~\ref{lm:shapley_diff}}

\begin{proof}
\begin{align}
 &s_i-s_j \\
 &= \sum_{S\subseteq I\setminus\{i\}} \frac{|S|!(N-|S|-1)!}{N!} \big[\nu(S\cup \{i\})-\nu(S)\big]- \sum_{S\subseteq I\setminus\{j\}} \frac{|S|!(N-|S|-1)!}{N!} \big[\nu(S\cup \{j\})-\nu(S)\big]   \\
 &=\sum_{S\subseteq I\setminus\{i,j\}} \frac{|S|!(N-|S|-1)!}{N!} \big[\nu(S\cup\{i\}) - \nu(S\cup \{j\})\big] \nonumber\\
 &\quad\quad+\sum_{S\in\{T|T\subseteq I,i\notin T, j\in T\}}  \frac{|S|!(N-|S|-1)!}{N!}   \big[\nu(S\cup \{i\})-\nu(S)\big]\nonumber \\
 &\quad\quad-\sum_{S\in\{T|T\subseteq I,i\in T, j\notin T\}} \frac{|S|!(N-|S|-1)!}{N!} \big[\nu(S\cup \{j\})-\nu(S)\big]\\
 &=\sum_{S\subseteq I\setminus\{i,j\}} \frac{|S|!(N-|S|-1)!}{N!} \big[\nu(S\cup\{i\}) - \nu(S\cup \{j\})\big] \nonumber\\
 &\quad\quad+ \sum_{S'\subseteq I\setminus\{i,j\}} \frac{(|S'|+1)!(N-|S'|-2)!}{N!} \big[\nu(S'\cup\{i\}) - \nu(S'\cup \{j\})\big]\\
 &= \sum_{S\subseteq I\setminus\{i,j\}} \big( \frac{|S|!(N-|S|-1)!}{N!} + \frac{(|S|+1)!(N-|S|-2)!}{N!}\big)  \big[\nu(S\cup\{i\}) - \nu(S\cup \{j\})\big]\\
 &=\frac{1}{N-1}  \sum_{S\subseteq I\setminus\{i,j\}} \frac{1}{C_{N-2}^{|S|}}\big[\nu(S\cup\{i\}) - \nu(S\cup \{j\})\big] \, .
\end{align}
\end{proof}

\section{Proof of Theorem~\ref{thm:knn_shapley_approx}}

\begin{proof}
We first observe that if the true Shapley value $|s_{\alpha_i}|\leq \min(\frac{1}{i},\frac{1}{K})$, then $|s_i|\leq \epsilon$ for $i\geq i^*=\max(K,\ceil*{1/\epsilon})$. Hence, when $i\geq i^*$, the approximation error is given by
\begin{align}
   |\hat{s}_{\alpha_i}-s_{\alpha_i}| = | s_{\alpha_i}|\leq \epsilon.
\end{align}
When $i\leq i^*-1$, $\hat{s}_{\alpha_i}$ and $s_{\alpha_i}$ follow the same recursion, i.e.,
\begin{align}
  \hat{s}_{\alpha_i}-\hat{s}_{\alpha_{i+1}}= s_{\alpha_i}-s_{\alpha_{i+1}} = \frac{\mathbbm{1}[y_{\alpha_{i}}=y_\text{test}]-\mathbbm{1}[y_{\alpha_{i+1}}=y_\text{test}]}{K} \frac{\min(K-1,i-1)+1}{i}.
\end{align} 
As a result, we have
\begin{align}
    |\hat{s}_{\alpha_i}-s_{\alpha_i}| = |\hat{s}_{\alpha_{i+1}}-s_{\alpha_{i+1}}|=\cdots =|\hat{s}_{\alpha_{i^*}}-s_{\alpha_{i^*}}| \leq \epsilon
\end{align}
To sum up, $|\hat{s}_{\alpha_i}-s_{\alpha_i}|\leq \epsilon$ for all $i=1,\ldots,N$, provided that $|s_{\alpha_i}|\leq \min(\frac{1}{i},\frac{1}{K})$. In the following, we will prove that the aforementioned condition is satisfied.

We can convert the recursive expression of the $K$NN Shapley value in Theorem~\ref{thm:KNN_unweighted_class} to a non-recursive one:
\begin{align}
    &s_{\alpha_N} = \frac{\mathbbm{1}[y_{\alpha_N}=y_\text{test}]}{N}\\
    \label{eqn:knn_shapley_greater_K}
    &s_{\alpha_i} = \frac{\mathbbm{1}[y_{\alpha_i}=y_\text{test}]}{i} - \sum_{j=i+1}^{N}\frac{\mathbbm{1}[y_{\alpha_{j}}=y_\text{test}]}{j(j-1)} \text{ for $i\geq K$}\\
    \label{eqn:knn_shapley_less_K}
    &s_{\alpha_i} = \frac{\mathbbm{1}[y_{\alpha_i}=y_\text{test}]}{K} - \sum_{j=K+1}^{N}\frac{\mathbbm{1}[y_{\alpha_{j}}=y_\text{test}]}{j(j-1)} \text{ for $i\leq K-1$}
\end{align}

We examine the bound on the absolute value of the Shapley value in three cases: (1) $i=N$, (2) $i\geq K$, and (3) $i\leq K-1$.

\textbf{Case (1).} It is easy to verify that $|s_{\alpha_N}|\leq \frac{1}{N}$.

\textbf{Case (2).} We can bound the second term in (\ref{eqn:knn_shapley_greater_K}) by
\begin{align}
\label{eqn:bound_second_term}
    0 \leq \sum_{j=i+1}^{N}\frac{\mathbbm{1}[y_{\alpha_{j}}=y_\text{test}]}{j(j-1)} \leq \sum_{j=i+1}^N \frac{1}{j(j-1)}  = \sum_{j=i+1}^N (\frac{1}{j-1} - \frac{1}{j}) = \frac{1}{i}-\frac{1}{N}
\end{align}
Thus, $s_{\alpha_i}$ can be bounded by
\begin{align}
    -(\frac{1}{i}-\frac{1}{N})\leq s_{\alpha_i} \leq \frac{1}{i},
\end{align}
which yields the bound on the absolute value of $s_{\alpha_i}$:
\begin{align}
    |s_{\alpha_i}|\leq \frac{1}{i}.
\end{align}

\textbf{Case (3).} The absolute value of $s_{\alpha_i}$ for $i\leq K-1$ can be bounded using a similar technique as in Case (2). By (\ref{eqn:knn_shapley_less_K}), we have
\begin{align}
   -(\frac{1}{K}-\frac{1}{N})\leq s_{\alpha_i} \leq \frac{1}{K}
\end{align}
Therefore, $|s_{\alpha_i}|\leq 1/K$.

Summarizing the results in Case (1), (2), and (3), we obtain $|s_{\alpha_i}|\leq \min(1/i,1/K)$ for $i=1,\ldots,N$.

\end{proof}

\section{Proof of Theorem~\ref{thm:lsh_complexity}}

\begin{proof}
For the hashing function $h(x)=\floor*{\frac{w^Tx+b}{t}}$, \cite{datar2004locality} have shown that
\begin{align}
\label{eqn:lhs_l2}
    P(h(x_i)=h(x_\text{test}))=f_h(\|x_i-x_\text{test}\|_p)
\end{align}
where the function $f_h(a)=\int_0^t \frac{1}{a}f_p(\frac{z}{a}(1-\frac{z}{t})dz$ is monotonically decreasing with $a$. $f_p$ is the probability density function of the absolute value of a $p$-stable random variable.

Suppose the data are normalized by a factor such that $D_{\text{mean}}=1$. Since such a normalization does not change the nearest neighbor search results, $D_{k}=1/C_k$ for $k=1,\ldots,K$. Denote the probability for one random test point $x_\text{test}$ and a random training point to have the same code with one hash function by $p_\text{rand}$ and the probability for $x_\text{test}$ and its $k$-nearest neighbor to have the same code by $p_\text{nn}^k$. According to (\ref{eqn:lhs_l2}), 
\begin{align}
    p_\text{rand} = f_h(1)
\end{align}
and 
\begin{align}
    p_{\text{nn},k} = f_h(1/C_k)
\end{align}
because the expected distance between $x_\text{test}$ and a random training point is $D_\text{mean}=1$, and the expected distance between $x_\text{test}$ and its $k$-nearest neighbor is $1/C_k$. 

Let $E_k$ denote the event that the $k$-nearest neighbor of $x_\text{test}$ is included by one of the hash tables. Then, the probability of the inclusion of all $K$ nearest neighbors is
\begin{align}
    P(E_1,\ldots,E_K)& = 1 - P(\cup_{k=1}^K \bar{E}_k) \\
    &\geq 1 - \sum_{k=1}^K P(\bar{E}_k).
\end{align}
We want to make sure that $P(E_1,\ldots,E_K) \geq 1-\delta$, so it suffices to let $P(\bar{E}_k)\leq \delta/K$ for all $k=1,\ldots,K$.

Suppose there are $m$ hash bits in one table and $l$ hash tables in LSH. The probability that the true $k$-nearest neighbor has the same code as the query in one hash table is $p_{\text{nn},k}^m$. Hence, the probability that the true $k$-nearest neighbor is missed by $l$ hash tables is $P(\bar{E}_k) = (1-p_{\text{nn},k}^m)^l$. In order to ensure $P(\bar{E}_k)\leq \delta/K$, we need 
\begin{align}
\label{eqn:num_hash_tables}
    l\geq \frac{\log \frac{\delta}{K}}{\log (1-p_{\text{nn},k}^m)} 
\end{align}
The RHS is upper bounded by $ \frac{-\log \frac{\delta}{K}}{p_{\text{nn},k}^m} = p_{\text{nn},k}^{-m}\log \frac{K}{\delta}$. Therefore, it suffices to ensure 
\begin{align}
    l\geq p_{\text{nn},k}^{-m}\log \frac{K}{\delta}
\end{align}
 Note that $p_{\text{nn},k} = p_\text{rand}^{\frac{\log p_{\text{nn},k}}{\log p_\text{rand}}}$ and we can choose $Np_\text{rand}^m = \OO(1)$, i.e., $m=\OO(\frac{\log N}{\log p_\text{rand}^{-1}})$, as discussed in~\cite{gionis1999similarity}.  Hence, \begin{align}
\label{eqn:prob_nn}
   p_{\text{nn},k}^m =  p_\text{rand}^{m\frac{\log p_{\text{nn},k}}{\log p_\text{rand}}} = \OO((\frac{1}{N})^{\frac{\log p_{\text{nn},k}}{\log p_\text{rand}}}) = \OO(N^{-g(C_k)})
\end{align}
where $g(C_k) =\frac{\log p_{\text{nn},k}}{\log p_\text{rand}} = \frac{\log f_h(1/C_k)}{\log f_h(1)} $. Plugging (\ref{eqn:prob_nn}) into (\ref{eqn:num_hash_tables}), we obtain
\begin{align}
    l\geq \OO(N^{g(C_k)}\log \frac{K}{\delta})
\end{align}

In order to guarantee $P(\bar{E}_k)\leq \delta/K$ for all $k=1,\cdots, K$, the number of hash tables needed is 
\begin{align}
   \OO( N^{g(C_K)} \log \frac{K}{\delta})
\end{align}


\end{proof}

\section{Detailed Algorithms and Proofs for the Extensions}
We first extend the algorithms to calculate the SV for unweighted $K$NN regression and weighted $K$NN.
Further, we address the data valuation problem wherein a data curator contributes multiple data points. We then discuss how to valuate the parties offering computation in the data market.

\subsection{Unweighted $K$NN Regression}
\label{section:appendix_regression}
For regression tasks, we define the utility function by the negative mean square error of an unweighted $K$NN regressor:
\begin{align}
\label{eqn:utility_regression_dup}
 \nu(S) = -\bigg(\frac{1}{K}\sum_{k=1}^{\min\{K,|S|\}}  y_{\alpha_k(S)} - y_\text{test}\bigg)^2
\end{align}
The following theorem provides a simple iterative procedure to compute the SV for unweighted $K$NN regression. The derivation of the theorem requires to analyze the utility difference between two adjacent training points, similar to $K$NN classification. 


\begin{theorem}
\label{thm:KNN_unweighted_regression}
Consider the $K$NN regression utility function in (\ref{eqn:utility_regression_dup}). Then, the SV of each training point can be calculated recursively as follows:
\begin{align}
\label{eqn:KNN_regression_1}
&s_{\alpha_N} =-\frac{K-1}{NK}y_{\alpha_N}\bigg[\frac{1}{K}y_{\alpha_N}-2y_\text{test}+ \frac{1}{N-1}\sum_{l\in I\setminus\{N\}} y_{\alpha_l}\bigg] -\frac{1}{N}\bigg[\frac{1}{K} y_{\alpha_N} - y_\text{test}\bigg]^2
\end{align}
\begin{align}
\label{eqn:KNN_regression_2}
 &s_{\alpha_i} = s_{\alpha_{i+1}}+\frac{1}{K} (y_{\alpha_{i+1}} - y_{\alpha_i})\frac{\min\{K,i\}}{i}  (\frac{1}{K}\sum_{l=1}^N A_i^{(l)} y_{\alpha_l}-2y_\text{test})
\end{align}
where
\begin{align}
    A_i^{(l)}=\left\{\begin{array}{ll}
         \frac{\min\{K-1,i-1\}}{i-1}& \text{ if $1\leq l\leq i-1$}\\
         1& \text{ if $l\in\{i,i+1\}$}\\
          \frac{\min\{K,l-1\}\min\{K-1,l-2\} i}{(l-1)(l-2)\min\{K,i\}}& \text{ if $i+2\leq l \leq N$}
    \end{array}
    \right.
\end{align}
\end{theorem}

According to (\ref{eqn:KNN_regression_2}), two adjacent training points will have the same SV if they have the same label. Otherwise, their SV difference will depend on three terms: (1) their difference in the labels $y_{\alpha_{i+1}}-y_{\alpha_i}$, (2) the rank of their distances to the test point $\frac{\min(K,i)}{i}$, and (3) the goodness of fit term $\frac{1}{K}\sum_{l=1}^N A_i^{(l)} y_{\alpha_l}-2y_\text{test}$ of a ``weighted'' $K$NN regression model in which $A_i^{(l)}$ stands for the weight. By simple algebraic operations, it can be obtained that $y_{\alpha_i}$ and $y_{\alpha_{i+1}}$ are weighted highest among all training points; therefore, the third term can be roughly thought of as how much error $y_{\alpha_i}$ and $y_{\alpha_{i+1}}$ induce for predicting $y_\text{test}$. If the goodness of fit term represents a positive error and $y_{\alpha_i} > y_{\alpha_{i+1}}$, then adding $(x_{\alpha_i},y_{\alpha_i})$ into the training dataset will even enlarge the positive prediction error. Thus, $(x_{\alpha_i},y_{\alpha_i})$ is less valuable than $(x_{\alpha_{i+1}},y_{\alpha_{i+1}})$ in terms of the SV. Similar intuition about the interaction between the first and third term can be established when $y_{\alpha_i} < y_{\alpha_{i+1}}$. Moreover, the training points closer to the test point are more influential to the prediction result; this phenomenon is captured by the second term. In summary, the SV difference between two adjacent training points is large when their labels differ largely, their distances to the test point are small, and their presence in the training set leads to large prediction errors.

\begin{proof}[Proof of Theorem~\ref{thm:KNN_unweighted_regression}]
W.l.o.g., we
assume that $x_1,\ldots, x_n$ are
sorted according to their similarity to $x_\text{test}$, that is, $x_i = x_{\alpha_i}$. We split a subset $S\subseteq I\setminus\{i,i+1\}$ into two disjoint sets $S_1$ and $S_2$ such that $S=S_1\cup S_2$ and $S_1\cap S_2 = \emptyset$. Given two neighboring points with indices $i,i+1\in I$, we constrain $S_1$ and $S_2$ to $S_1\subseteq\{1,\ldots,i-1\}$ and $S_2\subseteq\{i+2,\ldots,N\}$.

We analyze the difference between $s_i$ and $s_{i+1}$ by considering the following cases:

\textbf{Case 1.\quad} Consider the case $|S_1|\geq K$. We know that $i>K$ and therefore $\nu(S\cup \{i\}) = \nu(S\cup \{i+1\})= \nu(S)$. From Lemma~\ref{lm:shapley_diff}, it follows that
\begin{align*}
    &s_i - s_{i+1} = \frac{1}{N-1}\sum_{k=0}^{N-2} \frac{1}{{N-2\choose k}}\sum \limits_{ \substack{  S_1 \subseteq \{1,...,i-1\}, \\ S_2 \subseteq \{i+2,...,N\}: \\ |S_1| + |S_2|=k, |S_1| \geq K} } \bigg[\nu(S\cup\{i\}) - \nu(S\cup \{i+1\})\bigg] = 0.
\end{align*}

\textbf{Case 2.\quad} Consider the case $|S_1|<K$. The difference between $\nu(S\cup \{i\})$ and $\nu(S\cup \{i+1\})$ can be expressed as
\begin{align*}
   &\nu(S\cup \{i\}) -  \nu(S\cup \{i+1\}) \\ 
   =&(\frac{1}{K}\sum_{j=1}^K y_{\alpha_j(S\cup \{i+1\})} - y_\text{test})^2 - (\frac{1}{K} \sum_{j=1}^K y_{\alpha_j(S\cup \{i\})} - y_\text{test})^2\\
   =& \frac{1}{K} (y_{i+1} - y_i) \cdot \bigg( \frac{1}{K} (y_{i+1} + y_i)- 2y_\text{test}  +\frac{2}{K} \sum_{j=1,\ldots,K-1} y_{\alpha_j(S)}\bigg)
\end{align*}
By Lemma~\ref{lm:shapley_diff}, the Shapley difference between $i$ and $i+1$ is
\begin{align*}
    &s_i - s_{i+1} = \frac{1}{K} (y_{i+1} - y_i) \\
    &\cdot \bigg(\underbrace{\frac{1}{N-1}\sum_{k=0}^{N-2} \frac{1}{{N-2\choose k}} \sum \limits_{ \substack{  S_1 \subseteq \{1,...,i-1\}, \\ S_2 \subseteq \{i+2,...,N\}: \\ |S_1| + |S_2|=k, |S_1| \leq K-1} } \big(\frac{1}{K} (y_{i+1}+ y_i) - 2y_\text{test}\big)}_{U_1}\\
    &+\underbrace{\frac{2}{K}\frac{1}{N-1}\sum_{k=0}^{N-2} \frac{1}{{N-2\choose k}} \sum \limits_{ \substack{  S_1 \subseteq \{1,...,i-1\}, \\ S_2 \subseteq \{i+2,...,N\}: \\ |S_1| + |S_2|=k, |S_1| \leq K-1} }  \sum_{j=1,\ldots,K-1} y_{\alpha_j(S)}}_{U_2}\bigg)
\end{align*}
We firstly simplify $U_1$. Note that $\frac{1}{K} (y_{i+1}+ y_i) - 2y_\text{test}$ does not depend on the summation; as a result, we have
\begin{align}
    U_1 =& \big(\frac{1}{K} (y_{i+1}+ y_i) - 2y_\text{test}\big) \frac{1}{N-1}\sum_{k=0}^{N-2} \frac{1}{{N-2\choose k}} \bigg(\sum \limits_{ \substack{  S_1 \subseteq \{1,...,i-1\}, \\ S_2 \subseteq \{i+2,...,N\}: \\ |S_1| + |S_2|=k, |S_1| \leq K-1} } 1\bigg)\nonumber\\
    \label{eqn:U_1_half}
    =&\big(\frac{1}{K} (y_{i+1}+ y_i) - 2y_\text{test}\big) \frac{1}{N-1}\sum_{k=0}^{N-2} \frac{1}{{N-2\choose k}} \sum_{m = 0}^{\min(K-1,k)} \binom{i-1}{m}  \binom{N-i-1}{k-m}
\end{align}
The sum of binomial coefficients in (\ref{eqn:U_1_half}) can be further simplified as follows: 
\begin{align*}
 &\sum_{k=0}^{N-2} \frac{1}{{N-2\choose k}} \sum_{m = 0}^{\min(K-1,k)} \binom{i-1}{m}  \binom{N-i-1}{k-m} \\
 &= \sum_{m=0}^{\min(K-1,i-1)} \sum_{k=0}^{N-i-1} \frac{\binom{i-1}{m} \binom{N-i-1}{k}}   {\binom{N-2}{m+k}}\\
 &=\sum_{m=0}^{\min(K-1,i-1)} \frac{N-1}{i}\\
 &= \min(K,i)\frac{N-1}{i}
\end{align*}
where the second equality follows from the binomial coefficient identity $ \sum_{j=0}^M \frac{{N\choose i}{M\choose j}}{{N+M\choose i+j}} = \frac{M+N+1}{N+1}$. Hence,
\begin{align*}
    U_1 = \big(\frac{1}{K} (y_{i+1}+ y_i) - 2y_\text{test}\big) \frac{\min(K,i)}{i}
\end{align*}

Then, we analyze $U_2$. We let
\begin{align}
\label{eqn:U_2_half_sum}
    \sum \limits_{ \substack{  S_1 \subseteq \{1,...,i-1\}, \\ S_2 \subseteq \{i+2,...,N\}: \\ |S_1| + |S_2|=k, |S_1| \leq K-1} }  \sum_{j=1,\ldots,K-1} y_{\alpha_j(S)} = \sum_{l \in I\setminus\{i,i+1\}} c_l y_l 
\end{align}
where $c_l$ counts the number of occurrences of $y_l$ in the left-hand side expression and 
\begin{align}
\label{eqn:c_l}
    c_l = \left\{
    \begin{array}{l}
         \sum_{m=0}^{\min(K-2,k-1)} {i-2 \choose m} {N-i-1\choose k-m-1} \text{ if $l\in \{1,\ldots,i-1\}$} \\
         \sum_{m=0}^{\min(K-2,k-1)} {l-3 \choose m}{N-l\choose k-m-1}\text{ if $l\in \{i+2,\ldots,N\}$}
    \end{array}   
    \right.   
\end{align}

Plugging in (\ref{eqn:U_2_half_sum}) and (\ref{eqn:c_l}) into $U_2$ yields
\begin{align}
    U_2 &= \frac{2}{K(N-1)} \sum_{k=0}^{N-2} \frac{1}{{N-2\choose k}}\bigg[\sum_{l\in \{1,\ldots,i-1\}}  \sum_{m=0}^{\min(K-2,k-1)} {i-2 \choose m} {N-i-1\choose k-m-1} y_l \nonumber\\
    &\quad\quad\quad\quad+\sum_{l\in \{i+2,\ldots,N\}} \sum_{m=0}^{\min(K-2,k-1)} {l-3 \choose m}{N-l\choose k-m-1} y_l \bigg]  \nonumber\\
   &= \frac{2}{K(N-1)} \bigg[\sum_{l\in \{1,\ldots,i-1\}} y_l\bigg]\cdot \bigg[\underbrace{ \sum_{k=0}^{N-2} \frac{1}{{N-2\choose k}} \sum_{m=0}^{\min(K-2,k-1)} {i-2 \choose m} {N-i-1\choose k-m-1}}_{U_{21}}\bigg]\nonumber\\
   \label{eqn:U_2}
  & + \frac{2}{K(N-1)} \bigg[\sum_{l\in \{i+2,\ldots,N\}} y_l\cdot \underbrace{ \sum_{k=0}^{N-2} \frac{1}{{N-2\choose k}} \sum_{m=0}^{\min(K-2,k-1)} {l-3 \choose m}{N-l\choose k-m-1}}_{U_{22}}\bigg]
\end{align}
Using the binomial coefficient identity $ \sum_{j=0}^M \frac{{N\choose i}{M\choose j}}{{N+M+1\choose i+j+1}} = \frac{(i+1)(M+N+2)}{(N+2)(N+1)}$, we obtain
\begin{align}
    U_{21} &= \sum_{m=0}^{\min(K-2,i-2)} \sum_{k=0}^{N-i-1} \frac{{i-2\choose m}{N-i-1\choose k}}{{N-2\choose k+m+1}}\nonumber\\
    &=\sum_{m=0}^{\min(K-2,i-2)} \frac{N-1}{(i-1)i} (m+1)\nonumber\\
    \label{eqn:U_21}
    &=\frac{N-1}{(i-1)i} \frac{ \min(K,i)\min(K-1,i-1)}{2}
\end{align}
and 
\begin{align}
    U_{22} &= \sum_{m=0}^{\min(K-2,l-3)}\sum_{k=0}^{N-l} \frac{{l-3\choose m}{N-l\choose k}}{{N-2\choose k+m+1}}\nonumber\\
    &=\sum_{m=0}^{\min(K-2,l-3)} \frac{N-1}{(l-1)(l-2)} (m+1)\nonumber\\
    \label{eqn:U_22}
    &=\frac{N-1}{(l-1)(l-2)} \frac{ \min(K,l-1)\min(K-1,l-2)}{2}
\end{align}
Now, we plug (\ref{eqn:U_21}) and (\ref{eqn:U_22}) into the expression of $U_2$ in (\ref{eqn:U_2}). Rearranging (\ref{eqn:U_2}) gives us
\begin{align*}
   & U_2 = \frac{1}{K}\sum_{l\in \{1,\ldots,i-1\}} y_l\frac{ \min(K,i)\min(K-1,i-1)}{(i-1)i}+ \frac{1}{K}\sum_{l\in \{i+2,\ldots,N\}}y_l\frac{\min(K,l-1)\min(K-1,l-2) }{(l-1)(l-2)}
\end{align*}
Therefore, we have
\begin{align*}
    &s_i-s_{i+1} \\
    &= \frac{1}{K} (y_{i+1}-y_i) (U_1 + U_2)\\
    &= \frac{1}{K} (y_{i+1} - y_i)\cdot\bigg[ \big(\frac{1}{K} (y_{i+1}+ y_i) - 2y_\text{test}\big) \frac{\min(K-1,i-1)+1}{i} +\frac{1}{K}\sum_{l\in \{1,\ldots,i-1\}} y_l\frac{ \min(K,i)\min(K-1,i-1)}{(i-1)i}\nonumber\\
    &+ \frac{1}{K}\sum_{l\in \{i+2,\ldots,N\}}y_l\frac{\min(K,l-1)\min(K-1,l-2) }{(l-1)(l-2)}\bigg]
\end{align*}
Now, we analyze the formula for $s_N$, the starting point of the recursion. Since $x_N$ is farthest to $x_\text{test}$ among all training points, $x_N$ results in non-zero marginal utility only when it is added to a set of size smaller than $K$. Hence, $s_N$ can be written as
\begin{align*}
    s_N &= \frac{1}{N} \sum_{k=0}^{K-1} \frac{1}{{N-1\choose k}} \sum_{|S|=k, S\subseteq I\setminus\{N\},} \nu(S\cup \{N\}) - \nu(S)\\
    &=\frac{1}{N}\sum_{k=1}^{K-1} \frac{1}{{N-1\choose k}} \sum_{|S|=k, S\subseteq I\setminus\{N\}}\bigg[(\frac{1}{K}\sum_{i\in S} y_{i}- y_\text{test})^2- (\frac{1}{K}\sum_{i\in S\cup \{N\}} y_{i} - y_\text{test})^2\bigg] + \frac{\nu(\{N\})}{N}\\
    &=\frac{1}{N}\sum_{k=1}^{K-1} \frac{1}{{N-1\choose k}} \sum_{|S|=k, S\subseteq I\setminus\{N\}} \bigg[(-\frac{1}{K}y_{N})\cdot(\frac{2}{K}\sum_{i\in S} y_{i} + \frac{1}{K} y_{N} - 2y_\text{test})\bigg]+\frac{\nu(\{N\})}{N}\\
    &=-\frac{K-1}{NK}y_N(\frac{1}{K}y_N-2y_\text{test})- \frac{2}{NK^2} y_N\sum_{k=1}^{K-1}\frac{{N-2\choose k-1}}{{N-1\choose k}} \sum_{l\in I\setminus \{N\}} y_l + \frac{1}{N} \nu(\{N\})\\
    &= -\frac{1}{N}y_N\bigg[\frac{K-1}{K}(\frac{1}{K}y_N-2y_\text{test} )+ \frac{2}{K^2}(\sum_{l\in I\setminus\{N\}} y_l)\sum_{k=1}^{K-1}\frac{k}{N-1}\bigg] + \frac{\nu(\{N\})}{N}\\
    &=-\frac{K-1}{NK}y_N\bigg[\frac{1}{K}y_N-2y_\text{test} + \frac{1}{N-1}\sum_{l\in I\setminus\{N\}} y_l\bigg] + \frac{\nu(\{N\})}{N}
\end{align*}
~
\end{proof}

\subsection{Weighted KNN}
\label{section:appendix_weighted_knn}
A weighted $K$NN estimate produced by a training set $S$ can be expressed as
\begin{align}
\label{eqn:knn_estimate}
    \hat{y}(S) = \sum_{k=1}^{\min\{K,|S|\}} w_{\alpha_k(S)} y_{\alpha_k}
\end{align}
where $w_{\alpha_k(S)}$ is the weight associated with the $k$th nearest neighbor of the test point in $S$. The weight assigned to a neighbor in the weighted $K$NN estimate often varies with the neighbor-to-test distance so that the evidence from more nearby neighbors are weighted more heavily~\cite{dudani1976distance}.
Correspondingly, we define the utility function associated with weighted $K$NN classification and regression tasks as
\begin{align}
\label{eqn:utility_classification_weighted_dup}
    \nu(S) = \sum_{k=1}^{\min\{K,|S|\}} w_{\alpha_k(S)} \mathbbm{1}[y_{\alpha_k(S)} = y_\text{test}]
\end{align}
and 
\begin{align}
\label{eqn:utility_regression_weighted_dup}
 \nu(S) = -\bigg(\sum_{k=1}^{\min\{K,|S|\}} w_{\alpha_k(S)} y_{\alpha_k(S)} - y_\text{test}\bigg)^2.
\end{align}

For weighted $K$NN classification and regression, the SV can no longer be computed exactly in $\mathcal{O}(N\log(N))$ time. The next theorem shows that it is however possible to compute the exact SV for weighted $K$NN in $\mathcal{O}(N^K)$ time. The theorem applies the definition (\ref{eqn:shapley_definition_no_order}) to calculating the SV and relies on the following idea to circumvent the exponential complexity:
when applying (\ref{eqn:shapley_definition_no_order}) to $K$NN, we only need to focus on the sets $S$ whose utility might be affected by the addition of $i$th training instance.
Moreover, since there are only $N^K$ possible distinctive combinations for $K$ nearest neighbors, the number of distinct utility values for all $S\subseteq I$ is upper bounded by $N^K$, in contrast to $2^N$ for general utility functions.

\begin{theorem}
\label{thm:KNN_weighted}
Consider the utility function in (\ref{eqn:utility_classification_weighted_dup}) or (\ref{eqn:utility_regression_weighted_dup}) with some weights $w_{\alpha_k(S)}$. Let $B_k(i)=\{S:|S|=k,i\notin S, S\subseteq I\}$, for $i=1,\ldots,N$ and $k=0,\ldots,K$. Let $r(\cdot)$ be a function that maps the set of training data to their ranks of similarity to $x_\text{test}$. Then, the SV of each training point can be calculated recursively as follows:
\begin{align}
\label{eqn:KNN_weighted_recursion_1}
    &s_{\alpha_N} = \frac{1}{N}\sum_{k=0}^{K-1} \frac{1}{{N-1\choose k}} \!\!\sum_{S\in B_k(\alpha_N)}\!\!
  \big[\nu(S\cup \{\alpha_N\}) - \nu(S)\big]\\
    \label{eqn:KNN_weighted_recursion_2}
    &s_{\alpha_{i+1}} = s_{\alpha_i}+ \frac{1}{N-1}\sum_{k=0}^{N-2} \frac{1}{{N-2\choose k}} \sum_{S\in D_{i,k}} A_{i,k}
\end{align}
where 
\begin{align}
   \!\!\!\! D_{i,k}\! =\! \left\{\begin{array}{l}
       \!\! B_k(\alpha_i)\cap B_k(\alpha_{i+1}),0\leq k\leq K-2  \\
      \!\!  B_{K-1}(\alpha_i)\cap B_{K-1}(\alpha_{i+1},K-1\leq k\leq N-2
    \end{array}
    \right.
\end{align}
and 
\begin{align}
     A_{i,k}\! =\! \left\{\begin{array}{l}
     \!\! 1,0\leq k\leq K-2  \\
      \!\!{N-\max r(S\cup\{\alpha_{i},\alpha_{i+1}\})\choose k-K+1}, K-1\leq k\leq N-2
    \end{array}
    \right.
\end{align}

\end{theorem}

Note that $|B_k(i)|\leq {N-1\choose k}$. Thus, the complexity for computing the weighted $K$NN SV is at most 
\begin{align}
    N(N-1)\times {N-1\choose K-1}\leq \big(\frac{e}{K-1})^{K-1}N^{K+1}
\end{align}

\begin{proof}[Proof of Theorem~\ref{thm:KNN_weighted}]
Without loss of generality, we assume that the training points are sorted according to their distance to $x_\text{test}$, such that $d(x_1,x_\text{test})\leq \ldots \leq d(x_N,x_\text{test})$. 

We start by analyzing the SV for $x_N$. Since the farthest training point does not affect the utility of $S$ unless $|S|\leq K-1$, we have
\begin{align*}
    s_N = \frac{1}{N}\sum_{k=0}^{K-1} \frac{1}{{N-1\choose k}} \sum_{|S|=k,S\subseteq I\setminus\{N\}} \big[\nu(S\cup\{N\})-\nu(S)\big]
\end{align*}
For $i\leq N-1$, the application of Lemma~\ref{lm:shapley_diff} yields
\begin{align}
\label{eqn:shapley_diff_not_efficient}
    s_i-s_{i+1}& = \frac{1}{N-1}\sum_{k=0}^{N-2}  \sum_{|S|=k,S\subseteq I\setminus\{i,i+1\}}\frac{1}{{N-2\choose k}}\cdot\big[\nu(S\cup\{i\})- \nu(S\cup\{i+1\})\big]
\end{align}
Recall that for $K$NN utility functions, $\nu(S)$ only depends on the $K$ training points closest to $x_\text{test}$. Therefore, we can also write $s_i-s_{i+1}$ as follows:
\begin{align}
\label{eqn:shapley_diff_efficient}
     s_i-s_{i+1} = \frac{1}{N-1}\sum_{k'=0}^{K-1}  \sum_{S'\in B_{k'}(i)\cap B_{k'}(i+1)} 
     M^{k'}_{i,i+1} \big[\nu(S'\cup\{i\})- \nu(S'\cup\{i+1\})\big]
\end{align}
which can be computed in at most $\sum_{k'=0}^{K-1}{N-2\choose k'}\sim \mathcal{O}(N^K)$, in contrast to $\mathcal{O}(2^{N-2})$ with (\ref{eqn:shapley_diff_not_efficient}). Our goal is thus to find $M_{i,i+1}^{k'}$ such that the right-hand sides of (\ref{eqn:shapley_diff_efficient}) and (\ref{eqn:shapley_diff_not_efficient}) are equal. More specifically, for each $S'\in B_{k'}(i)\cap B_{k'}(i+1)$, we want to count the number of $S\subseteq I\setminus\{i,i+1\}$ such that $|S|=k$, and $\nu(S\cup\{i\}) =\nu(S'\cup\{i\})$ and $\nu(S\cup\{i+1\}) =\nu(S'\cup\{i+1\})$; denoting the count by $C^{k,k'}_{i,i+1}$, we have 
\begin{align}
\label{eqn:weighted_const_count}
    M^{k'}_{i,i+1} = \sum_{k=0}^{N-2}C^{k,k'}_{i,i+1}/{N-2\choose k}.
\end{align}

When $k'\leq K-2$, only $S=S'$ satisfies $\nu(S\cup\{i\}) =\nu(S'\cup\{i\})$ and $\nu(S\cup\{i+1\}) =\nu(S'\cup\{i+1\})$. Therefore, 
\begin{align}
\label{eqn:C_when_k_small}
    C^{k,k'}_{i,i+1}=\left\{\begin{array}{l}
         1 \text{ if } k'\leq K-2 \text{ and } k=k'\\
         0 \text{ otherwise}
    \end{array}
    \right.
\end{align}
When $k' = K-1$, there will be multiple subsets $S$ of $I\setminus\{i,i+1\}$ that obey $\nu(S\cup\{i\}) =\nu(S'\cup\{i\})$ and $\nu(S\cup\{i+1\}) =\nu(S'\cup\{i+1\})$. Let $r$ denote the index of the training point that is farthest to $x_\text{test}$ among $S\cup\{i,i+1\}$, i.e., $r= \max S\cup\{i,i+1\}$. Note that adding any training points with indices larger than $r$ into $S'\cup\{i\}$ or $S'\cup\{i+1\}$ would not affect their utility. Hence,
\begin{align}
\label{eqn:C_when_k_large}
    C^{k,k'}_{i,i+1}=\left\{\begin{array}{l}
        {N-r\choose k-K+1} \text{ if } k'=K-1,k\geq k'\\
         0 \text{ otherwise}
    \end{array}
    \right.
\end{align}
Combining (\ref{eqn:shapley_diff_efficient}), (\ref{eqn:weighted_const_count}), (\ref{eqn:C_when_k_small}), and (\ref{eqn:C_when_k_large}) yields the recursion in (\ref{eqn:KNN_weighted_recursion_1}) and (\ref{eqn:KNN_weighted_recursion_2}).
\end{proof}

\subsection{Multiple Data Per Contributor}
\label{section:appendix_multi_data_per_selelr}
Now, we investigate the method to compute the SV when each seller provides more than one data instance. The goal is to fairly value individual sellers in lieu of individual training points. Following the previous notations, we still use $I=\{1,\ldots,N\}$ to denote the set of all training instances and use $I^s$ to denote the set of all sellers, i.e., $I^s=\{1,\ldots,M\}$. The number of training instances owned by $j$th seller is $N_j$. We denote the $i$th training point contributed by $j$th seller as $x_j^{(i)}$. Without loss of generality, we assume that every seller's data is sorted such that $d(x_j^{(1)},x_\text{test})\leq \ldots\leq d(x_j^{(N_j)},x_\text{test})$. 

Let $h(i)$ denote the owner of $i$th training instance. With slight abuse of notations, we denote the owners of a set $S$ of training instance as $h(S)$, where $S\subseteq I$, and denote the training instances from the set of sellers $\tilde{S}\subseteq I^s$ by $h^{-1}(\tilde{S})$.  Let $\mathcal{N}(S)=\{\alpha_1(S),\ldots,$ $\alpha_{\min\{K,|S|\}}(S)\}$ be a function that maps a set of training instances to its $K$-nearest neighbors. Let $\mathcal{A}=\{S: \tilde{S}\subseteq I^s, |\tilde{S}|\leq K, S=\mathcal{N}(h^{-1}(\tilde{S}))\}$ be the collection of all possible $K$-nearest neighbors formed by sellers; $|\tilde{S}|\leq K$ because the top $K$ instances cannot belong to more than $K$ sellers. 
The next theorem shows that we can compute the SV of each seller with $\OO(M^K)$. 

\begin{theorem}
Consider the utility functions (\ref{eqn:utility_classification_unweighted}), (\ref{eqn:utility_regression}), (\ref{eqn:utility_classification_weighted}) or (\ref{eqn:utility_regression_weighted}). Let $\mathcal{A}^{\setminus j} =\{S: S\in \mathcal{A}, j\notin h(S)\} $ be the set of top-$K$ elements that do not contain sell $j$'s data, $\mathcal{D}(\tilde{S})=\{S: S\in \mathcal{A}, h(S)=\tilde{S}\}$ be the set of top-$K$ elements of the data from the set $\tilde{S}$ of sellers, and $G(S,j)=\{j':d(x_{j'}^{(1)},x_\text{test})\geq \max_{x\in S} d(x,x_\text{test}),$ $S\in\mathcal{A}^{\setminus j}, j'\in I^s\setminus\{ h(S),j\}\}$ be the set of sellers that do not affect the $K$-nearest neighbors when added into the sellers $h(S)$ and $S$ does not include seller $j$'s data. Then, the SV of seller $j$ can be represented as
\begin{align}
  \!\!\!\!  s_j\! = \!\frac{1}{M}\!\! \sum_{S\in\mathcal{A}^{\setminus j}} \!\!\!\!\sum_{k=0}^{|G(S,j)|}\!\!\frac{{|G(S,j)|\choose k}}{{M-1\choose |h(S)|+k} } \big[\nu(\mathcal{D}(h(S)\cup \{j\}))\!-\!\nu(S)\big]
\end{align}
\end{theorem}


\subsection{Valuing Computation}
\label{section:appendix_computation_valuation}

We show that one can compute the SV for both the sellers and the analyst with the same computational complexity as the one needed for the data-only game. The procedures to compute the SV for unweighted/weighted $K$NN classification/regression in the composite game setup are exhibited in the theorems below.

\subsubsection{Unweighted $K$NN classification}
\begin{theorem}
\label{thm:KNN_unweighted_class_composite}
Consider the utility function $\nu_c$ in (\ref{eqn:utility_composite}), where $\nu(\cdot)$ is the $K$NN classification performance measure in (\ref{eqn:utility_classification_unweighted}).
Then, the SV of each training point and the computation contributor can be calculated recursively as follows:
\begin{align}
\label{eqn:KNN_unweighted_class_composite_1}
&s_{\alpha_N}=\frac{K+1}{2(N+1)N}\mathbbm{1}[y_{\alpha_{N}} = y_\text{test}]
\\
\label{eqn:KNN_unweighted_class_composite_2}
& s_{\alpha_i} = s_{\alpha_{i+1}} +  \frac{\mathbbm{1}[y_{\alpha_i} = y_\text{test}] - \mathbbm{1}[y_{\alpha_{i+1}} = y_\text{test}]}{K}\cdot\frac{\min\{i,K\}(\min\{i,K\}+1)}{2i(i+1)}\\
& s_C = \nu(I)-\sum_{i=1}^N s_i
\end{align}
\end{theorem}

Comparing $s(\nu,i)$ in Theorem~\ref{thm:KNN_unweighted_class} and $s(\nu_c,i)$ in the above theorem, we have 
\begin{align}
\label{eqn:ratio_1}
    &\frac{s(\nu_c,\alpha_N)}{s(\nu,\alpha_N)}  =\frac{\min\{N,K\}+1}{2(N+1)}\\
    \label{eqn:ratio_2}
   &\frac{s(\nu_c,\alpha_i) - s(\nu_c,\alpha_{i+1})}{s(\nu,\alpha_i) - s(\nu,\alpha_{i+1})} =  \frac{\min\{i,K\}+1}{2(i+1)}
\end{align}
Note that the right-hand side of (\ref{eqn:ratio_1}) and (\ref{eqn:ratio_2}) are at most $1/2$ for all $i=1,\ldots,N-1$; thus, each seller will receive a much smaller share of the total revenue in the composite game than that in the data-only game. Moreover, the analyst obtains a least one half of the total revenue in the composite game setup.

\subsubsection{Unweighted $K$NN Regression}
\begin{theorem}
\label{thm:KNN_unweighted_regression_composite}
Consider the utility function in (\ref{eqn:utility_composite}), where $\nu(\cdot)$ is the $K$NN regression performance measure in (\ref{eqn:utility_regression}). Then, the SV of each training point and the computation contributor can be calculated recursively as follows:

\begin{align}
\label{eqn:KNN_unweighted_regression_composite_1}
&s_{\alpha_N} = -\frac{1}{K(N+1)}y_{\alpha_N}\bigg[\frac{(K+2)(K-1)}{2N}(\frac{1}{K}y_{\alpha_N}-2y_\text{test}) + \frac{2(K-1)(K+1)}{3N(N-1)}\sum_{l\in I\setminus\{\alpha_N\}} y_l\bigg] \nonumber\\
&\quad\quad\quad-\frac{1}{N(N+1)}\bigg[\frac{1}{K} y_{\alpha_k(N)} - y_\text{test}\bigg]^2\\
\label{eqn:KNN_unweighted_regression_composite_2}
 &s_{\alpha_i} = s_{\alpha_{i+1}}+\frac{1}{K} (y_{\alpha_{i+1}} - y_{\alpha_i}) \cdot\bigg[ \big(\frac{1}{K} (y_{\alpha_{i+1}}+ y_{\alpha_i}) - 2y_\text{test}\big)\cdot\frac{\min\{K+1,i+1\}\cdot\min\{K,i\}}{2i(i+1)}\nonumber\\
   &\quad\quad\quad+\frac{1}{K}\sum_{l\in \{1,\ldots,i-1\}} y_{\alpha_l}\cdot\frac{ 2\min(K+1,i+1)\min(K,i)\min(K-1,i-1)}{3(i-1)i(i+1)}\nonumber\\
   &\quad\quad\quad+ \frac{1}{K}\sum_{l\in \{i+2,\ldots,N\}}y_{\alpha_l}\cdot\frac{2\min(K+1,l)\min(K,l-1)\min(K-1,l-2) }{3l(l-1)(l-2)}\bigg]\\
& s_C = \nu(I)-\sum_{i=1}^N s_i
\end{align}
\end{theorem}

\subsubsection{Weighted $K$NN}

\begin{theorem}
\label{thm:KNN_weighted_composite}
Consider the utility function in (\ref{eqn:utility_composite}), where $\nu(\cdot)$ is the weighted $K$NN performance measure in (\ref{eqn:utility_classification_weighted}) or (\ref{eqn:utility_regression_weighted}) with some weights $w_{\alpha_k(S)}$. Let $B_k(i)=\{S:|S|=k,i\notin S, S\subseteq I\}$, for $i=1,\ldots,N$ and $k=0,\ldots,K$. Let $r(\cdot)$ be a function that maps the set of training data to their ranks in terms of similarity to $x_\text{test}$. Then, the SV of each training point and the computation contributor can be calculated recursively as follows:
\begin{align}
\label{eqn:KNN_weighted_recursion_1_multi}
    &s_{\alpha_N} = \frac{1}{N+1}\sum_{k=0}^{K-1} \frac{1}{{N\choose k+1}} \sum_{S\in B_k(\alpha_N)} \nu(S\cup \{\alpha_N\}) - \nu(S)\\
    \label{eqn:KNN_weighted_recursion_2_multi}
    &s_{\alpha_{i+1}} = s_{\alpha_i}+ \frac{1}{N}\sum_{k=0}^{K-2} \frac{1}{{N-1\choose k+1}}\sum_{S\in B_k(\alpha_i)\cap B_k(\alpha_{i+1})} \nu(S\cup \{\alpha_i\}) - \nu(S\cup \{\alpha_{i+1}\})\nonumber\\
    &+ \frac{1}{N}\sum_{k=K-1}^{N-2} \frac{1}{{N-1\choose k+1}}\sum_{S\in B_{K-1}(\alpha_i)\cap B_{K-1}(\alpha_{i+1})} {N-\max r( S\cup\{\alpha_{i},\alpha_{i+1}\})\choose k-K+1}\nu(S\cup \{\alpha_i\}) - \nu(S\cup \{\alpha_{i+1}\})\\
    &s_C = \nu(I)-\sum_{i=1}^N s_i
\end{align}
\end{theorem}

\subsubsection{Multi-data-per-seller $K$NN}

\begin{theorem}
Consider the utility functions (\ref{eqn:utility_classification_unweighted}), (\ref{eqn:utility_regression}), (\ref{eqn:utility_classification_weighted}) or (\ref{eqn:utility_regression_weighted}).

Let $\mathcal{A}^{\setminus j} =\{S: S\in \mathcal{A}, j\notin h(S)\} $ be the set of top-$K$ elements that do not contain sell $j$'s data, $\mathcal{D}(\tilde{S})=\{S: S\in \mathcal{A}, h(S)=\tilde{S}\}$ be the set of top-$K$ elements of the data from the set $\tilde{S}$ of sellers, and $G(S,j)=\{j':d(x_{j'}^{(1)},x_\text{test})\geq \max_{x\in S} d(x,x_\text{test}),$ $S\in\mathcal{A}^{\setminus j}, j'\in I^s\setminus\{ h(S),j\}\}$ be the set of sellers that do not affect the $K$-nearest neighbors when added into the sellers $h(S)$ and $S$ does not include seller $j$'s data. 
Then, the SV of seller $j$ can be represented as
\begin{align}
    s_j = \frac{1}{M+1} \sum_{S\in \mathcal{A}^{\setminus j}} \sum_{k=0}^{|G(S,j)|}\frac{{|G(S,j)|\choose k}}{{M\choose |h(S)|+k+1} } \big[\nu(\mathcal{D}(h(S)\cup \{j\}))-\nu(S)\big]
\end{align}
and the SV of the computation contributor is
\begin{align}
    s_C = \nu(I)-\sum_{i=1}^M s_i
\end{align}
\end{theorem}


\section{Generalization to Piecewise Utility Difference}
\label{section:appendix_piecewise}
A commonality of the utility functions between the unweighted $K$NN classifier and its extensions is that the difference in the marginal contribution of $i$ and $j$ to a set $S\subseteq I\setminus\{i,j\}$ has a ``piecewise'' form:
\begin{align}
\label{eqn:piecewise_util_dup}
    \nu(S\cup i) - \nu(S\cup j) = \sum_{t=1}^T C_{ij}^{(t)} \mathbbm{1}[S\in \mathcal{S}_t]
\end{align}
where $\mathcal{S}_t\subseteq 2^{I\setminus\{i,j\}}$. For instance, the utility difference for unweighted $K$NN classification obeys 
\begin{align}
    & \nu(S\cup \{i\}) - \nu(S\cup \{i+1\})
    = \frac{\mathbbm{1}[y_i=y_\text{test}] - \mathbbm{1}[y_{i+1}=y_\text{test}]}{K}\mathbbm{1}[S\in \mathcal{S}_1]
\end{align}
where we assume the training data is sorted according to their similarity to the test point and  
\begin{align}
\label{eqn:S_1}
    \mathcal{S}_1 = \{S:\sum_{l\in S} \mathbbm{1}[d(x_l,x_\text{test}) - d(x_i,x_\text{test})<0] < K\}
\end{align}
Hence, we have $T=1$ and $C_{ij}^1 =(\mathbbm{1}[y_{i}=y_\text{test}] - \mathbbm{1}[y_{i+1}=y_\text{test}])/K$ for unweighted $K$NN classification utility function.

The utility difference for unweighted $K$NN regression can be expressed as
\begin{align}
\label{eqn:util_diff_knn_regression}
  & \nu(S\cup \{i\}) - \nu(S\cup \{i+1\})\nonumber \\
    = &\frac{1}{K}(y_{i+1}-y_i)(\frac{1}{K}(y_{i+1}+y_i)-2y_\text{test}) \mathbbm{1}[S\in \mathcal{S}_1]+ \frac{2}{K^2}(y_{i+1}-y_i) \sum_{l\in I\setminus\{i,i+1\}} y_l \mathbbm{1}[S\in \mathcal{S}_1,S\ni l ] 
\end{align}
where $\mathcal{S}_1$ is defined in (\ref{eqn:S_1}). Therefore, we can obtain the piecewise form of the utility difference in (\ref{eqn:piecewise_util_dup}) by letting $T=N-1$, $C_{ij}^{(1)}=\frac{1}{K}(y_{i+1}-y_i)(\frac{1}{K}(y_{i+1}+y_i)-2y_\text{test})$, $\{\mathcal{S}_t\}_{t=2}^{N-1} = \{S_l\}_{l\in I\setminus\{i,{i+1}\}}$ where $S_l = \mathcal{S}_1 \cap \{S:l\in S,S\subseteq I\setminus\{i,{i+1}\}\}$, and the corresponding $\{C_{ij}^{(t)}\}_{t=2}^{N-1} = \{\frac{2}{K^2}(y_{i+1}-y_i) y_l\}_{l\in I\setminus\{i,i+1\}} $.

For weighted $K$NN utility functions, we can instantiate the utility difference (\ref{eqn:piecewise_util_dup}) with $T = \sum_{k=0}^K {N-2\choose k}$ adn $\mathcal{S}_t\subseteq 2^{I\setminus\{i,j\}}$ is a collection of sets that have the same top $K$ elements.

An application of Lemma~\ref{lm:shapley_diff} to the utility functions with the piecewise utility difference form indicates that the difference in the SV between $i$ and $j$ can be represented as
\begin{align}
    &s_i - s_j = \frac{1}{N-1}\sum_{S\subseteq I\setminus\{i,j\}} \sum_{t=1}^T \frac{C_{ij}^{(t)}}{{N-2\choose |S|}} \mathbbm{1}[S\in \mathcal{S}_t] \\
    \label{eqn:piecewise_shapley_diff_dup}
    &= \frac{1}{N-1}\sum_{t=1}^T C_{ij}^{(t)} \bigg[\sum_{k=0}^{N-2} \frac{|\{S:S\subseteq I\setminus\{i,j\},S\in \mathcal{S}_t, |S|=k\}| }{{N-2\choose k}}\bigg]
\end{align}
With the piecewise property (\ref{eqn:piecewise_util_dup}), the SV calculation is reduced to a counting problem. As long as the quantity in the bracket of (\ref{eqn:piecewise_shapley_diff_dup}) can be efficiently evaluated, the SV can be obtained in $\mathcal{O}(NT)$.

\section{Proof of Theorem~\ref{thm:bennett_bound}}
\begin{proof}
We will use Bennett's inequality to derive the approximation error associated with the estimator in (\ref{eqn:permutation_estimator}). Bennett's inequality provides an upper bound on the deviation of the empirical mean from the true mean in terms of the variance of the underlying random variable. Thus, we first provide an upper bound on the variance of $\phi_i$ for $i=1,\ldots,N$.

Let the range of $\phi_i$ for $i=1,\ldots,N$ be denoted by $[-r,r]$. Further, let $q_i=P[\phi_i=0]$. Let $W_i$ be  an indicator of whether or not $\phi_i=0$, i.e., $W_i=\mathbbm{1}[\phi_i\neq 0]$; thus $P[W_i=0] = q_i$ and $P[W_i=1] = 1-q_i$.

We analyze the variance of $\phi_i$. By the law of total variance,
\begin{align}
    \texttt{Var}[\phi_i] = \E[\texttt{Var}[\phi_i|W_i]] +\texttt{Var}[\E[\phi_i|W_i]]
\end{align}
Recall $\phi_i\in [-r,r]$. Then, the first term can be bounded by
\begin{align}
    &\E[\texttt{Var}[\phi_i|W_i]]\nonumber\\
    &= P[W_i=0] \texttt{Var}[\phi_i|W_i=0] + P[W_i=1] \texttt{Var}[\phi_i|W_i=1]\\
    &= q_i\texttt{Var}[\phi_i|\phi_i = 0] + (1-q_i) \texttt{Var}[\phi_i|\phi_i\neq 0]\\
    &= (1-q_i) \texttt{Var}[\phi_i|\phi_i\neq 0]\\
    &\leq (1-q_i) r^2
\end{align}
where the last inequality follows from the fact that if a random variable is in the range $[m,M]$, then its variance is bounded by $\frac{(M-m)^2}{4}$.

The second term can be expressed as
\begin{align}
\label{eqn:var_of_exp}
   & \texttt{Var}[\E[\phi_i|W_i]]\nonumber\\
    & = \E_{W_i}[(\E[\phi_i|W_i] - \E[\phi_i])^2]\\
    &= P[W_i=0] (\E[\phi_i|W_i=0] - \E[\phi_i])^2 + P[W_i=1] (\E[\phi_i|W_i=1] - \E[\phi_i])^2\\
    &= q_i (\E[\phi_i|\phi_i = 0] - \E[\phi_i])^2 + (1-q_i)(\E[\phi_i|\phi_i\neq 0] - \E[\phi_i])^2\\
    &=q_i (\E[\phi_i])^2 + (1-q_i)(\E[\phi_i|\phi_i\neq 0] - \E[\phi_i])^2
\end{align}

Note that
\begin{align}
    \E[\phi_i] &= P[W_i=0] \E[\phi_i|\phi_i=0] + P[W_i=1]\E[\phi_i|\phi_i\neq 0]\\
    \label{eqn:expectation_decomp}
    &=(1-q_i) \E[\phi_i|\phi_i\neq 0]
\end{align}
Plugging (\ref{eqn:expectation_decomp}) into (\ref{eqn:var_of_exp}), we obtain 
\begin{align}
    \texttt{Var}[\E[\phi_i|W]]& =(q_i(1-q_i)^2  + q_i^2 (1-q_i)) (\E[\phi_i|\phi_i\neq 0])^2
\end{align}
Since $|\phi_i|\leq r$, $(\E[\phi_i|\phi_i\neq 0])^2\leq r^2$. Therefore,
\begin{align}
     \texttt{Var}[\E[\phi_i|W]] \leq q_i(1-q_i) r^2
\end{align}
It follows that
\begin{align}
\label{eqn:bound_variance}
    \texttt{Var}[\phi_i]\leq (1-q_i^2)r^2
\end{align}

Therefore, we can upper bound the variance of $\phi_i$ in terms of the probability that $\phi_=0$. Now, let us compuate $P[\phi_i=0]$ for $i=1,\ldots,N$.

Without loss of generality, we assume that $x_i$ are sorted according to their distance to the test point $x_\text{test}$ in an ascending order. 

When $i\leq K$, then whatever place $x_i$ appears in the permutation $\pi$, adding $x_i$ to the set of points preceding $i$ in the permutation will always potentially lead to a non-zero utility change. Therefore, we know that $q_i\geq 0$ and 
\begin{align}
   \Var[\phi_i]\leq r^2 \equiv \sigma_i^2 \text{ for $i=1,\ldots,K$}
\end{align}

When $i \geq K+1$, adding $x_i$ to $P^\phi_i$ may lead to zero utility change. More specifically, if there are no less than $K$ elements in $\{x_1,\ldots,x_{i-1}\}$ appearing in $P^\phi_i$, then adding $i$ would not change the $K$ nearest neighbors of $P^\phi_i$ and thus $\phi_i$. Let the position of $x_i$ in the permutation $pi$ be denoted by $k$. Note that if there are at least $K$ elements in $\{x_1,\ldots,x_{i-1}\}$ appearing before $x_i$ in the permutation, then $x_i$ must at least locate in order $K+1$ in the permutation, i.e., $k\geq K+1$. 

The number of permutations such that $x_i$ is in the $k$th slot and there are at least $K$ elements appearing before $x_i$ is
\begin{align}
    \sum_{m=K}^{\min\{i-1,k-1\}} {k-1\choose m} {N-k\choose i-1-m} (i-1)!(N-i)!
\end{align}
Thus, the probability that $\phi_i$ is zero is lower bounded by
\begin{align}
   q_i^*&=\frac{ \sum_{k=K+1}^{N} \sum_{m=K}^{\min\{i-1,k-1\}} {k-1\choose m} {N-k\choose i-1-m} (i-1)!(N-i)!}{N!}\\
   &= \frac{ \sum_{k=K+1}^{N} \sum_{m=K}^{\min\{i-1,k-1\}} {k-1\choose m} {N-k\choose i-1-m} }{{N-1\choose i-1}N}\\
   &= \frac{i-K}{i}
\end{align}
By (\ref{eqn:bound_variance}), we have
\begin{align}
    \Var[\phi_i]\leq (1-q_i^{*2})r^2 \text{ for $i=K+1,\ldots,N$}
\end{align}

By Bennett's inequality, we can bound the approximation error associated with $\hat{s}_i$ by
\begin{align}
    P[|\hat{s}_i - s_i|> \epsilon] \leq 2\exp(-\frac{T\sigma_i^2}{r^2} h(\frac{r\epsilon}{\sigma_i^2}))
\end{align}
By the union bound, if $P[|\hat{s}_i - s_i| > \epsilon]\leq \delta_i$ for all $i=1,\ldots,N$ and $\sum_{i=1}^N \delta_i=\delta$, then we have
\begin{align}
    P[\max_i |\hat{s}_i-s_i|> \epsilon]= P[\cup_{i=1,\ldots,N] \{|\hat{s}_i-s_i}|> \epsilon\}]\leq \sum_{i=1}^N P[|\hat{s}_i-s_i|> \epsilon] \leq \sum_{i=1}^N \delta_i = \delta
\end{align}
Thus, to ensure that $P[\max_i |\hat{s}_i-s_i|> \epsilon]\leq \delta$, we only need to choose $T$ such that
\begin{align}
   2\exp(-\frac{T\sigma_i^2}{r^2} h(\frac{r\epsilon}{\sigma_i^2})) \leq \delta_i
\end{align}
which yields
\begin{align}
    T&\geq \frac{r^2}{\sigma_i^2 h(\frac{r\epsilon}{\sigma_i^2})}\log \frac{2}{\delta_i}
\end{align}
Since
\begin{align}
    \frac{r^2}{\sigma_i^2 h(\frac{r\epsilon}{\sigma_i^2})}\leq \frac{1}{(1-q_i^2) h(\frac{\epsilon}{(1-q_i^2)r})}
\end{align}
it suffices to let 
\begin{align}
    T&\geq \frac{\log \frac{2}{\delta_i}}{(1-q_i^2) h(\frac{\epsilon}{(1-q_i^2)r})}
\end{align}
for all $i=1,\ldots,N$. Therefore, we would like to choose $\{\delta_i\}_{i=1}^N$ such that $\max_{i=1,\ldots,N}T_i^*$ is minimized. We can do this by letting
\begin{align}
\frac{\log \frac{2}{\delta_i}}{(1-q_i^2) h(\frac{\epsilon}{(1-q_i^2)r})} = T^*
\end{align}
which gives us
\begin{align}
    \delta_i = 2\exp(-T^*(1-q_i^2) h(\frac{\epsilon}{(1-q_i^2)r}))
\end{align}
Since $\sum_{i=1}^N\delta_i = \delta$, we get
\begin{align}
    \sum_{i=1}^N \exp(-T^*(1-q_i^2) h(\frac{\epsilon}{(1-q_i^2)r})) = \delta/2
\end{align}
and the value of $T^*$ can be solved numerically.

\end{proof}



\section{Derivation of the Approximate Lower Bound on Sample Complexity for the Improved MC Approximation}
\label{section:appendix_lower_bound}

Because $\log (1+u)>\frac{2x}{2+x}$~\cite{topsok2006some}, we have $h(u)> \frac{x^2}{2+x}$. Thus, $(1-q_i^2)h(\frac{\epsilon}{(1-q_i^2)r}))> \frac{\epsilon^2}{2(1-q_i^2)r+\epsilon r}$. Furthermore, by the definition of $q_i$, $(1-q_i)^2=1$ for $i=1,\ldots,K$ and decreases approximately with the speed $2K/i$ otherwise. Thus, the lower bound of $(1-q_i^2)h(\frac{\epsilon}{(1-q_i^2)r}))$ increases linearly with $i$ when $i\geq K+1$. Letting $x = \exp(-T^*)$, we can rewrite (\ref{eqn:bennett_bound}) as $\sum_{i=1}^N x^{(1-q_i^2)h(\frac{\epsilon}{(1-q_i^2)r}))}=\delta/2$. In light of the above analysis, $x^{(1-q_i^2)h(\frac{\epsilon}{(1-q_i^2)r}))}$ will have significant values when $i\geq K$ and is comparatively negligible otherwise. Therefore, we can derive an approximate solution $\tilde{T}$ to $T^*$ by solving the following equation

\begin{align}
    K \exp(-\tilde{T} h(\frac{\epsilon}{r})) = \delta/2.
\end{align}
which gives us
\begin{align}
    \tilde{T} =\frac{1}{h(\epsilon/r)} \log \frac{2K}{\delta}
\end{align}
Due to the inequality $h(u)\leq u^2$, we can obtain the following lower bound on $\tilde{T}$:
\begin{align}
    \tilde{T}\geq \frac{r^2}{\epsilon^2}\log \frac{2K}{\delta}
\end{align}


\end{document}